\documentclass[twoside,11pt]{article}

\usepackage{jmlr2e}
\usepackage{subfig}
 \usepackage{amsmath}
\usepackage{graphicx}
\usepackage[bb=boondox]{mathalfa}
\usepackage{hhline}
\usepackage{multirow}
\usepackage{array}
\usepackage{epsfig}
 \usepackage[export]{adjustbox}

\usepackage{setspace}

\usepackage{relsize}
\usepackage{array}
\usepackage{stackrel}
\usepackage{enumerate}
\graphicspath{{Figures/}}
\usepackage{epstopdf}
\usepackage{subcaption}
\usepackage{algorithm}
\usepackage{algpseudocode}

\newcommand{\Ic}{\mathcal{I}}
\newcommand{\Jc}{\mathcal{J}}

\newcommand{\Xc}{\mathcal{X}}
\newcommand{\Yc}{\mathcal{Y}}
\newcommand{\Bc}{\mathcal{B}}
\newcommand{\Hc}{\mathcal{H}}
\newcommand{\Ptheta}{\pmb{\theta}}
\newcommand{\Pb}{\mathbb{P}}

\newcommand{\Sc}{\mathcal{S}}

\newcommand{\Mc}{\mathcal{M}}

\newcommand{\Nc}{\mathcal{N}}

\DeclareMathOperator*{\argmax}{\arg\!\max}

\usepackage{lastpage}
\jmlrheading{27}{2026}{1-\pageref{LastPage}}{\ \ }{\ \ }{21-0000}{Hajiani and Seyfe}

\ShortHeadings{Model Sorting: Complexity Management Using the Principle of Nested Models}{Mohammad Ali Hajiani and Babak Seyfe}
\firstpageno{1}

\begin{document}

\title{Model Sorting: Complexity Management Using the Principle of Nested Models}

\author{\name Mohammad Ali Hajiani \email mohammadali.hajiani@ shahed.ac.ir \\
       \addr Information Theoretic Learning Systems Laboratory (ITLSL)\\ Department of Electrical Engineering\\ Shahed University\\ Tehran, Iran 
       \AND
       \name Babak Seyfe \email Seyfe@ shahed.ac.ir \\
       \addr Information Theoretic Learning Systems Laboratory (ITLSL)\\ Department of Electrical Engineering\\ Shahed University\\ Tehran, Iran }

 \editor{ \ \ }

\maketitle

\begin{abstract}
We propose a novel approach to select the best model for the data. Based on the exclusive properties of nested models, we find the most parsimonious model containing the risk-minimizing predictor. We prove the existence of probable approximately correct (PAC) bounds on the difference of the minimum empirical risk of two successive nested models, termed successive empirical excess risk. Based on these bounds, we propose a model order selection method called nested empirical risk (NER). Using the sorted NER (S-NER) method to sort models intelligently, the minimum risk decreases. We establish a test for deciding whether further expansion of the model is likely to decrease the minimum achievable risk, providing a principled criterion for model growth. With high probability, NER and S-NER choose the true model order and the most parsimonious model containing the risk-minimizing predictor, respectively. We apply S-NER to linear regression and show that, without any prior information, it outperforms feature sorting algorithms like orthogonal matching pursuit (OMP) that rely on prior knowledge of the true model order. Furthermore, on the UCR dataset, the NER method dramatically reduces classification complexity with a negligible loss in accuracy.
\end{abstract}

\tableofcontents
\section{Introduction}\label{Sec1}

The development of data storage hardware, the manufacture of precision instruments, the increasing accuracy of measuring devices, and the emergence of the internet and social networks have led to the growth of big data.  With the advancement of technology,  the power of data analysis has increased, and artificial intelligence (AI) has become rapidly popular \cite{ding2018model}. Machine learning algorithms, neural networks, generative adversarial networks, and large language models (chatbots) are examples of AI-based methods.  These models are extremely complex and may contain up to several million parameters. Learning, optimizing, and evaluating these parameters requires substantial storage capacity and significant computational resources.  Advanced processors are used for data processing, resulting in high energy consumption \cite{falk2026computation}. With the development of machine learning and artificial intelligence models, their energy consumption and storage requirements are expected to increase \cite{mishra2025sustainability}. As a result, the corresponding energy and storage costs increase accordingly.

As the complexity of AI models grows, processing their parameters costs  significant energy and storage space.  Reducing the complexity of the model is an approach to controlling the energy and storage costs of AI models. Thus, the goal is to lower the computational, energy, and storage costs of AI models via complexity control, without sacrificing predictive performance \cite{peykani2026green}.                                             

An effective strategy for achieving this objective is to select, from among the models that adequately fit the data, the one with the lowest complexity, such as the model containing the fewest parameters   \cite{hastie2009elements}. The procedure of selecting a statistical model from a class of models based on the observed data is called \textit{model selection}. Selecting the simplest model that fits the data well significantly saves energy and storage space for learning AI models \cite{cruciani2025choosing}. 

The model selection is used in many applications, such as engineering \cite{brunton2024engineering,mohr2023fast,hajiani2026random}, economy and finance \cite{giannone2024model,kim1999has}, epidemiology \cite{greenland1989modeling,cramer2025epidemiology}, ecology and biology \cite{johnson2004model,burnham2025ecological}, feature selection \cite{owrang2018model,li2024feature,li2025globally,nandy2024bayes}, mixture order selection \cite{chen2023estimation,nguyen2026mixture}, selecting the number of neurons and layers in neural networks\cite{rivals2000statistical,ren2024nas}, and signal enumeration \cite{asadi2018signal,yi2022source,ge2025source}.

Several model selection methods are proposed for different problems. Some model selection methods are designed to be applied to a vast range of problems, e.g., minimum description length (MDL) \cite{rissanen1978modeling,yamanishi2023learning}, structural risk minimization (SRM) \cite{vapnik1998statistical,qu2025model}, Akaike information criteria (AIC)\cite{akaike1974new}, and Bayesian information criteria (BIC) \cite{schwarz1978estimating}. However, some model selection methods are designed for a specific problem. For example, extended Fisher information criteria (EFIC) \cite{owrang2018model} and extended Bayesian information criteria-robust (EBICR) \cite{gohain2023robust} are presented as feature selection in high-dimensional linear regression. 

In feature selection problems, we deal with a large number of features or covariates, and we want to select the most parsimonious set of valuable features.
For instance, in \cite{owrang2018model},  the least angle regression (LARS) algorithm is exploited to sort the features \cite{efron2004least,yin2025lowrank}, and the EFIC method selects the model order of a high-dimensional linear regression. Also, in \cite{gohain2023robust},  the orthogonal matching pursuit (OMP) algorithm sorts the features \cite{weisberg2005applied,tohidi2025revisiting}, and the extended robust information criterion, called EBICR, was proposed for model order selection in the sparse high-dimensional linear regression problem. In  \cite{kallummil2018signal}, the residual ratio test as a stopping condition provides the support recovery of a sparse high-dimensional linear regression model based on the sequence of support features estimated by the OMP algorithm.

As previously stated, feature sorting constitutes a fundamental component of several model selection procedures, such as the EFIC and EBICR methods. In its absence, these methods are unable to estimate the true model order at a manageable computational cost. 

While AI's primary target is choosing the model with the minimum (population) risk, model selection takes this a step further by seeking the simplest model achieving the minimum population risk. Therefore, model selection will play vital role in AI to control the model complexity, and consequently reduce energy and storage costs.

In this paper, we control model expansion intelligently to manage the complexity, relying on the fact that extending the model will be beneficial if it decreases the minimum risk.  

However, the risk is often unavailable since the data distribution is unknown. Therefore, we cannot find the predictor with minimum risk and determine whether the minimum risk has changed or not through the model expansion. 

Alternatively, we can use the empirical risk based on the observed data. The empirical risk minimization (ERM) selects the predictor with the minimum empirical risk, called the ERM predictor  \cite{vapnik1999nature}. However, ERM desires to choose more complex models.

To use the intelligent model expansion method, first, we introduce a specific kind of model class, known as \textit{Nested Models}, that can be applied in model selection (See Section \ref{Sec2}). 

Although the term "nested" is commonly used in the literature   \cite{mohri2018foundations,pesaran1978testing,marcellino2008model}, we use it in a more rigorous sense that is consistent with its meaning in the previous works. We show that the nested models class has unique properties useful for model selection\footnote{It seems that the properties of the nested models are often overlooked in the literature.}.  

Usually, we deal with the class of nested models. They are used frequently in statistics and machine learning. Regression, autoregressive, and polynomial models are popular examples of nested models used to fit the data for prediction or estimation purposes \cite{pei2022local,moon2021ar}.   As another example, we will show that the models used in \cite{owrang2018model} to select the coefficients of a high-dimensional linear regression are nested. Also, in \cite{omidi2026random}, Omidi et al. use nested models for random feature pruning in multi-class classification. We can show that these model classes used in \cite{li2019matching} are also nested. In addition, in some engineering applications,  the models that are applied in signal enumeration \cite{lu2012generalized,asadi2018signal}, array signal processing \cite{barthelme2021machine}, multi-user detection \cite{wang1998blind},  classification \cite{chapelle1999model,omidi2026random}, speech processing \cite{ephraim1995signal}, and blind source separation \cite{wang1998blind} are nested models. 

In the nested model families, the ERM desires to select the biggest model, and most model selection methods use penalties (usually based on information criteria) for the empirical risk to control this desire \cite{owrang2018model,gohain2023robust}. 

As stated, extending the model will be beneficial if it decreases the minimum risk while the risk function is unavailable. Therefore, investigations on the statistical behavior of empirical risk (as an alternative to risk) during nesting modeling helps us to know any change in the minimum risk. Hence, we use Probably Approximately Correct (PAC) bounds for this purpose.

Recently, several PAC bounds have been proposed for different learning problems \cite{alquier2021user}. The generalization bounds are the upper bounds on the difference between risk and empirical risk for every predictor in the model \cite{mohri2018foundations}. These bounds have an essential role in studying the behavior of a learning problem.  Hence, the generalization bounds in model selection methods are investigated. In most of these methods, an additive penalty controls the desire to select the biggest model \cite{mohri2018foundations,shawe1998structural,koltchinskii2001rademacher}. For instance, in \cite{shawe1998structural},  Shawe-Taylor et al. propose the SRM method based on the concentration of the predictor's population risk around its empirical risk. Also, in \cite{ koltchinskii2001rademacher}, using another concentration inequality, an additive penalty was used for model selection. 

In this paper, we propose PAC bounds for a new empirical criterion called successive empirical excess risk (SEER). The SEER calculates the difference between the minimum empirical risk of two successive nested models. Our proposed nested empirical risk (NER) method uses a PAC bound on SEER to predict when the minimum risk decreases.Our proposed nested empirical risk (NER) method uses a PAC bound on SEER to predict when the minimum risk decreases.  In general, we ensure the existence of a PAC bound for the difference of empirical risks of ERM predictors in nested models.  Since these PAC bounds are proposed for the empirical criteria, they are more useful in practice than oracle generalization bounds \cite{hajiani2023oracle}.  

Now, fundamental questions will arise. Do model selection methods rely solely on feature sorting algorithms in general cases? Or, is there any way to \textit{sort the models} rather than sorting the features? 

To address these questions, we propose a novel model sorting and selection method called sorting nested empirical risk (S-NER) that sorts the models without using feature sorting algorithms like OMP or LARS. The S-NER method expands the space of the models intelligently. Indeed, we extend the model using the useful parameter space, as each model space will be nested within another one. In our method, along with sorting the models, we eliminate the redundant and useless parameter space to achieve the most parsimonious model. The S-NER method terminates when it finds the simplest model with minimum risk and does not explore the entire space of parameters. 

 As stated above,  the S-NER  model selection method can transform every set of models into a nested family of models, and this transformation will be done intelligently. In this paper, we call this type of transformation of models \textit{model sorting}. Then, we propose an approach to transform every set of models into a nested model class, and the S-NER model selection uses this approach and other properties of the nested model class to select the model with the minimum risk and simplest complexity.
 
 The S-NER model selection method using these properties proposes a test on the SEER to determine whether the minimum risk is decreased or not. Consequently, it finds the most parsimonious model containing the risk minimizer predictor. Unlike the other model selection methods, the S-NER does not require a new penalty for the empirical risk.  We show that, with high probability, the S-NER method chooses the most parsimonious model containing the risk minimizer predictor. As the sample number tends to infinity, with probability one, the S-NER  selects the simplest model model with the minimum risk.  

In addition, using the nested models' properties, we compute the SEER bounds exclusively for the linear regression problem. We show that the S-NER test using these bounds can be used in linear regression model selection, and it outperforms the state-of-the-art methods in this problem. We use the S-NER method in high-dimensional linear regression model selection using synthetic data sets. For the first time, we show that the proposed method in this paper, without any prior information, can outperform the accuracy of OMP and LARS feature sorting algorithms aided with the prior knowledge of the true model order. Note that the accuracy of sorting algorithms like OMP or LARS aided with the true model order is the upper bound for the accuracy of other model selection methods like EFIC and EBICR. In this regard,  in \cite{hajiani2026random}, we exploit the S-NER model selection for band selection of hyperspectral images (HSI), which reduces HSI classification complexity by at least 50 percent.

Also, the NER model order selection method is used as the feature selection method for the random convolutional kernels (ROCKET) features in classifying the UCR data sets \cite{dempster2020rocket,dau2019ucr}. We show that using the least-square loss prevents selecting features more than the number of observations. Therefore, we use the ridge loss function as a regularized loss in the NER method for feature selection to overcome this limitation. It is shown that this method decreases the complexity (number of features) dramatically, with a negligible loss of accuracy in the UCR classification task.  

Concisely, we present the main contributions of this paper as follows.
\begin{enumerate}
    \item We present a new empirical criterion called SEER (in Section \ref{Sec3}), and using this criterion, we propose a novel model sorting and selection method, called S-NER, to arrange the models without using the feature sorting algorithms (in Section \ref{Section50}).
    \item We present new SEER PAC bounds in the linear regression problem and use these bounds for the S-NER linear regression model selection (in Section \ref{Sec5}).
    \item For the first time in the literature on model selection,  the S-NER model selection method, without any prior information, can outperform the widespread feature sorting algorithms, OMP and LARS, that are aided with the true order of models in the high-dimensional linear regression problem (See Section \ref{Section6.1}). Note that the accuracy of aided OMP or aided LARS is the upper bound for the accuracy of other existing model selection approaches, e.g., EFIC and EBICR, that use these algorithms for feature sorting.
\end{enumerate}

The rest of the paper is structured as follows: Section \ref{Sec2} provides the notation and theoretical preliminaries on learning theory underlying model selection. Section  \ref{Sec3} presents the definitions of the nested, non-nested, and partially-nested models. Also, necessary definitions and corollaries for proposing the SEER bounds and the S-NER method are provided.  Section \ref{Sec4} presents some properties of minimum risks and minimum empirical risks in the class of nested models. We find the bounds on the difference of minimum empirical risks in nested models in this section. In addition, the NER model order selection method is proposed, and the lower bound of the correct model order selection probability is calculated. Moreover, its consistency is investigated. The S-NER model selection method is presented in Section \ref{Section50}. The lower bound on S-NER accuracy and its consistency are investigated in this section.  In Section \ref{Sec505}, we do some customization on the S-NER method for the linear regression model selection. Also, the upper and lower bounds on the SEER in linear regression models are provided. In addition, two applications of the NER model selection in the linear regression problem of the high-dimensional synthetic dataset and the UCR dataset for classification are presented in Section \ref{SEC6}, and the results are compared with other model selection methods. Finally, Section \ref{Sec7} concludes the paper. 

\section{Preliminaries}\label{Sec2}
In this section, we present the concepts and definitions regarding the nested model selection problems. Through this paper, we consider both supervised and unsupervised methods. Let $\textbf{x}_1,\textbf{x}_2,…,\textbf{x}_n$ be $n$ independent and identically distributed (i.i.d) observed data belongs to $\mathcal{X} \subset \mathbb{R}^p$. In the unsupervised scheme, a model set $ \mathcal{M}_k =\{f_{{\pmb{\theta}}_k} (\textbf{x}): {\pmb{\theta}}_k \in {\mathcal{H}}_k \} $ is the $k$-th set of probability density functions (pdf) to describe the observed data or an estimate of their distribution, where ${\pmb{\theta}}_k  (k\in\{1, 2, ..., L\})$ is the function parameter in the $k$-th model, and $ {\mathcal{H}}_k $ is the $k$-th model parameter space. 

In the case of supervised learning, $\textbf{y}_1,\textbf{y}_2,…,\textbf{y}_n$ are the members of the set $\mathcal{Y} \subset \mathbb{R}^q$ corresponding responses to the observed data,
and $\mathcal{M}_k =\{f_{{\pmb{\theta}}_k} \left(\textbf{y}|\textbf{x}\right):{\pmb{\theta}}_k \in {\mathcal{H}}_k \}$ is a set of conditional probability density functions (pdf) to describe the relationship between the observed data and the corresponding response. Also, the collection of $L$ models, $\{\mathcal{M}_k\}_{k=1}^L$, is called the class of models set. 

For every  $ \pmb{\theta} $, let $l(f_{{\pmb{\theta}}}(\textbf{x})):\mathcal{X} \to \mathbb{R}^{+}$ be a non-negative loss function of parameter $ {\pmb{\theta}} $ given observed data $ \textbf{x} $ in unsupervised methods, and $l(f_{{\pmb{\theta}}}(\textbf{x}),\textbf{y}):\mathcal{X}\times \mathcal{Y} \to \mathbb{R}^{+}$ be the loss function of parameter $ {\pmb{\theta}} $ given observed data $ \textbf{x}$ and its corresponding response $\textbf{y} $ in the supervised method. In the following, we use $l(\textbf{z},\pmb{\theta})$ instead of both $l(f_{{\pmb{\theta}}}(\textbf{x}))$ and $l(f_{{\pmb{\theta}}}(\textbf{x}),\textbf{y})$ for ease of notation in unsupervised and supervised learning methods, respectively. Also, let $\textbf{z}=\textbf{x}$ with an unknown distribution $\mathcal{P}$ over $\mathcal{X}$ in the unsupervised method, and $\textbf{z}=(\textbf{x},\textbf{y})$ be a random tuple with an unknown distribution $\mathcal{P}$ over $\mathcal{X}\times\mathcal{Y}$ in supervised methods.

Based on all the machine learning schemes, finding the predictor with the minimum risk is the main target. The risk of a function is the expected loss function and is defined as follows
\begin{equation}\label{riskEQ}
	R(\pmb{\theta})=\mathbb{E}_{\textbf{Z}} \{l(\textbf{z},\pmb{\theta})\}.
\end{equation}
\noindent Note that $  R(\pmb{\theta}) $ is not available since the data distribution is unknown. Then, based on the observed data $S_n=\left( \textbf{z}_1 ,\textbf{z}_2,…,\textbf{z}_n \right)$, we use the empirical risk function as an alternative to the risk function where it will be computed as follows
\begin{equation}\label{empEQ}
	R_{emp} (S_n,\pmb{\theta})=\frac{1}{n} \sum_{i=1}^{n} l(\textbf{z}_i,\pmb{\theta}).
\end{equation}
Let $\pmb{\theta}^*=\arg{\displaystyle\min_{{\pmb{\theta}_k \in \mathcal{H}_k , \ k\in\{1,2,...,L\}}}{R(\pmb{\theta}_k)} }$  be any parameter with the minimum risk in all models of the class.  $f_{\pmb{\theta}^* }$ is the function with the best generalization ability in the model class and is called the global risk minimizer predictor. Also, the parameters with the minimum risk in each model are as follows
\begin{equation}\label{minriskparEQ}
	{\pmb{\theta}}_k^*=\arg\min_{\pmb{\theta}_k \in \mathcal{H}_k} {R(\pmb{\theta}_k)}  , \ k\in\{1,2,...,L\}.
\end{equation}
As the risk of a parameter is unavailable, ${\pmb{\theta}}_k^*$ is also unavailable. So, based on the observed data, the empirical risk minimization (ERM) finds the parameters with the minimum empirical risk by
\begin{equation}\label{minempriskparEQ}
	\hat{\pmb{\theta}}_k = \arg\min_{\pmb{\theta}_k \in \mathcal{H}_k} {R_{emp}(S_n,\pmb{\theta}_k )} , \ k\in\{1,2,...,L\}.
\end{equation}
as an alternative to  ${\pmb{\theta}}_k^*$. Moreover, we say $f(x)=\mathcal{O}(g(x))$ if there are positive real number $T$ and real number $x_0$ such that for every $x\geq x_0$, $f(x)\leq T g(x)$. Also, we say $f(x)=o(g(x))$ if for every constant $\varepsilon>0$, there is a real number $x_0$, such that for every $x\geq x_0$, $f(x)\leq \varepsilon g(x)$ \cite{sipser1996introduction}.

The following section addresses the definitions and concepts required in this paper.  

\section{Concepts and Definitions}\label{Sec3}
Now, we focus on model selection by sorting the models based on the concept of nested models. It requires a precise specification of nested models. As a result, the definition of nested models is provided in this section.
\begin{definition}[\textbf{Nested Models}]\label{NestedDEF} 
	$\mathcal{M}_1$ is nested in $\mathcal{M}_2$ if and only if, for all $\pmb{\theta}_1 \in \mathcal{H}_1$, there is a $\pmb{\theta}_2 \in \mathcal{H}_2$ (not necessarily equal  $\pmb{\theta}_1 $), where $f_{\pmb{\theta}_1}=f_{\pmb{\theta}_2}$.
\end{definition}

\noindent Note that we say $f_{\pmb{\theta}_1}=f_{\pmb{\theta}_2}$ if for every $\textbf{z}\in \Xc \times \Yc $ in the supervised scheme (or $\textbf{z}\in \Xc $ for the unsupervised scheme), $f_{\pmb{\theta}_1}(\textbf{z})= f_{\pmb{\theta}_2}(\textbf{z})$. It is worth noticing that in \cite{pesaran1978testing}, the authors use Kullback-Leibler divergence to determine whether two probability distribution functions are equal or not.  

Now, we can define other class types of models. We can define two models that could be non-nested or partially nested. In the following, the definitions of non-nested and partially-nested models are presented.
\begin{definition}[\textbf{Partially Nested Models}]\label{PartiallyNestedDEF} 
	$\mathcal{M}_1$ is partially nested with $\mathcal{M}_2$ if and only if there are $\pmb{\theta}_1 \in \mathcal{H}_1$ and $\pmb{\theta}_2 \in \mathcal{H}_2$ ($\pmb{\theta}_2$ not necessarily equals $\pmb{\theta}_1 $), where $f_{\pmb{\theta}_1}=f_{\pmb{\theta}_2}$.
\end{definition}
\begin{definition}[\textbf{Non-Nested Models}]\label{Non-NestedDEF} 
	$\mathcal{M}_1$ and $\mathcal{M}_2$ are non-nested if and only if for all $\pmb{\theta}_1 \in \mathcal{H}_1$ and $\pmb{\theta}_2 \in \mathcal{H}_2$, $f_{\pmb{\theta}_1}\neq f_{\pmb{\theta}_2}$.
\end{definition}
Based on the nested and partially nested definitions, if $\mathcal{M}_1$ is nested in $\mathcal{M}_2$, $\mathcal{M}_2$  is either nested in $\mathcal{M}_1$ or partially nested with $\mathcal{M}_1$. Similarly, if $\mathcal{M}_1$ is partially nested with $\mathcal{M}_2$, $\mathcal{M}_2$ is either nested in $\mathcal{M}_1$ or partially nested with $\mathcal{M}_1$. However, based on the non-nested definition,  two non-nested models are always mutually non-nested.

Definition \ref{NestedDEF} is only appropriate for two models; however, we frequently deal with a class of more than two models.  Hence, the following definition will be presented for the class of sequentially nested models.

\begin{definition}[\textbf{Sequentially Nested Model Class}]\label{SNDEF}
	A model class $\{\mathcal{M}_k \}_{k=1}^L$ is a sequentially Nested class if for $L>2$, every $i,j \in \{1,2,…,L\},$ where if $i \leq j$, $\mathcal{M}_i$ is nested in $\mathcal{M}_j$.
\end{definition}
Notice that every set of models (consisting of nested, partially nested, and non-nested models) can be combined and arranged to achieve a sequentially nested model family of the models.  We call the procedure of generating nested model class the \textit{nesting process}. 

An approach to arranging and combining every set of models to make a sequentially family of models is discussed in the following corollary. 
\begin{corollary}[\textbf{Nesting Process}]\label{SNgeneratingcorollary}
    Let $\{\bar{\Mc}_k\}_{k=1}^L$ be an arbitrary set of models. Then, let $\mathcal{M}_1=\bar{\Mc}_1$ and for every $k\in\{1,2,...,L-1\}$,  $\mathcal{M}_{k+1}=\mathcal{M}_{k} \cup \bar{\Mc}_{k+1}$.  The model family $\{\Mc_k\}_{k=1}^L$ is  sequentially nested.
\end{corollary}
\begin{proof}
    It can be seen that because of using the union operator for every $ k\in\{1,2,...,L-1\}$, $\mathcal{M}_k$ is nested in $\mathcal{M}_{k+1}$.
\end{proof}
 There are several ways to generate a sequentially nested model family with different arrangements. However, a question that emerges is which one of these arrangements is more useful, and does every model arrangement lead to obtaining the most parsimonious model containing the risk minimizer predictor? 
 
Viewing model selection as complexity management, we seek the simplest model that contains the risk-minimizing predictor. We show that a way to manage the complexity is to sort the models in an intelligent way just to consider the relevant parameters space of the models. However, the sorting of the models is challenging. In this paper, we present a procedure that uses sorting information to obtain the most parsimonious model.

First, to address these challenges, we need to consider some characteristics of the sequentially nested models family. In the following corollary, we present an important side of sequentially nested model families.
\begin{corollary}\label{KASSUMPTION}
	 In every sequentially nested model family $\{\mathcal{M}_k \}_{k=1}^L$,  there is a model index $1< K \leq L$ such that a global risk minimizer predictor is in the $K$-th model, $f_{\pmb{\theta}^*} \in \mathcal{M}_K$ and $f_{\pmb{\theta}^*} \notin \mathcal{M}_{K-1}$, otherwise $f_{\pmb{\theta}^*} \in \mathcal{M}_1$, i.e. $K=1$.
\end{corollary}
\begin{proof}
    Since $\{\mathcal{M}_k \}_{k=1}^L$ is a sequentially nested model family, any global risk minimizer predictor $f_{\pmb{\theta}^*} $ is always a member of $\mathcal{M}_L$. Therefore, we investigate $\mathcal{M}_{L-1}$. If there is not any  $f_{\pmb{\theta}^*} $ in  $\mathcal{M}_{L-1}$, we have $K=L$. Otherwise, an $f_{\pmb{\theta}^*} $ is in $\mathcal{M}_{L-1}$, and we check  $\mathcal{M}_{L-2}$. Similarly, if there is no  $f_{\pmb{\theta}^*} $ in  $\mathcal{M}_{L-2}$, we have $K=L-1$. Otherwise, an $f_{\pmb{\theta}^*} $ is in $\mathcal{M}_{L-2}$, and we check $\mathcal{M}_{L-3}$. We repeat this procedure until model  $\mathcal{M}_{2}$ is investigated. If  $\mathcal{M}_{2}$ contains $f_{\Ptheta^*}$, we check $\Mc_1$. Then, If $f_{\Ptheta^*} \notin \Mc_1$, we have $K=2$. Otherwise, $f_{\Ptheta^*}\in \Mc_1$, and $K=1$. These statements show that there is an index $K$, $1\leq K \leq L$, that $\Mc_K$ contains the risk minimizer model, and this model is not in $\Mc_{K-1}$.
\end{proof}
Corollary \ref{KASSUMPTION} shows that the $K$-th model is the \textit{most parsimonious model} of a sequentially nested model family $\{\mathcal{M}_k \}_{k=1}^L$ that contains the risk minimizer, and the model selection procedure aims to find $K$. However, $K$ is related to the arrangement of the sequentially nested model family. An optimum arrangement leads to the minimum available $K$. It means that the intelligent model arranging causes model order reduction or reduces the complexity of the desired model with the minimum risk. 

Achieving the risk minimizer is the target of model selection, and degradation in minimum risk is a clue to achieving this minimum risk. However, this risk, due to a lack of knowledge about the probability distribution of the observed data, is not generally available. Hence, we use empirical criteria to recognize whether the minimum risk in a sequentially nested model family is decreased or not. 
 
For two nested models $\mathcal{M}_i$ and $\mathcal{M}_j$ where $\mathcal{M}_i$ is nested in $\mathcal{M}_j$, we introduce the difference between the minimum empirical risk of these models as $\Delta R_{emp}(S_n,i,j)=R_{emp}(S_n,\hat{\pmb{\theta}}_i)-R_{emp}(S_n,\hat{\pmb{\theta}}_j)$. In Subsection \ref{Section4a}, we will show that this criterion is non-negative.  Also, in the sequentially nested model family $\{\mathcal{M}_k\}_{k=1}^L$, we call the difference between the minimum empirical risk of two successive models, $\Delta R_{emp}(S_n,k)=R_{emp}(S_n,\hat{\pmb{\theta}}_{k-1})-R_{emp}(S_n,\hat{\pmb{\theta}}_{k})$, \textit{successive empirical excess risk (SEER)}.  We will show that SEER has an important role in model selection.

Analyzing the behavior of the difference between empirical risks and risks in a model helps to understand the statistical behavior of the SEER. 

Now, we introduce a Glivenko-Cantelli functions class that will be useful in our considerations in the sequel.
\begin{definition}[\textbf{Glivenko-Cantelli Function Class \cite{wainwright2019high}}]\label{GCdef}
    The class of functions $\mathcal{F}$ is a Glivenko-Cantelli class for a distribution $\mathcal{P}$ $(\textbf{Z} \sim \mathcal{P})$ if $\sup_{f\in\mathcal{F}} |\frac{1}{n}\sum_{i=1}^{n} f(\textbf{z}_i) - \mathbb{E}\{f(\textbf{Z})\} |$ converges to zero in probability as $n\to \infty$, i.e., for every $\varepsilon>0$,
    \begin{equation}
        \lim_{n\to \infty}\mathbb{P}\Big\{ \sup_{f\in\mathcal{F}} \Big| \frac{1}{n}\sum_{i=1}^{n} f(\textbf{z}_i) - \mathbb{E}\{f(\textbf{Z})\} \Big|>\varepsilon\Big\}=0.
    \end{equation}
\end{definition}
The definition of the Glivenko-Cantelli class of functions states that the concentration of every function around its mean converges to zero, asymptotically.  Also, Theorem 4.10 in \cite{wainwright2019high} represents a bound based on the Rademacher complexity for bounded function classes. It is shown that if the Rademacher complexity of the class of bounded functions is in the order of $o(1)$ (for $n\to \infty $,  the Radmacher complexity of the class tends to zero), this class of functions is Glivenko-Cantelli class. 

 Building on the concepts introduced in this section, the next section provides a detailed exposition of the proposed method. We first examine the properties of sequentially nested model families and derive bounds for nested models, and then introduce our novel model selection approach.

\section{Model Selection in Nested Model Families}\label{Sec4}
\sloppy As discussed in the previous sections, ${\pmb{\theta}}_k^*$ and $\hat{\pmb{\theta}}_k$ are crucial parameters in the model selection problems. Corresponding to these parameters in the model class $\{\mathcal{M}_k\}_{k=1}^L$, we have sequences of the minimum risks $R({\pmb{\theta}}_1^* ),R({\pmb{\theta}}_2^* ),…,R({\pmb{\theta}}_L^* )$, and the minimum empirical risks $R_{emp}(S_n,\hat{\pmb{\theta}}_1 ),R_{emp}(S_n,\hat{\pmb{\theta}}_2 ),…,R_{emp}(S_n,\hat{\pmb{\theta}}_L )$.  In the sequentially nested model families, these sequences have special properties that we will use in the \textit{nested empirical risk (NER)} model selection methods. We present these properties in the next subsection.
\subsection{Properties of Minimum Risks and Minimum Empirical Risks in Nested Families\label{Section4a} }
In this subsection, first, we present the non-increasing property of sequences of minimum risks and minimum empirical risks in the sequentially nested families.
\begin{lemma}\label{lem1}
	The sequences of minimum risks and minimum empirical risks in a sequentially nested model class are non-increasing.
\end{lemma}
\begin{proof}
	In a sequentially nested class, for every $i,j\in\{1,2,…,L\},i\leq j$, $\mathcal{M}_i$ is nested in $\Mc_j$. Based on the nested model Definition \ref{NestedDEF}, for any risk minimizer function $f_{{\pmb{\theta}_i^*}} $ in $\mathcal{M}_i$, there is a function $f_{\bar{\pmb{\theta}}_j} $ in $\mathcal{M}_j$, where $f_{{\pmb{\theta}_i^*}}=f_{\bar{\pmb{\theta}}_j} $ and $R({\pmb{\theta}_i^*})=R({\bar{\pmb{\theta}}_j})$. Since  $f_{\bar{\pmb{\theta}}_j}$ is in $\mathcal{M}_j$,  the minimum risk in $\Mc_j$ is not greater than $R({\bar{\pmb{\theta}}_j})$, i.e. $R({\pmb{\theta}_j^*})\leq R({\bar{\pmb{\theta}}_j})$, and consequently $R({\pmb{\theta}_j^*})\leq R({\pmb{\theta}_i^*})$.
	
	Similarly, for the empirical risk minimizer function $f_{\hat{\pmb{\theta}_i}} $ in $\mathcal{M}_i$, there is a function $f_{\tilde{\pmb{\theta}}_j} $ in $\mathcal{M}_j$, where $f_{\hat{\pmb{\theta}_i}}=f_{\tilde{\pmb{\theta}}_j}$ and $R_{emp}(S_n,\hat{\pmb{\theta}}_i)=R_{emp}(S_n,{\tilde{\pmb{\theta}}_j})$. Since $f_{\tilde{\pmb{\theta}}_j}$ is in $\mathcal{M}_j$, the minimum empirical risk in $\Mc_j$ is not greater than $R_{emp}(S_n,{\tilde{\pmb{\theta}}_j})$, i.e. $R_{emp}(S_n,{\hat{\Ptheta}}_j)\leq R_{emp}(S_n,{\tilde{\pmb{\theta}}_j})$, and consequently  $R_{emp}(S_n,\hat{\pmb{\theta}}_j)\leq R_{emp}(S_n,\hat{\pmb{\theta}}_i)$.
\end{proof}
Moreover, in the sequentially nested model families, the sequence of minimum risk has an elbowing property, where the following lemma introduces this property. It is an exclusive property for the sequentially nested model class and generally does not hold for every model class.
\begin{lemma}\label{lem2}
	For the sequentially nested model family $\{\mathcal{M}_k \}_{k=1}^L$, the following hold for every $k\geq K$ and every $i< K$,
	\begin{flalign}
		R({\pmb{\theta}}_{k}^*)= R(\pmb{\theta}^*),\quad
		R({\pmb{\theta}}_{i}^*)> R(\pmb{\pmb{\theta}}^*).\label{eq11_2}
	\end{flalign}
\end{lemma}
\begin{proof}
	Based on Corollary \ref{KASSUMPTION}, the global risk minimizer $f_{\pmb{\theta}^*}$ is an element of $\mathcal{M}_K$. Therefore, in the sequentially nested model family $\{\mathcal{M}_k \}_{k=1}^L$, for $K \leq k \leq L$,  $f_{\pmb{\theta}^*}$ is also an element of $\mathcal{M}_k$, and the minimum risk of all these models equals the global minimum risk, i.e., $R(\pmb{\theta}^* ) = R({\pmb{\theta}}_k^* )$.
 
Now, based on Corollary \ref{KASSUMPTION}, there is not any global risk minimizer predictor $f_{\pmb{\theta}^*}$ in $\mathcal{M}_{K-1}$. Consequently, in the sequentially nested model family $\{\mathcal{M}_i \}_{i<K}$, there is not any global risk minimizer predictor $f_{\pmb{\theta}^*}$ in $\mathcal{M}_{i}$, and the minimum risk of these models is greater than $R(\pmb{\theta}^*)$, i.e., $R({\pmb{\theta}}_{i}^*) >R(\pmb{\theta}^*)$.
\end{proof}
Lemma \ref{lem2} illustrates the role of risk minimizer predictor in the sequence of minimum risks in sequentially nested model families. In the following Corollary, we show the trend of the sequence of minimum risks in the sequentially nested model classes.
\begin{corollary}\label{trendcorollary}
	For the sequentially nested model class $\{\mathcal{M}_k\}_{k=1}^L$, the following trend in the sequence of minimum risks holds
	\begin{equation}\label{17}
		R({\pmb{\theta}}_1^*) \geq R( {\pmb{\theta}}_2^*) \geq ... \geq R( {\pmb{\theta}}_{K-1}^*) > R({\pmb{\theta}}_K^*)=...=R( {\pmb{\theta}}_L^*).
	\end{equation}
\end{corollary}
\begin{proof}
	Simply, by combining the results in Lemmas \ref{lem1} and \ref{lem2}, \eqref{17} will be obtained.
\end{proof}
As shown in \eqref{17}, the sequence of minimum risks exhibits an elbowing property. It is non-increasing, with successive terms eventually becoming equal. Consequently, according to this corollary, extending the model beyond $\mathcal{M}_K$ to any $\mathcal{M}_{K+1},...,\mathcal{M}_{L}$  yields no benefit, as the minimum risk no longer decreases. 

The following subsection proposes upper and lower bounds on the difference between the models' minimum empirical risks and the SEER.
\subsection{Probably Approximately Correct (PAC) Bounds on the SEER}\label{Section4b}
 In this subsection, we obtain the PAC bounds on the SEER using the properties of the Glivenko-Cantelli function class. We use these bounds in the proposed model selection method called the \textit{nested empirical risk (NER)}. In this regard,  the following theorem proves the existence of the $o(1)$ lower and upper bounds on the difference of the minimum empirical risks of two nested models in two states, i.e., the state with different minimum risks and the state with the same minimum risks, respectively.
 \begin{theorem}\label{LboundTH}
	Consider two models,  $\mathcal{M}_i $ and $\mathcal{M}_j$,  where $\mathcal{L}_i=\{l(.,\Ptheta_i):\Ptheta_i\in \Hc_i\}$ and $\mathcal{L}_j=\{l(.,\Ptheta_j):\Ptheta_j\in \Hc_j\}$ are Glivenko-Cantelli classes of loss functions. Let $\mathcal{M}_i $ be nested in $\mathcal{M}_j$, so $R(\Ptheta^*_i)\geq R(\Ptheta^*_j)$. Therefore, there is an $o(1)$ function  $\gamma_{i,j}\left(S_n,{\delta}\right)$, where
	\begin{itemize}
		\item{ \textbf{State 1}: For $R(\pmb{\theta}^*_i)>R(\pmb{\theta}^*_j)$ and every $0<\delta\leq 1$, there is a positive integer $N_1$, such that for every $n\geq N_1$, with the probability of at least $1-\delta$, the following holds\begin{equation}\label{minusKboundEQ}
				R_{emp}(S_n,\hat{\pmb{\theta}}_{i})-R_{emp}(S_n,\hat{\pmb{\theta}}_{j})  \geq \gamma_{i,j}\left(S_n,{\delta}\right)  ,
		\end{equation}}
		\item{\textbf{State 2}: For $R(\pmb{\theta}^*_i)=R(\pmb{\theta}^*_j)$, there is a positive integer $N_2$, such that for every $n\geq N_2$,  $0<\delta\leq 1$,  with the probability of at least $1-\delta$, the following holds
			\begin{equation}\label{empexcessERMINEQ}
				R_{emp}(S_n,\hat{\pmb{\theta}}_{i})-R_{emp}(S_n,\hat{\pmb{\theta}}_{j})  \leq \gamma_{i,j}\left(S_n,{\delta}\right).
		\end{equation}}
	\end{itemize}
\end{theorem}
\begin{proof}
   See appendix \ref{APPA}.
\end{proof}

In state 1 of Theorem \ref{LboundTH}, we have shown that for two nested models $\mathcal{M}_i$ and $\mathcal{M}_j$ where $\Mc_i$ is nested in $\Mc_j$, if the minimum risk of $\mathcal{M}_j$ will be smaller than the minimum risk of $\mathcal{M}_i$, with a high probability, the difference between minimum empirical risks will be lower bounded by $\gamma_{i,j}(S_n,{\delta})$. 

Also, in state 2 of Theorem \ref{LboundTH}, we show that if the minimum risk of these nested models is the same, with a high probability, the difference between minimum empirical risks will be upper bounded by $\gamma_{i,j}(S_n,{\delta})$. Based on Fig. \ref{Oracle_SEER_fig}, we show that the difference in minimum risks (oracle space) relates to the difference between minimum empirical risks (empirical space) in two nested models. Therefore, for two models,  $\Mc_i$ and $\Mc_j$ where $\Mc_i$ is nested in $\Mc_j$, we can use $\gamma_{i,j}(S_n,{\delta})$ as a threshold on the difference of minimum empirical risks to predict whether their\textit{ minimum risk} is the same or not. 

 \begin{figure}
	\centering
	\adjincludegraphics[height=5.5cm,width=16.5cm,clip]{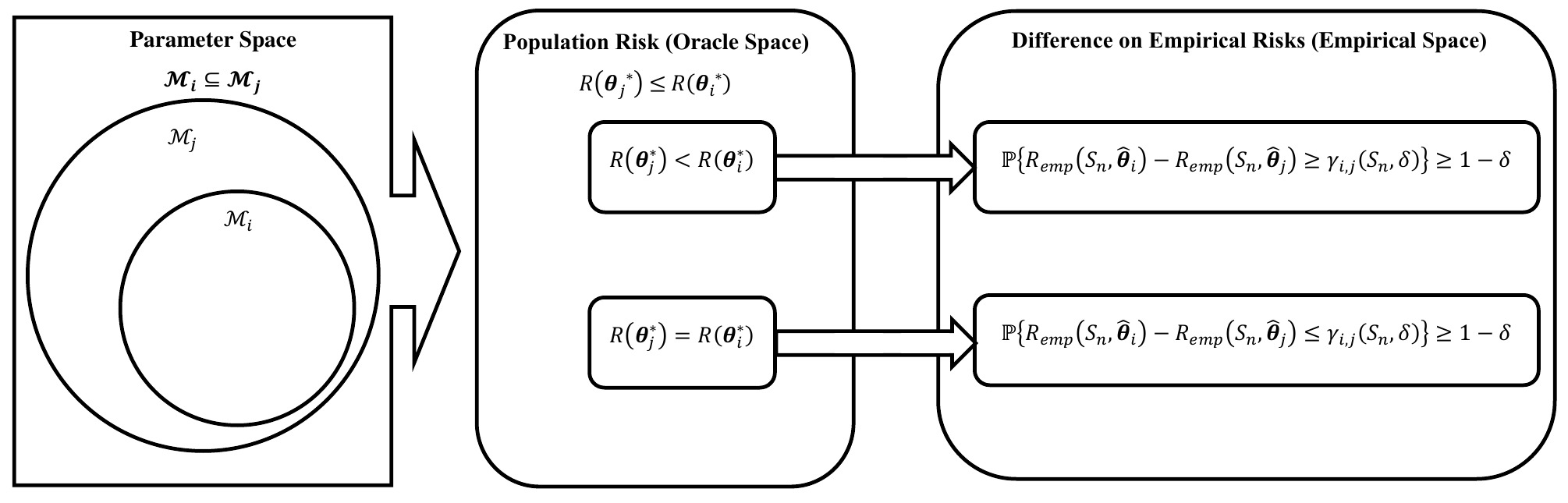}
	\caption{Tracking population risk change in two nested models (oracle space) based on differences in empirical risks (empirical evidence).}\label{Oracle_SEER_fig}
\end{figure}

In the sequel, for ease of notation, in a sequentially nested model family $\{\Mc_k\}_{k=1}^L$, for every $k\in\{2,...,L\}$, we use $\gamma_{k}(S_n,{\delta})$ instead of $\gamma_{k-1,k}(S_n,{\delta})$ as an $o(1)$  upper bound on $| R_{emp}(S_n,\hat{\pmb{\theta}}_{k-1})- R_{emp}(S_n,\hat{\pmb{\theta}}_k)|$ for the two successive nested models $\Mc_{k-1}$ and $\Mc_{k}$.

 In the following corollary, using Theorem \ref{LboundTH}, we prove the existence of $o(1)$ upper and lower bounds for the SEER in the sequentially nested model families.
\begin{corollary}[\textbf{SEER PAC bounds}]\label{NestedboundTH}
	Consider the sequentially nested model class  $\{\mathcal{M}_k \}_{k=1}^L$ where for every $k\in\{1,2,...,L\}$,  $\mathcal{L}_k=\{l(.,\Ptheta_k):\Ptheta_k\in \Hc_k\}$ is Glivenko-Cantelli class of loss functions. Therefore, for every $k\in\{ 2,3,...,L\}$, there is an $o(1)$   function $\gamma_k(S_n,{\delta})$,  where 
	\begin{enumerate}
		\item{  For every $0<\delta\leq 1$, there is a positive integer $N_K\geq1$ such that for every $n\geq N_K$, with the probability of at least $1-\delta$, the following holds
			\begin{equation}\label{minusKboundEQ11}
				R_{emp}(S_n,\hat{\pmb{\theta}}_{K-1})-R_{emp}(S_n,\hat{\pmb{\theta}}_{K})  \geq \gamma_{K}\left(S_n,{\delta}\right) .
		\end{equation}}
		\item{For every $0<\delta\leq 1$,  and every $k>K$, there is a positive integer  $N_k$, such that  $n\geq N_k$, with the probability of at least $1-\delta$, the following holds
			\begin{equation}\label{empexcessERMINEQ11}
				R_{emp}(S_n,\hat{\pmb{\theta}}_{k-1})-R_{emp}(S_n,\hat{\pmb{\theta}}_{k})  \leq \gamma_{k}\left(S_n,{\delta}\right).
		\end{equation}}
	\end{enumerate}
\end{corollary}
\begin{proof}
	Since $\{\mathcal{M}_k \}_{k=1}^L$  is a sequentially nested model class, based on Lemma \ref{lem2}, $R(\pmb{\theta}^*_{K-1})>R(\pmb{\theta}^*_{K})$. Then, state 1 of Corollary \ref{NestedboundTH} satisfies all conditions of state 1 of Theorem \ref{LboundTH}, and based on that, for every  $0<\delta\leq 1$, there is an integer $N_K$ such that for every $n\geq N_K$, with the probability of $1-\delta$, \eqref{minusKboundEQ11} holds. 
 
 In addition, for every $k>K$, based on Lemma \ref{lem2}, $R(\pmb{\theta}^*_{k-1})=R(\pmb{\theta}^*_{k})$.  Now, it is easy to know that state 2 of Corollary \ref{NestedboundTH} satisfies all conditions of state 2 of the Theorem \ref{LboundTH}, and based on that, there is an integer $N_k$ such that for every  $n\geq N_k$ and $0<\delta\leq 1$,  with the probability of at least $1-\delta$, \eqref{empexcessERMINEQ11} holds.
\end{proof}
 Corollary \ref{NestedboundTH} proves the existence of  $o(1)$ upper and lower bounds on the SEER (In Subsection \ref{Sec5}, we will calculate the SEER bounds for the linear regression problem as an important example of a learning problem).  

 Now, for every $k\in\{2,3,...,L\}$, let $\Delta R_{emp}(S_n,k)=R_{emp}(S_n,\hat{\Ptheta}_{k-1})-R_{emp}(S_n,\hat{\Ptheta}_k)$. We can rewrite \eqref{minusKboundEQ11} as follows
\begin{equation}
	\mathbb{P}\{ \Delta R_{emp}(S_n,K)\geq \gamma_{K}\left(S_n,{\delta}\right) \}\geq 1-\delta. \label{Reverscound}
\end{equation}
This inequality states that for some small values of $\delta$, it is very likely to  $ \Delta R_{emp}(S_n,K)\geq \gamma_{K}\left(S_n,{\delta}\right)$. Also,
for $k > K$, \eqref{empexcessERMINEQ11} can be rewritten as follows
\begin{equation}\label{DELTAEQ}
	\mathbb{P}\{{\Delta R}_{emp}(S_n,k)\geq \gamma_{k}\left(S_n,{\delta}\right) \}\leq \delta.
\end{equation}
 This inequality states that if the model index $k$ is greater than $K$, for some small $\delta$ values, it is very unlikely to  $ {\Delta R}_{emp}(S_n,k)\geq\gamma_{k}\left(S_n,{\delta}\right) $. In the following subsection, we use these facts for the proposed model selection method.
\subsection{Nested Empirical Risk (NER) Model Order Selection } \label{SEC4.1}
In this subsection, we show how our proposed method, called nested empirical risk (NER), uses the SEER PAC bounds to estimate the model order from a prepared sequentially nested model family. 
Consider a sequentially nested model family $\{\mathcal{M}_k\}_{k=1}^L$,  where $\mathcal{M}_k=\{f_{\Ptheta_k}: {\Ptheta_k}\in\Hc_k\}$. It is unknown which model in this sequentially nested model family contains the global risk minimizer predictor. In other words, in this family, the model index $K$ is unknown. In this regard, we will use inequalities \eqref{Reverscound} and \eqref{DELTAEQ} to estimate the model index $K$. Inequality \eqref{Reverscound} states that for the model index $K$, it is very likely that $ \Delta R_{emp}(S_n,K)$ will be greater than $\gamma_{K}\left(S_n,{\delta}\right)$. Also,  based on inequality  \eqref{DELTAEQ}, it is very unlikely that $ \Delta R_{emp}(S_n,K+1),\Delta R_{emp}(S_n,K+2),...,\Delta R_{emp}(S_n,L)$ will be greater than $\gamma_{K+1}\left(S_n,{\delta}\right),\gamma_{K+2}\left(S_n,{\delta}\right),...,\gamma_{L}\left(S_n,{\delta}\right)$, respectively. So, for every $k$, we use $ \gamma_k \left(S_n,{\delta}\right) $ as a threshold for $  {\Delta R}_{emp}(S_n,k)  $. Therefore, for a model index $k$, if $  {\Delta R}_{emp}(S_n,k)  $ is greater than $ \gamma_{k}\left(S_n,{\delta}\right) $, with a high probability, we can predict that model index $k$ is not greater than $K$, and based on Corollary \ref{KASSUMPTION}, $\mathcal{M}_{k-1}$ does not contain $f_{\pmb{\theta}^*}$.

Now, for every $k\in\{2,3,...,L\}$, we propose the following test to predict the absence of  $f_{\pmb{\theta}^*}$ in $\mathcal{M}_{k-1}$,
\begin{equation}\label{TESTEQ}
	T_k=\mathbb{I}\{{\Delta R}_{emp}(S_n,k)\geq\gamma_{k}\left(S_n,\delta\right)\},
\end{equation}
where $\mathbb{I}$ is an indicator function such that if its argument is true, the output equals $1$, and if its argument is false, the output equals $0$. Based on Corollary \ref{NestedboundTH}, if $ T_k=1 $, with high probability, we can have $R(\pmb{\theta}^*_{k-1})>R(\pmb{\theta}^*_k)$. As a consequence, we predict that $f_{\pmb{\theta}^*}$ is not an element of $\mathcal{M}_{k-1}$. Also, if $ T_k=0 $, we can guess that  $R(\pmb{\theta}^*_{k-1})=R(\pmb{\theta}^*_k)$. Therefore, for the sequentially nested model class $\{\mathcal{M}_k\}_{k=1}^L$, we can estimate $K$ as follows 
\begin{flalign}\label{KHATEQ}
	\hat{K}&=\max \{k\in\{1,2,...,L\}:T_k=1\},
\end{flalign}
 where in the NER model order selection, we let $T_1=1$ since, based on Corollary \ref{KASSUMPTION}, we have $K\geq 1$.  The maximum operator is employed to account for the possibility that the true model index may be $k< K$  where $  {\Delta R}_{emp}(S_n,k)  $ is less than $\gamma_{k}\left(S_n,\delta \right) $, and $T_k=0$. But, based on \eqref{Reverscound} and \eqref{DELTAEQ},  it is very likely that $T_K=1$, and for  $k>K$, it is very unlikely that $T_k=1$. 
 
 Fig. \ref{NER_fig} illustrates how the nested empirical risk method selects the model order.
 In the following theorem, we calculate a lower bound on the probability of the correct estimate of the model order. 

\begin{figure}
	\centering
	\includegraphics[height=8cm,width=16.5cm,clip]{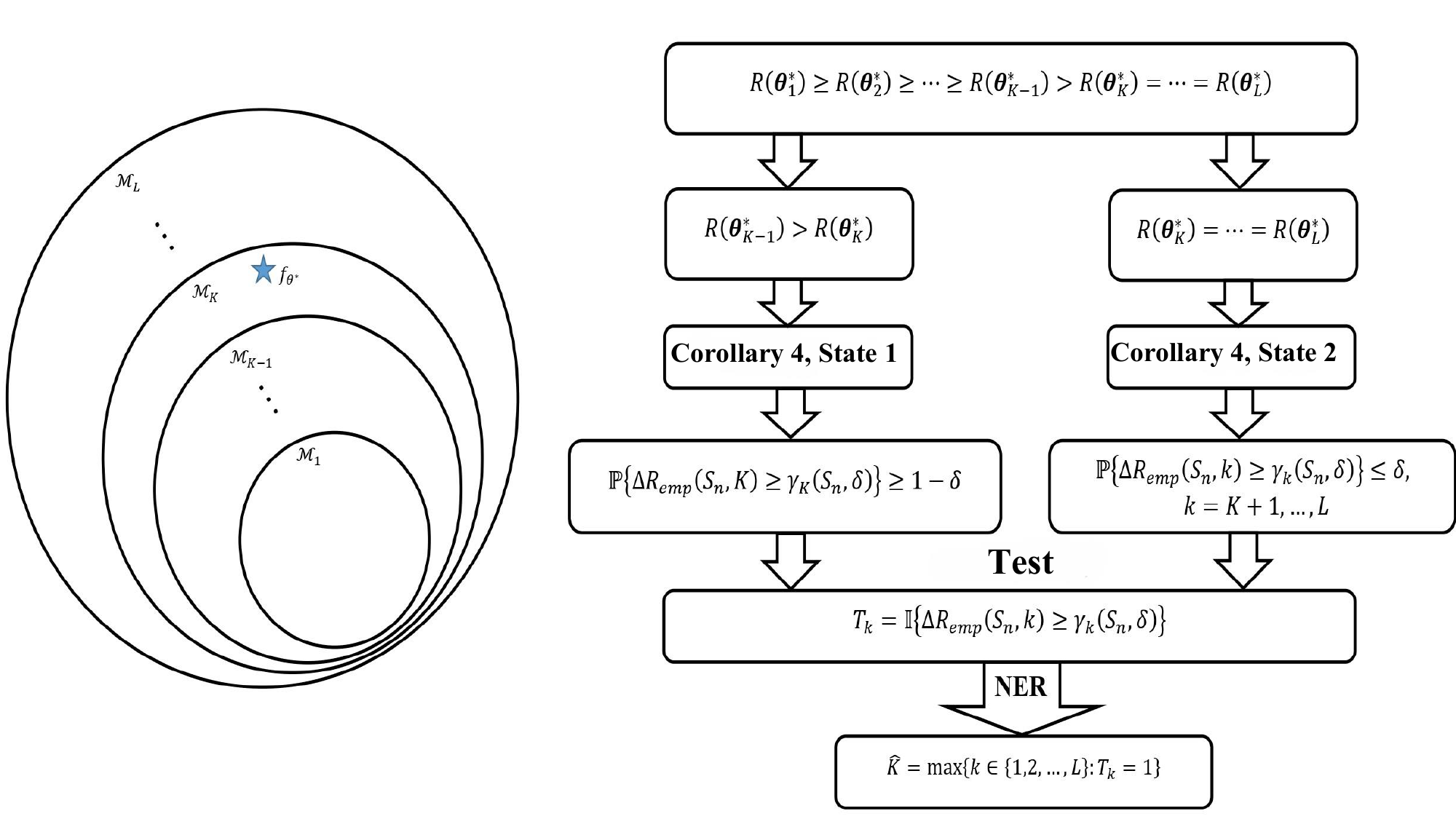}
	\caption{Block diagram of the nested empirical risk model order selection method. }\label{NER_fig}
\end{figure}
 
\begin{theorem}\label{PCMST} 
Let $\{\mathcal{M}_k\}_{k=1}^L$ be a sequentially nested model family, $\mathcal{L}_k=\{l(.,\Ptheta_k):\Ptheta_k\in \Hc_k\}, k\in\{1,2,...,L\}$ be corresponding Glivenko-Cantelli classes of loss functions, and $\hat{K}$ be calculated based on \eqref{KHATEQ}. Therefore, for every $0<\delta<1/(L-K+1)$, there is a positive integer $N$, where for $n\geq N$, we have 
	\begin{equation}
 \mathbb{P}\{\hat{K}=K\}\geq 1-(L-K+1)\delta.
	\end{equation}
\end{theorem}
\begin{proof}
See Appendix \ref{APPB}.
\end{proof}
This theorem implies that for some small $\delta$, the NER method can correctly predict the model order with a high probability of at least $1-(L-K+1)\delta$.

 Note that, for $L\geq \frac{1}{\delta}+K-1$, the lower bound $1- (L-K+1)\delta$ will be non-positive. In this case, this bound will be trivial, and it does not have any information. This case will mostly happen in high-dimensional scenarios. Usually, in high-dimensional scenarios, as prior information, it is assumed that there is a maximum model order $k_{max}$, where $K<k_{max}\ll L$. In this case, we investigate only the first $k_{max}$ models of the sequentially nested model family $(\{\Mc_k\}_{k=1}^{k_{max}})$, and we use $k_{max}$ instead of $L$ in the lower bound of the correct decision. For a small $k_{max}\ll \frac{1}{\delta}+K-1$, the upper bound will be informative.
In the following corollary, using Theorem \ref{PCMST}, we show that the NER method is consistent.
\begin{corollary}[\textbf{Consistency of the NER Model Order Selection}]\label{consistencycorollary}
Let the conditions of Theorem \ref{PCMST} hold. Then, as $n\to\infty$, $\mathbb{P}\{\hat{K}=K\}=1$ where $\hat{K}$ is calculated based on \eqref{KHATEQ}.
\end{corollary}
\begin{proof}
	Based on \eqref{PKHATKEQ3}, for every  $0<\delta^{\prime}<1/(L-K+1)$, there is a positive integer $N$ where for every $n\geq N$, we have
 \begin{equation}\label{constemp1}
     \mathbb{P}\{\hat{K}=K\}\geq1-(L-K+1)\delta^{\prime}.
 \end{equation} 
Now, let $\delta = (L-K+1)\delta^{\prime}$. Then, for every $0<\delta\leq1$, there is a positive integer $N$ where for every $n\geq N$, we have
\begin{equation}\label{constemp12}
     \mathbb{P}\{\hat{K}\neq K\}<\delta.
 \end{equation} 
 Based on \eqref{constemp12}, the following holds
	\begin{flalign}
		\lim_{n\to \infty} \mathbb{P}\{\hat{K}\neq K\}=0. \label{deltatempEQ}
\end{flalign}
Then, using \eqref{deltatempEQ},  we have $\lim_{n\to \infty} \mathbb{P}\{\hat{K}= K\}=1$.
\end{proof}
In this section, we considered the properties of the sequences of minimum risks and minimum empirical risks and the bounds of the SEER. We proposed the NER method and used it to select the order of a sequentially nested model family. The NER method does not provide an additive penalty to the minimum empirical risk. This method, using test \eqref{TESTEQ}, predicts the true model order. Also, we have shown that the NER model order selection method is a consistent predictor.

In the next section, we will show how to use the properties of the nested models intelligently to achieve the most parsimonious model in general classes of models.
\section{Model Sorting and Selection}\label{Section50}
In this section, we develop a framework for identifying the most parsimonious model that contains the risk minimizer. The proposed approach constructs a sequence of nested models by progressively expanding their parameter spaces, thereby ensuring that each model is a subset of the next. This nested structure enables us to exploit the theoretical properties of model inclusion for complexity control and model selection. We term this procedure \textit{model sorting}. 

Note that we can transform every \textit{arbitrary set} of models $\{\bar{\Mc}_k\}_{k=1}^L$ into a sequentially nested model family.  We let  $\mathcal{M}_1=\bar{\mathcal{M}}_1$, and for every $1\leq k < L$,  $\mathcal{M}_{k+1}=\mathcal{M}_{k} \cup \bar{\mathcal{M}}_{k+1}$.  Based on Corollary \ref{SNgeneratingcorollary}, the model family $\{\Mc_k\}_{k=1}^L$ is sequentially nested. Using this corollary, we can transform every arbitrary set of models into a sequentially nested model family. We call the process of transforming a set of models into a sequentially nested model family the \textit{nesting process}. 

In this section, we perform the model arrangement intelligently such that the valuable parameters have a more prominent role in the model selection procedure. To arrange the models intelligently, first, we try to use the most valuable parameters and then eliminate the useless parameters in model selection. We call this procedure of intelligent arrangement of nested models and using the most valuable parameters, nested model sorting, or in brief, \textit{model sorting}.
 
In the model sorting procedure, we need a criterion to determine whether a model expansion is useful or not. Note that we aim to achieve the risk minimizer predictor in the whole model parameter space, called the global minimum risk predictor. Therefore, a model expansion that reduces the local minimum risk (minimum risk in the special model space) helps us to get closer to the global minimum risk predictor.  We use a criterion to investigate the degradation of the minimum risk through the model expansion.  We say a model expansion is useful if the local minimum risk will be decreased. However, among all parameters, we may have many candidates for model expansion that decrease the minimum risk.  Since our objective is to prioritize the most informative parameters, we select, among all candidate model expansions, the one that yields the largest reduction in the local minimum risk. This strategy ensures that the most valuable parameter directions are incorporated first. Consequently, the resulting sequence of nested models is ordered according to their contribution to risk reduction, allowing the most useful model spaces to be explored at the earliest stages of the expansion process.  

However, since the risk function is unavailable, we do not know whether the minimum risk has changed or not. For this reason, we use the test \eqref{TESTEQ} to determine the possible changes of minimum risk through the model expansion. 

In the following, we propose an iterative model sorting procedure to achieve the most parsimonious model containing a risk minimizer.  We call this procedure the \textit{sorted NER} (S-NER) model selection method. 

It is worth noting that the S-NER is based on the expansion of the models based on the test \eqref{TESTEQ} to approve or decline this expansion procedure. In other words, this algorithm examines whether the present and the expanded models can be considered successive models in a model sorting method. Now, we describe the sorted nested empirical risk (S-NER) model selection procedure. 

Let $\mathcal{M}_0=\{f_{\Ptheta_0},\Ptheta_0\in\Hc_0\}$ be a worst-case model that does not apply any prediction based on data, where $\Ptheta_0$ is the model parameter, and $\Hc_0$ is the parameter space of $\Mc_0$. Also, let the empirical risk of the model $\Mc_0$ be $R_{emp}(S_n,\hat{\pmb{\theta}}_0)$\footnote{For example, in the regression problems, $\mathcal{M}_0$ is the model of zero function with the empirical risk, $R_{emp}(S_n,\hat{\pmb{\theta}}_0)=(1/n)\sum_{i=1}^n l(0,\textbf{y}_i)$, or in an array of sensors, this situation will happen when there is not any source in the environment.}. 

Now, we generate a sequentially nested model family by sorting the models, and simultaneously, we approve or decline the models using the test \eqref{TESTEQ}.  
 
Consider a parameters set $\Theta^{\prime}=\{ {\Ptheta}_1^{\prime},{\Ptheta}_2^{\prime},...,{\Ptheta}_L^{\prime} \}$ as a set of potential parameters, where for every $i\in \bar{\mathcal{I}}$, $\bar{\mathcal{I}}=\{1,2,...,L\}$ be the set of indices corresponding to $\Theta^{\prime}$ and ${\pmb{\theta}}^{\prime}_i\in \mathcal{B}_i$,  $\mathcal{B}_i$ is the parameter space of ${\Ptheta}_i^{\prime}$ . Also, for every combination set of $\{1,2,...,L\}$, e.g. $\mathcal{J}=\{j_1,j_2,...,j_k\}$ where $k,j_1,j_2,...,j_k\in\{1,2,...,L\}$,  ${\Ptheta}_{\Jc}^{\prime}$ is a vector composed by concatenating of every ${\Ptheta}_i^{\prime}$ that  $i\in\mathcal{J}$, i.e., ${\Ptheta}_{\mathcal{J}}^{\prime}=[{{\Ptheta}_{j_1}^{\prime}}^T,{{\Ptheta}_{j_2}^{\prime}}^T,...,{{\Ptheta}_{j_k}^{\prime}}^T]^T$. We would like to arrange a sequentially nested model family exploiting valuable parameters of the set $\Theta^{\prime}$.  

We generate the candidate models $\{\mathcal{M}_j^{\prime}\}_{j\in\bar{\mathcal{I}}}$ for the first model of this sequentially nested model family, where for every index $j\in\bar{\mathcal{I}}$,    $\mathcal{M}_j^{\prime}=\Mc_0 \cup \{f_{{{\Ptheta}}_{j}^{\prime}}:{{{\Ptheta}}_{j}^{\prime}}\in \mathcal{B}_j \}$, where $\Ptheta_{j}^{\prime}$ is a parameter of the model $\mathcal{M}_j^{\prime}$.  Then, we assign model $\mathcal{M}_{\hat{j}}^{\prime}$  as the model with the least minimum empirical risk among all these candidate models. Indeed, $\hat{j}=\arg \min_{j\in\bar{\Ic}} R_{emp}(S_n,\hat{\Ptheta}_{j}^{\prime})$, where $\hat{\Ptheta}_{j}^{\prime}$ is the minimum empirical risk parameter in $\mathcal{M}_{j}^{\prime}$.  To verify the assigned model $\mathcal{M}_{\hat{j}}^{\prime}$,  we use the test \eqref{TESTEQ}  ($T_1=\mathbb{I}\{R_{emp}(S_n,\hat{\pmb{\theta}}_0)-R_{emp}(S_n,\hat{\Ptheta}_{\hat{j}}^{\prime})\geq \gamma_1(S_n,\delta)\}$). If $T_1=1$, $\mathcal{M}_{\hat{j}}^{\prime}$ will be approved as the first model, and if $T_1=0$, we decline $\mathcal{M}_{\hat{j}}^{\prime}$.  

If the assigned model is approved, we let $\mathcal{M}_1=\mathcal{M}_{\hat{j}}^{\prime}$ as the first model of the sorted sequentially nested model family.  Then,  its set of active parameter indices will be  $\Ic_1=\{\hat{j}\}$. Therefore, the candidate parameter indices for the next model will be updated by removing the ${\hat{j}}$ from the set of potential parameter indices $\bar{\Ic}$, i.e., $ \bar{\mathcal{I}}= \bar{\mathcal{I}}-\{\hat{j}\}$. Also, the parameter of the first model, $\Mc_1$, will be $\pmb{\theta}_1\equiv{{\Ptheta}}^{\prime}_{\Ic_1}$. 

In the next step, we compose the candidates of the second model in the sorted sequentially nested model family by extending the parameter space of the first approved model $\Mc_1$. We use the parameters corresponding to the updated potential indices $\bar{\mathcal{I}}$. Then, using the nesting process in Corollary \ref{SNgeneratingcorollary}, the candidates of the second model in the sorted sequentially nested models, $\{\mathcal{M}_j^{\prime}\}_{j\in\bar{\mathcal{I}}}$, will be generated, where for every $j\in\bar{\mathcal{I}}$,  $\mathcal{M}_j^{\prime}=\Mc_{1}\cup \{f_{{{\Ptheta}}^{\prime}_{\mathcal{I}^{\prime}_j}}:{{{\Ptheta}}^{\prime}_{\mathcal{I}^{\prime}_j}}\in \prod_{r\in\mathcal{I}^{\prime}_j} \mathcal{B}_r \}$,  ${\mathcal{I}_j^{\prime}}=\mathcal{I}_1\cup\{j\} $, and $\prod_{r\in\mathcal{I}^{\prime}_j} \mathcal{B}_r \equiv\mathcal{B}_{r_1}\times \mathcal{B}_{r_2}\times...\times \mathcal{B}_{r_{|\mathcal{I}^{\prime}_j|}}$. Based on Corollary \ref{SNgeneratingcorollary}, the model $\Mc_1$ is nested in every candidate in the second step of our expansion. Similar to the former step, we find a model in the set $\{\mathcal{M}_j^{\prime}\}_{j\in\bar{\mathcal{I}}}$ with the least minimum empirical risk. Then, using the test \eqref{TESTEQ}, this model will be verified as the second model of the sorted sequentially nested models (Accept or reject the assigned model). Fig. \ref{NER_fig1} schematically illustrates the assignment and verification process of the S-NER method. This procedure will be repeated until the test \eqref{TESTEQ} declines the assigned model for the first time. 
\begin{figure}[!t]
    \centering

    \subfloat[Assigned model is accepted]
{\adjincludegraphics[height=4cm,width=0.95\textwidth,trim={{0 \width} {0.0 \height} {0.0\width} {0.0185 \height}},clip]{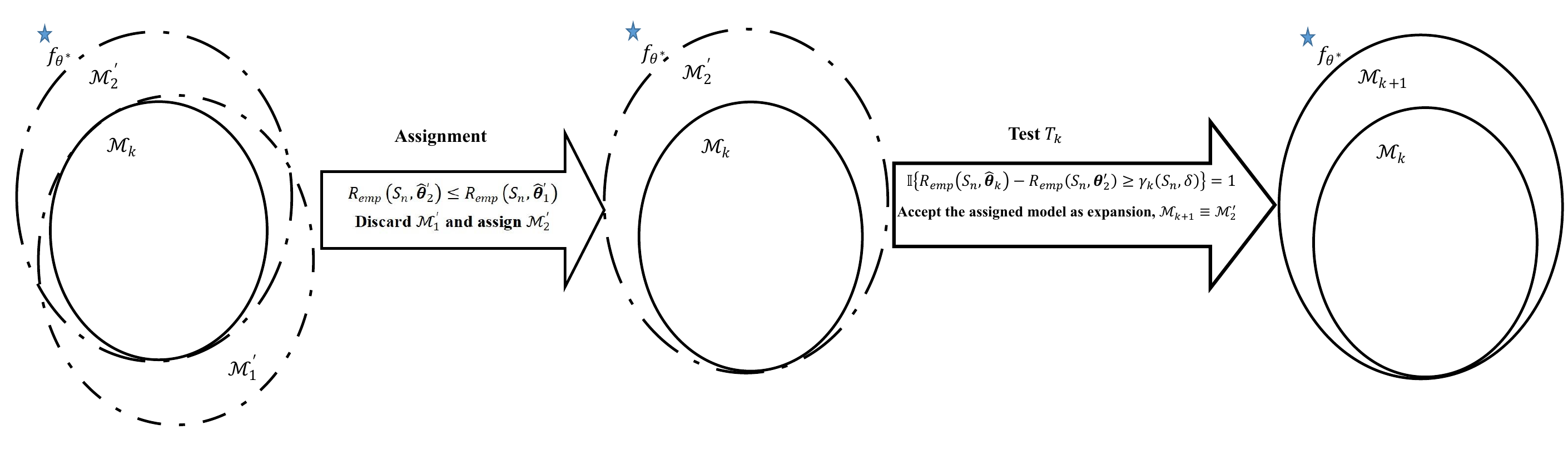}\label{UP_FIG}}
\\ \subfloat[Assigned model is rejected]   {\adjincludegraphics[height=4cm,width=0.95\textwidth,trim={{0 \width} {0.0 \height} {0.0\width} {0.0185 \height}},clip]{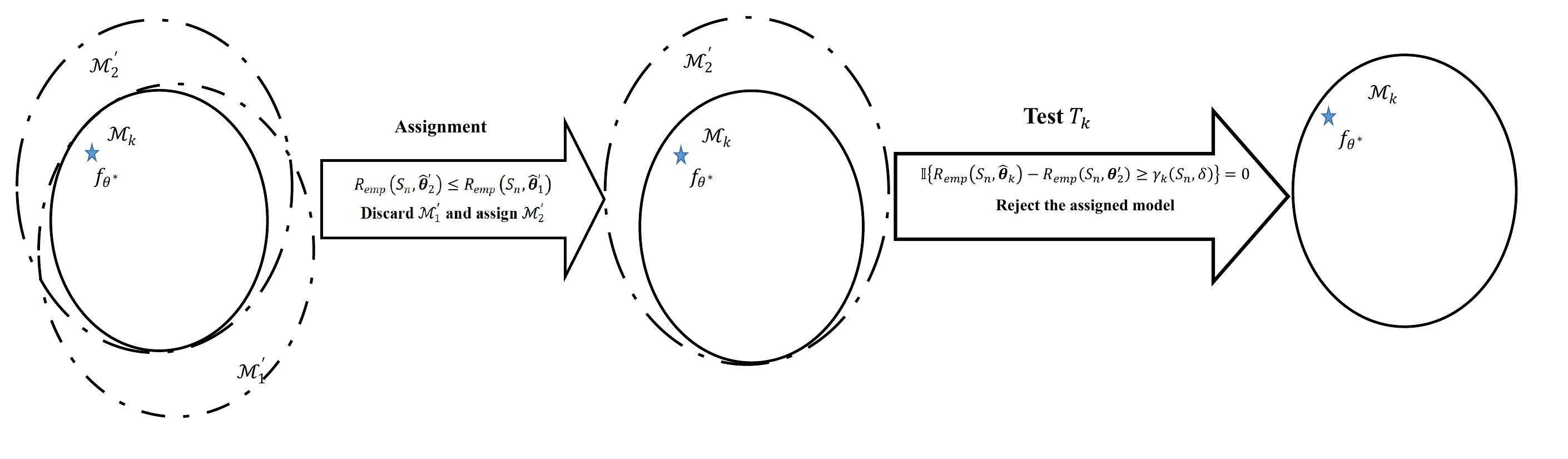}\label{NER_fig1:b}} 
        \caption{Block diagram of the assignment and validation process in the S-NER model selection method when (a) the assigned model is accepted and (b) the assigned model is rejected.}
    \label{NER_fig1} 
\end{figure}

If we call the last approved model by $\Mc_{\hat{K}}$, then we assume that $\Mc_{\hat{K}}$ is the most parsimonious model containing the global risk minimizer predictor. Briefly, the S-NER model selection method is presented in Algorithm \ref{NERALG}. In this algorithm, $\hat{K}$ and $\hat{\mathcal{I}}$ are the estimated model order and set of most parsimonious parameter indices, respectively. 

An important property of the S-NER method is verifying the expansion of the present model to the next model.  It provides verification to adopt the best candidate model for expansion from the present one. It is obvious that if the best candidate is declined in the verification procedure, then the other candidates can be ignored. This verification prevents an increase in the selected model complexity. 

 \begin{figure}
	\centering
	\adjincludegraphics[height=7.5cm,width=13cm,trim={{0 \width} {0.4 \height} {0.14\width} {0.14 \height}},clip]{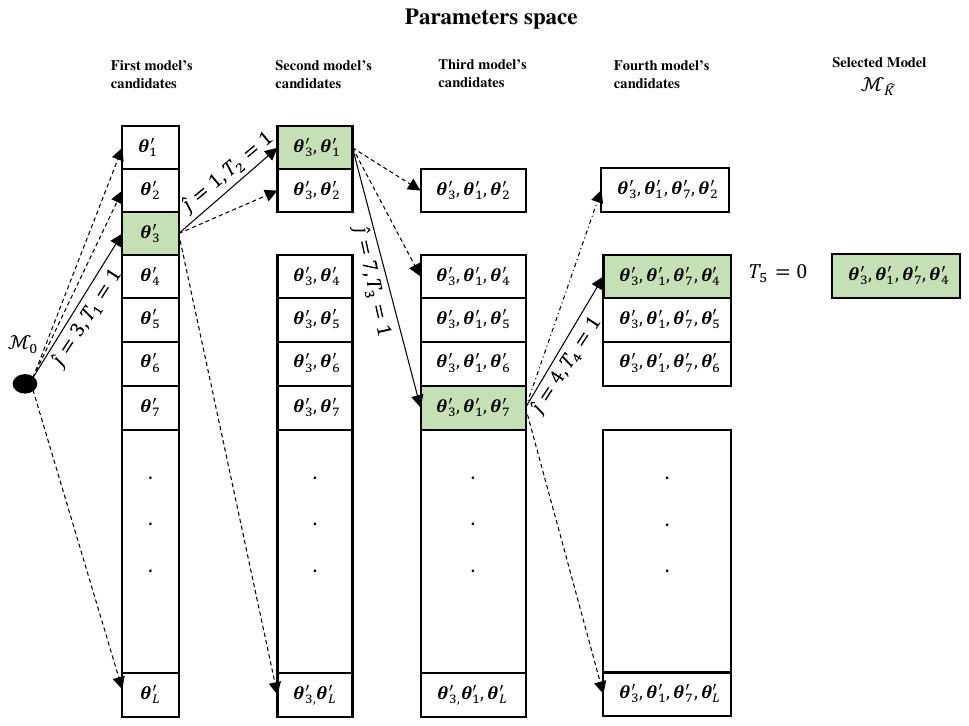}
	\caption{Example of parameter spaces in the S-NER model selection procedure. Dashed lines refer to the candidate models, and solid lines refer to the models with the least minimum empirical risk among the candidates. }\label{fig111}
\end{figure}

Fig. \ref{fig111} illustrates an example of the parameter space of models in a sorted sequentially nested model family in the S-NER model selection method. Based on this example, in the first iteration, the parameter $\Ptheta^{\prime}_3$ is determined as the most valuable parameter ($\hat{j}=3$),   $\mathcal{M}_3^{\prime}= \Mc_0 \cup \{f_{{{\Ptheta}}_{3}^{\prime}}:{{{\Ptheta}}_{3}^{\prime}}\in \mathcal{B}_3 \}$ is proposed as the first model, the test approves the proposed model ($T_1=1$), and $\Mc_1=\Mc_3^\prime$. In the second iteration, the parameter $\Ptheta^{\prime}_1$ is determined as the most valuable parameter ($\hat{j}=1$),  and  $\mathcal{M}_1^{\prime}= \Mc_1 \cup \{f_{[{\Ptheta}^{\prime}_3,\Ptheta^{\prime}_1]}:[{\Ptheta}^{\prime}_3,{\Ptheta}^{\prime}_1]\in  \mathcal{B}_3 \times \Bc_1 \}$  is proposed as the second model where ${\mathcal{I}^{\prime}_1}=\{ 3,1 \}$. This model is also approved by the test ($T_2=1$) and $\Mc_2=\mathcal{M}_1^{\prime}$. This procedure continues until the fifth iteration, where the test declines the proposed model ($T_5=0$), and the S-NER method chooses the last approved model, $\Mc_4$, as the selected model $\Mc_{\hat{K}}$. As seen, the procedure iterates only five iterations, and the selected model will be $\Mc_4$ while $L-5$ remaining parameters are eliminated automatically.

\begin{algorithm}[t!]
	\caption{S-NER model selection algorithm }\label{NERALG}
	\begin{algorithmic}
		\Require{$S_n=(\textbf{z}_1,\textbf{z}_2,...,\textbf{z}_n), \ \Mc_0=\{f_{\Ptheta_0}:\Ptheta_0\in\Hc_0\},$ $ \Theta^{\prime}=\{\Ptheta_1^{\prime}, \Ptheta_2^{\prime}, ...,\Ptheta_L^{\prime}\}$  }
		\Ensure  $k =0 , \ T_0=1, \ \bar{\mathcal{I}}=\{1,2,...,L\},\mathcal{I}_0=\emptyset $
		\While{$\{T_k=1\}\land \{k<L\}$}
  \State  $ \ k=k+1 $
        \For {$j\in \bar{\mathcal{I}}$}
        \State $\mathcal{I}^{\prime}_j=\mathcal{I}_{k-1}\cup \{j\}$
        \State $\mathcal{M}_j^{\prime} = \Mc_{k-1} \cup \{f_{{{\Ptheta}}^{\prime}_{\mathcal{I}^{\prime}_j}}:{{{\Ptheta}}^{\prime}_{\mathcal{I}^{\prime}_j}}\in \prod_{r\in\mathcal{I}^{\prime}_j} \mathcal{B}_{r}\}$
          \EndFor
        \State $\hat{j} =\arg \min_{j \in \bar{\mathcal{I}}} R_{emp}(S_n,\hat{{\Ptheta}}^{\prime}_{\mathcal{I}^{\prime}_j})$
        \State $T_k = \mathbb{I}\big\{R_{emp}(S_n,\pmb{\hat{\theta}}_{{k-1}})-R_{emp}(S_n,\hat{{\Ptheta}}^{\prime}_{\mathcal{I}^{\prime}_{\hat{j}}})\geq \gamma_k(S_n,\delta)\big\}$
        \If{$T_k=1$}
		\State $\mathcal{M}_k =\mathcal{M}_{\hat{j}}^{\prime}, \ {\mathcal{I}}_k={\mathcal{I}}_{_{k-1}}\cup\{\hat{j}\},  \ \pmb{\theta}_k\equiv{\Ptheta}^{\prime}_{\mathcal{I}_k}$
  \EndIf
          \State  $\bar{\mathcal{I}}=\bar{\mathcal{I}}-\{\hat{j}\}$
		\EndWhile
		\State $ \hat{K}=k, \mathcal{M}_{\hat{K}}=\mathcal{M}_{k}, \ \hat{\mathcal{I}}={\mathcal{I}}_{k}$
	\end{algorithmic}
\end{algorithm}	
Notice that, for $\hat{K}\leq L-1$, the S-NER algorithm needs $\hat{K}+1$ iterations, and for $\hat{K}= L$, it needs $L$ iterations. Then, for every $\hat{K}\in\{1,2,...,L\}$, the number of iterations in the S-NER algorithm is $\min(\hat{K}+1,L)$. However, all other model selection methods, like EFIC \cite{owrang2018model}, EBICR \cite{gohain2023robust},  MDL \cite{rissanen1978modeling}, and BIC \cite{schwarz1978estimating}, that use feature sorting algorithms need to calculate their criterion, e.g., information criterion, in $L$ iterations. 

 It will be shown in Subsection \ref{Section6.1}, for the first time in the literature on the linear regression problem, the S-NER method, using intelligently model arranging and without using any prior information, can outperform the accuracy of widespread feature sorting algorithms like OMP and LARS that aided with the prior knowledge of true model order. Note that in the linear regression problem, the accuracy of the aided OMP and the aided LARS algorithms are oracle and upper bounds for the accuracy of every model selection method that uses OMP or LARS as feature sorting \cite{gohain2023robust,owrang2018model,kallummil2018signal}.	

In Algorithm \ref{NERALG}, the target is to achieve the most parsimonious model containing a risk minimizer predictor.  In this regard, every model expansion approved by  Algorithm \ref{NERALG} must be useful, and every eliminated candidate model expansion must be useless. In other words, $\Mc_{\hat{K}}$ is the most parsimonious model, if for every $k\in\{1,2,...,\hat{K}\}$, through the model expansion from $\Mc_{k-1}$ into $\Mc_{k}$, the minimum risk will be decreased, and in the $(\hat{K}+1)$-th iteration,  the minimum risk of all eliminated candidate models $\{\Mc\}_{j\in\bar{\Ic}}$ will be the same as the minimum risk of $\Mc_{\hat{K}}$. Let the probability of correctly arranging to select the most parsimonious model containing a risk minimizer be $\Pb\mathbb{C}_{S-NER}$.  In the following theorem, we compute a lower bound on  $\Pb\mathbb{C}_{S-NER} $. 
\begin{theorem}\label{PCMSTH}
Let the set of loss functions of the candidate models in every iteration of  Algorithm \ref{NERALG} be the Glivenko-Cantelli class of functions, and  $\hat{K}$ and $\mathcal{M}_{\hat{K}}$ are calculated based on Algorithm \ref{NERALG}. Therefore, for every  $0<\delta\leq 1/L$, there is a positive integer $N$, where for $n\geq N$, $\mathcal{M}_{\hat{K}}$ is the most parsimonious model containing a risk-minimizing predictor with probability of at least $1-L\delta$, i.e., 
\begin{equation} \label{PCDLBEQ11}
\Pb\mathbb{C}_{S-NER} \geq1-L\delta.
\end{equation}
\end{theorem}
\begin{proof}
See Appendix \ref{APPB}.
\end{proof}
This theorem implies that for some small $\delta$, Algorithm \ref{NERALG}  can correctly select the most parsimonious model with a high probability of at least $1-L\delta$.

  Similar to the NER model order selection, for $L\geq \frac{1}{\delta}$, the value of the lower bound $1- L\delta$ on the correct decision probability of the S-NER algorithm will be non-positive. So, the lower bound will be trivial and will not provide any useful information. As mentioned before, this mostly happens in high-dimensional scenarios, where, usually in this scenario, it is assumed that there is a maximum model order $k_{max}$, where $K<k_{max}\ll L$. In this case, we use $k_{max}$ instead of $L$, and for a small $k_{max} \ll \frac{1}{\delta}$, the lower bound $1-k_{max}\delta$ for the probability of a correct decision will be informative. In the following corollary, we show that the S-NER method is a consistent predictor.
\begin{corollary}[\textbf{Consistency of the S-NER Model Selection}]\label{consistencycorollary}
Let the conditions of Theorem \ref{PCMSTH} hold and the number of models, $L$, be a fixed positive integer. Then, the estimated model $\mathcal{M}_{\hat{K}}$  using Algorithm \ref{NERALG} is the most parsimonious model containing a risk minimizer in probability as $n\to \infty$ ($\lim_{n\to\infty} \mathbb{PC}_{S-NER}=1$).
\end{corollary}
\begin{proof}
	Based on Theorem \ref{PCMSTH}, for every $0<\delta^{\prime}\leq 1/L$, there is a positive integer $N$ where for every   $n\geq N$, $ \mathbb{PC}_{S-NER} $ will be calculated as follows
   \begin{flalign}
    \mathbb{PC}_{S-NER}\geq 	 1-  L  \delta^{\prime}. \label{deltatempEQ111}
	\end{flalign}
Now, let $\delta=L\delta^{\prime}$.  Then, based on \eqref{deltatempEQ111}, for every $0<\delta\leq 1$, there is a positive integer $N$ where for every $n\geq N$,  $ \mathbb{PC}_{S-NER}\geq 1-\delta$. So, we have
         \begin{flalign}
  \lim_{n\to \infty}  \mathbb{PC}_{S-NER} 	=1. \label{deltatempEQ1}
	\end{flalign}
\end{proof}
This corollary implies the S-NER algorithm consistently selects the most parsimonious model containing risk minimizer.  

In this section, we have proposed the  S-NER method to select the most parsimonious model containing the risk minimizer predictor, and we have shown that the S-NER is a consistent method.

In the next section, we customize the S-NER model selection methods for the linear regression problem and achieve the SEER bound for this problem.

\section{Model Selection in Linear Regression}\label{Sec505}
Linear regression is one of the most important models in machine learning. Determining a linear model based on the most parsimonious set of variables to estimate the model of data is a crucial problem. Numerous works are presented to fit the linear model on data \cite{tropp2006just,tibshirani1996regression,dai2009subspace}. In this section, we want to fit the best linear model on observation data.

Consider the set of observations $S_n=\{(\textbf{x}_1,y_1),(\textbf{x}_2,y_2),...,(\textbf{x}_n,y_n)\}$ where for every $i\in \{1,2,...,n\}$,  $\textbf{x}_i$ is a vector of $L$ independent features or co-variates, and the scalar $y_i$ is  the corresponding response. Our target is to choose the best linear model to fit the observation data $S_n$.  In this regard, we consider the linear model as follows
\begin{equation}
    y=\textbf{x}^T\Ptheta^{\prime}+\varepsilon, 
\end{equation}
where $\Ptheta^{\prime}=[\theta_1^{\prime},\theta_2^{\prime},...,\theta_L^{\prime}]$ is an unknown $L$-dimensional parameter vector, and  $\varepsilon$ is a zero mean Gaussian random variable with the variance of $\sigma^2$, i.e.,  $\Nc(0,\sigma^2)$.  

We want to find the unknown parameter $\Ptheta^{\prime}$ with the minimum risk $R(\Ptheta^{\prime})=\mathbb{E}\{l(\textbf{x}^T\Ptheta^{\prime},y)\}$. However, some features may not be valuable. Then, we attempt to determine $\Ptheta^{\prime}$ such that the linear model has the least complexity using only valuable features in the linear model. Based on these $L$ features, we have $2^L$ combinations for the linear model, and the evaluation of the $2^L$ linear models is exhaustive, especially for the high-dimensional scenarios where the number of features, $L$,  is enormous.

To tackle such a huge complexity, many model selection methods like EBIC \cite{chen2008extended}, EBICR \cite{gohain2023robust}, and EFIC \cite{owrang2018model} try to select the most valuable features. They use algorithms like OMP or LARS  \cite{weisberg2005applied,efron2004least} to sort the features, and then, based on criteria, they estimate the order of the model. Therefore, the performance of these approaches is intrinsically tied to the accuracy of the OMP or LARS feature-ranking procedures. Because no subsequent assessment of feature relevance is conducted after ranking, errors in the sorting stage may propagate to the model-selection process, leading to the removal of valuable features or the retention of redundant ones.

We are confronted with the problem of finding a way to prevent error propagation from these feature sorting algorithms, such as OMP or LARS, in the model selection methods. 

In the next subsection, we use the S-NER algorithm for the linear regression problem to select the most valuable features that preserve less complexity as much as possible.

\subsection{S-NER Model Selection in the Linear Regression}\label{Sec51}
Consider  the set of supervised observations $S_n=\{(\textbf{x}_1,y_1),(\textbf{x}_2,y_2),...,(\textbf{x}_n,y_n)\}$ where for every $i\in \{1,2,...,n\}$,  $\textbf{x}_i$ is a vector of $L$ independent features or co-variates, and the scalar $y_i$ is the corresponding response. Let $\Theta^{\prime}=\{{\theta}_1^{\prime},{\theta}^{\prime}_2,...,{\theta}^{\prime}_L\}$ be the set of all parameters corresponding to these $L$ features, where for every $j\in\{1,2,...,L\}$,  ${{\theta}}^{\prime}_j \in \mathbb{R}$. Now, for every $l\in\{1,2,...,2^L-1\}$, we consider the linear model $\mathcal{M}_l^{\prime}=\{({\textbf{x}^{\Ic_l^{\prime}}})^T {\pmb{\theta}}_{\Ic_l^{\prime}}^{\prime}:{\pmb{\theta}}_{\Ic_l^{\prime}}^{\prime}\in \mathbb{R}^{k_l}\}$, where $\Ic_l^{\prime}$ is any non-empty subset of the set $\{1,2,...,L\}$, $k_l=card(\Ic_l^{\prime})$,  and ${\textbf{x}^{\Ic_l^{\prime}}}$ is a vector consisting of elements of $\textbf{x}$ that are indexed by $\Ic_l^{\prime}$. 

In this model class, we consider the parameters and the observations to be finite real variables. So, based on the weak law of large numbers \cite{loeve1977elementary}, for every $l\in \{1, 2, ..., L\}$ and every parameter $\Ptheta^{\prime}_{\Ic_l^{\prime}}$, the sample mean of the least square loss function  converges to its expected value  as $n\to \infty$. Therefore, based on Definition \ref{GCdef}, in the linear regression problem, for every $l$, the set of least square loss functions corresponding to each model is the Glivenko-Cantelli class of functions. 

 In the following, we present the S-NER algorithm to use in the linear regression problem.  The S-NER generates a sorted sequentially nested linear model family using the most valuable parameters of $\Theta^{\prime}$ and rejects the useless parameters. In this regard, based on Algorithm \ref{NERALG}, we let  $\Mc_0=\{\textbf{x}^T\Ptheta_0: \Ptheta_0=\textbf{0}_{L\times 1}\}$ be the model without any prediction. Since this model does not consider any features for predicting the response, the set of active parameters of $\Mc_0$ is $\Ic_0=\emptyset$. Using the least square function, $l(\textbf{x}^T\Ptheta_0,y)=(\textbf{x}^T\Ptheta_0-y)^2$, as the loss function, the empirical risk of $\Ptheta_0$ will be $R_{emp}(S_n,{\Ptheta}_0)=\frac{1}{n}\sum_{i=1}^{n} (\textbf{x}_i^T\Ptheta_0-y_i)^2=\frac{\|\textbf{y}\|_2^2}{n}$ ($\|.\|$ is the Euclidean norm).  In this model, the parameter with minimum empirical risk is $\hat{\Ptheta}_0=\Ptheta_0=\textbf{0}_{L\times 1}$ and $R_{emp}(S_n,\hat{\Ptheta}_0)=R_{emp}(S_n,{\Ptheta}_0)=\frac{\|\textbf{y}\|_2^2}{n}$.

Based on the S-NER method, we expand the model and verify the validity of this model expansion using the test \eqref{TESTEQ}. In other words, this algorithm investigates whether the present and the expanded models can be considered successive models in a model sorting method. In the test \eqref{TESTEQ}, for every $k\in\{1,2,...,L\}$, the threshold $\gamma_k(S_n,\delta)$ is an $o(1)$ bound on the SEER. In the next subsection, we will calculate an $o(1)$ bound $\gamma_k^R(S_n,\delta)$ on the SEER exclusively for the linear regression problem. In this subsection, $\gamma_k^R(S_n,\delta)$ is used as the threshold of the test \eqref{TESTEQ}.
 
Now, to expand the model from $\Mc_0$ (or $\Ic_0=\emptyset$) to a new model with one feature, we have $L$ candidate models $\{\mathcal{M}_j^{\prime}\}_{j=1}^L$. This is the first model expansion to make the sorted sequentially nested models. Therefore, for every $j\in \{1,2,...,L\}$, $\mathcal{M}_j^{\prime}=\{[\textbf{x}]_j {{\theta}}^{\prime}_{j}:{\theta}^{\prime}_{j}\in \mathbb{R}\}$, and $[\textbf{x}]_j$ is the $j$-th element of vector $\textbf{x}$. Using the least square function, $l([\textbf{x}]_j\theta_j^{\prime},y)=([\textbf{x}]_j\theta_j^{\prime}-y)^2$, as the loss function, for every $j$, the empirical risk of $\mathcal{M}_j^{\prime}$ will be $R_{emp}(S_n,{\theta}_j^{\prime})=\frac{1}{n}\sum_{i=1}^{n} ([\textbf{x}_i]_j\theta_j^{\prime}-y_i)^2$. Then, the algorithm assigns model $\mathcal{M}_{\hat{j}}^{\prime}$ as the model with the least minimum empirical risk among all candidate models. Indeed $\hat{j}=\arg\min_{j\in\{1,2,...,L\}} R_{emp}(S_n,\hat{\theta}_j^{\prime})$, where  $\hat{\theta}_j^{\prime}$ is the empirical risk minimizer parameter in $\Mc_j^{\prime}$.  Based on the Moore-Penrose solution $\hat{\theta}_j^{\prime}=({\textbf{X}^{\{j\}}})^{\dagger} \textbf{y}$, where $({\textbf{X}^{\{j\}}})^{\dagger}=\big( {\textbf{X}^{\{j\}}}({\textbf{X}^{\{j\}}})^T\big)^{-1}{\textbf{X}^{\{j\}}}$  and ${\textbf{X}^{\{j\}}}$ is the $j$-th row of matrix $\textbf{X}=[\textbf{x}_1,\textbf{x}_2,....\textbf{x}_n]$. Then, we verify the assigned model $\mathcal{M}_{\hat{j}}^{\prime}$  using the test \eqref{TESTEQ}  ($T_1=\mathbb{I}\{R_{emp}(S_n,\hat{\pmb{\theta}}_0)-R_{emp}(S_n,\hat{\theta}_{\hat{j}}^{\prime})\geq \gamma_1^R(S_n,\delta)\}$). If $T_1=1$, $\mathcal{M}_{\hat{j}}^{\prime}$ will be approved as the first model, and if $T_1=0$, we decline $\mathcal{M}_{\hat{j}}^{\prime}$.  
If the assigned model is approved, we let $\mathcal{M}_1=\mathcal{M}_{\hat{j}}^{\prime}$ as the first model of the sorted sequentially nested linear model family.  Then,  the set of active parameter indices of $\Mc_1$ is  $\Ic_1=\{\hat{j}\}$, and the parameter of the model $\Mc_1$ will be ${\theta}_1\equiv{{\theta}}^{\prime}_{\hat{j}}$. 

In the next step, we compose the candidates of the second model in the sorted sequentially nested linear model family by extending the parameter space of the first approved model $\Mc_1$. We use the parameters corresponding to the updated potential indices $\{1,2,...,L\}-\Ic_1$. Then, for every $j\in\{1,2,...,L\}-\Ic_1$,  the $j$-th candidate of the second model in the sorted sequentially nested linear models will be  $\mathcal{M}_j^{\prime}=\{(\textbf{x}^{\Ic_j^{\prime}})^T {{{\Ptheta}}^{\prime}_{\mathcal{I}^{\prime}_j}}:{{{\Ptheta}}^{\prime}_{\mathcal{I}^{\prime}_j}}\in \mathbb{R}^2 \}$ and  ${\mathcal{I}_j^{\prime}}=\mathcal{I}_1\cup\{j\}$.  Similar to the previous step, we find a model with the least minimum empirical risk among candidate models $\mathcal{M}_j^{\prime}, j\in \{1,2,...,L\}-\Ic_1$. Note that, for every $j\in\{1,2,...,L\}-\Ic_1$, using the least square loss function, the minimum empirical risk of the $j$-th model will be $R_{emp}(S_n,\hat{\Ptheta}_{\Ic_{j}^{\prime}}^{\prime})=\frac{\|\textbf{y}-({\textbf{X}^{\Ic_j^{\prime}}})^{T}{\hat{\pmb{\theta}}}_{\Ic_j^{\prime}}^{\prime}\|_2^2}{n}$.  Based on the Moore-Penrose solution, ${\hat{\pmb{\theta}}}_{\Ic_{j}^{\prime}}^{\prime}=({\textbf{X}^{\Ic_{j}^{\prime}}})^{\dagger} \textbf{y}$  where $({\textbf{X}^{\Ic_{j}^{\prime}}})^{\dagger}=\big( {\textbf{X}^{\Ic_{j}^{\prime}}}({\textbf{X}^{\Ic_{j}^{\prime}}})^T\big)^{-1}{\textbf{X}^{\Ic_{j}^{\prime}}}$ and ${\textbf{X}}^{\Ic_{j}^{\prime}}$ denotes the sub-matrix of $\textbf{X}$ form using the rows indexed by $\Ic_{j}^{\prime}$. 

  Then, we verify the assigned model $\mathcal{M}_{\hat{j}}^{\prime}$  using the test \eqref{TESTEQ}  ($T_2=\mathbb{I}\{R_{emp}(S_n,\hat{\pmb{\theta}}_1)-R_{emp}(S_n,\hat{\Ptheta}_{\Ic_{\hat{j}}^{\prime}}^{\prime})\geq \gamma_2^R(S_n,\delta)\}$). If $T_2=1$, the assigned model is approved, and we have $\mathcal{M}_2=\mathcal{M}_{\hat{j}}^{\prime}$    $(\mathcal{I}_2=\mathcal{I}_1\cup\{\hat{j}\} )$ as the second model of the sorted sequentially nested linear model family. The parameter of the model $\Mc_2$ will be ${\Ptheta}_2\equiv{{\Ptheta}}^{\prime}_{\Ic_2}$.  This procedure will be repeated until the test declines the assigned model for the first time. 

Based on Corollary  \ref{KASSUMPTION}, we have shown that the S-NER terminates at a step $\hat{K}$ where $0<\hat{K}\leq L$. We call the last approved model by $\Mc_{\hat{K}}$, then we assume that $\Mc_{\hat{K}}$ is the most parsimonious model containing the global risk minimizer predictor and $\hat{\mathcal{I}}=\mathcal{I}_{\hat{K}}$ is the S-NER  estimation of the most parsimonious set of features to predict the response.  In the following section, we present the SEER bound exclusively for linear regression problems.
\subsection{SEER PAC Bounds in Linear Regression}\label{Sec5}
Consider  the set of supervised observations $S_n=\{(\textbf{x}_1,y_1),(\textbf{x}_2,y_2),...,(\textbf{x}_n,y_n)\}$ where for every $i\in \{1,2,...,n\}$,  $\textbf{x}_i$ is a vector of $L$ independent features or co-variates, and the scalar $y_i$ is   the corresponding response. Assume that the observed data are generated based on a linear model as follows
\begin{flalign}\label{datagenEQ}
	\textbf{y}=\textbf{X}^T{\pmb{\theta}^*}+{\pmb{\varepsilon}},
\end{flalign}
where $\pmb{\varepsilon}$ is an i.i.d. zero-mean Gaussian  noise vector $(\mathcal{N}(\textbf{0},\sigma^2 \textbf{I}_n))$ independent of features, $\textbf{y}=[y_1,y_2,..,y_n]^T$, and $\textbf{X}=[\textbf{x}_1,\textbf{x}_2,...,\textbf{x}_n]$ is a full rank matrix. $\pmb{\theta}^*$ is an L-dimensional sparse vector with support set $\mathcal{S}$, where   $[\pmb{\theta}^*]_j\neq 0, j\in \mathcal{S}$ and $[\pmb{\theta}^*]_{j} = 0,  j\notin \mathcal{S}$. Also, we assume the number of elements of $\mathcal{S}$,  $card(\mathcal{S})=K \ (1\leq K\leq L)$. 

Based on Corollary  \ref{KASSUMPTION}, in every sequentially nested model class, there is a model index $ K \in\{1,2,...,L\}$ such that model $\Mc_K$ contains a risk minimizer and there is not any risk minimizer in  $\Mc_{K-1}$. In the S-NER algorithm, $\hat{K}$  is the estimated order of the sorted models as the most parsimonious model containing the risk minimizer. In the S-NER method, the number of models in the sorted sequentially nested models will be $\min{(\hat{K}+1,L)}$. We let $\bar{L}=\min{(\hat{K}+1,L)}$. Then, for every $k\in\{1,2,...,\bar{L}\}$,  the sorted sequentially nested linear model class will be $\{\Mc_k\}_{k=1}^{\bar{L}}$, where $\mathcal{M}_k=\{({\textbf{x}^{\Ic_k}})^T {\pmb{\theta}}_k:{\pmb{\theta}}_k\in \mathbb{R}^k\}$. In these models,  $\Ic_k$ is a unique subset of the set  $\{1,2,...,{L}\}$  with cardinality $k$ ($\Ic_k \subseteq \{1,2,...,{L}\}$),  and ${\textbf{x}^{\Ic_k}}$  is a vector consisting of elements of $\textbf{x}$ that are indexed by $\Ic_k$. In this sorted sequentially nested linear model class, for every $k\in\{2,3,...,\bar{L}\}$, $\Ic_{k-1}\subseteq \Ic_{k}$.  In addition, we let $\Mc_0=\{\textbf{x}^T{\Ptheta_0}:\Ptheta_0=\textbf{0}\}$ be a model without any prediction. Since this model does not use any feature in the linear model, we let  $\Ic_0=\emptyset$. Then, in the following theorem,  we investigate the statistical behavior of the SEER in linear regression. 
\begin{theorem}\label{LinearLowerboundTH}
	Let $F_{U_k}(.,\zeta_{{U_k}})$ be the commutative distribution function (CDF) of a non-central $\chi^2_1(\zeta{{U_k}})$ random variable with non-centrality parameter $ \zeta_{U_k}=\frac{\| (\textbf{P}_{k}-\textbf{P}_{k-1})\mathbb{\textbf{X}^T \pmb{\theta}^*} \|_2^2}{\sigma^2}$, $F_{1}(.)$ be the  CDF function of a central $\chi^2_1$ random variable, and $F_{1}^{-1}(.)$ be the inverse CDF function of a central $\chi^2_1$ random variable. Then,  for every $k\in\{1,2,...,\bar{L}\}$, 
 \begin{itemize}
 \item{if $ \mathcal{S} \subseteq \Ic_{k-1}$, for every $0<\delta\leq 1$,  $n\geq1$, and constant $c>0$, we have
	\begin{flalign}
		\mathbb{P}\Big\{R_{emp}(S_n,\hat{\pmb{\theta}}_{k-1})-R_{emp}(S_n,\hat{\pmb{\theta}}_k)\leq c\frac{\sigma^2}{n} F_1^{-1}(1-\delta) \Big\}=F_1(cF_1^{-1}(1-\delta)),  \label{LinRegUpboundEQ}
	\end{flalign}}
 \item {and if $\Ic_{k-1} \subset \mathcal{S} $, for every $0<\delta\leq 1$, $n\geq 1$, and constant $c>0$, the following holds}
  \end{itemize}
	\begin{flalign}
		\mathbb{P}\Big\{R_{emp}(S_n,\hat{\pmb{\theta}}_{k-1})-R_{emp}(S_n,\hat{\pmb{\theta}}_k)\geq c\frac{\sigma^2}{n} F_1^{-1}(1-\delta) \Big\}=  1-F_{U_k}(cF_1^{-1}(1-\delta),\zeta_{U_k}).\label{LinRegLowboundEQ}
	\end{flalign}

\end{theorem}
\begin{proof}
See the proof in Appendix \ref{APPC}.
\end{proof}
In this paper, we will use $c\frac{\sigma^2}{n}F_1^{-1}(1-\delta)$ as the threshold $\gamma_k^R(S_n,\delta)$ of test \eqref{TESTEQ} in the S-NER model selection methods for the linear regression problem. Note that, for every $0<\delta\leq 1$ and $\varepsilon>0$, there is a positive integer $N$, where for every $n\geq N$, $c\frac{\sigma^2}{n}F_1^{-1}(1-\delta)\leq \varepsilon$. Therefore, the function $c\frac{\sigma^2}{n}F_1^{-1}(1-\delta)$ has the order of $o(1)$, and can be used as a bound for our proposed model selection methods.

    
        


In the PAC bounds \eqref{LinRegUpboundEQ} and \eqref{LinRegLowboundEQ}, $\delta$ is a free parameter. Usually, $\delta$ can be chosen based on the conditions in our problem. In the linear regression problem, based on the number of observations and noise variance, we choose $\delta(n)=c_1(\sigma^2/n)$, where $c_1>0$ is a constant to adjust the bound and its probability of occurrence. 

In addition, consider the sparse Riesz condition on the feature matrix $\textbf{X}$ \cite{zhang2008sparsity,chen2008extended,gohain2023robust}.  It states that for every subset $\Ic$ from $\{1,2,...,{L}\} $ that $card(\Ic)\leq \bar{L}$, the following holds
\begin{equation}\label{RieszASS}
    \lim_{n\to\infty} \frac{(\textbf{x}^{\Ic})^T \textbf{x}^{\Ic}}{n}=\textbf{M}_{\Ic},
\end{equation}
where $\textbf{M}_{\Ic}$ is a positive definite matrix. In the following theorem, we will show that for the above choice of $\delta(n)$ and based on sparse Riesz condition \eqref{RieszASS}, S-NER for this bound is reliable when  $n\to \infty$. In other words, S-NER is a consistent algorithm for regression problems.
\begin{theorem} \label{consistency} 
Let  $F_{1}^{-1}(.)$ be the inverse CDF function of a central $\chi^2_1$ random variable and assume that sparse Riesz condition \eqref{RieszASS} holds. Also, for a constant  $c_1>0$, let $\delta(n)=c_1(\sigma^2/n)$  for every constants $c,c_1>0$ and every $k\in\{1,2,...,L\}$,  
 \begin{enumerate}
 \item if  $ \mathcal{S} \subseteq \Ic_{k-1}$,  we have
\begin{flalign}
 \lim_{n\to \infty} \mathbb{P}\Big\{R_{emp}(S_n,\hat{\pmb{\theta}}_{k-1})-R_{emp}(S_n,\hat{\pmb{\theta}}_k) \leq \frac{\sigma^2}{n}cF_1^{-1}\Big(1-c_1\Big(\frac{\sigma^2}{n}\Big)\Big)\Big\}=1,\label{limlowerbounddeltalinEQ}
 \end{flalign}
 \item and if  $\Ic_{k-1} \subset \mathcal{S} $, we have
 \begin{flalign}
		\lim_{n\to \infty} \mathbb{P}\Big\{R_{emp}(S_n,\hat{\pmb{\theta}}_{k-1})-R_{emp}(S_n,\hat{\pmb{\theta}}_k) \geq \frac{\sigma^2}{n}cF_1^{-1}\Big(1-c_1\Big(\frac{\sigma^2}{n}\Big)\Big)\Big\}=1.\label{limupperbounddeltalinEQ}
	\end{flalign}
  \end{enumerate}
\end{theorem}
\begin{proof}
See the proof in Appendix \ref{APPC}.
\end{proof}

\begin{figure}[!t]

    \centering
    
{\includegraphics[width=13cm, height=6.5cm]{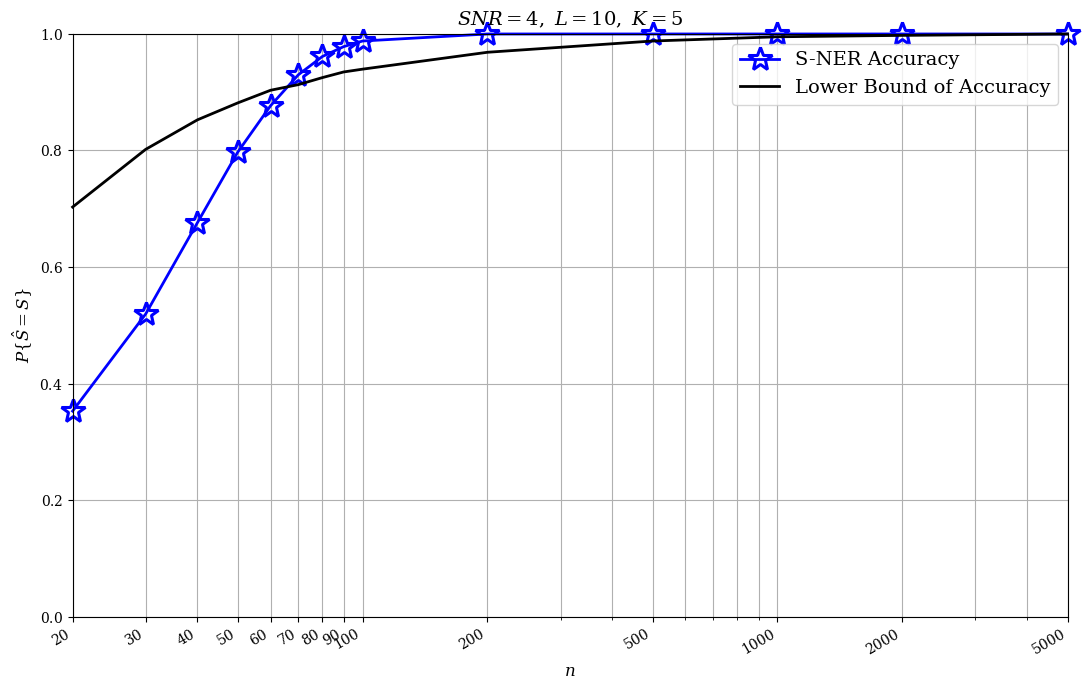}}
        \caption{Comparison of the S-NER correct model selection probability with its theoretical lower bound $1-L\delta$ using synthetic data. }
    \label{SNERLB_fig}
\end{figure}

To investigate the consistency of the S-NER algorithm, we evaluate the empirical probability of S-NER correct model selection using linear regression synthetic data and compare it with the $1-L\delta$ theoretical bound derived in Theorem~\ref{PCMSTH}. The data generation procedure is explained in Subsection \ref{Section6.1}.
As illustrated in Fig.~\ref{SNERLB_fig}, for any sample size $n$ exceeding a threshold $N_0$ (where $60 < N_0 < 70$), the S-NER's correct model selection probability is lower-bounded by  $1-L\delta$. Moreover, as the number of observations increases, the lower bound becomes tighter, and the probability of correct model selection converges to~1.

We use the SEER bound in Theorem \ref{consistency} as a threshold,  and  for every $k$, the S-NER model selection test for the linear regression problem will be as follows
\begin{flalign}\label{LRtestEQ}
    T_k =\mathbb{I}\Big\{\Delta R_{emp}(S_n,k)\geq c\frac{\sigma^2}{n}F_1^{-1}\Big(1-c_1\Big(\frac{\sigma^2}{n}\Big)\Big) \Big\},
\end{flalign}
where $\mathbb{I}$ is the indicator function.

In the next section, we use the S-NER model selection methods for the regression and classification applications.

\section{Applications} \label{SEC6}
In this section, we present the applications of the proposed model selection methods in linear regression and classification problems. In Subsection \ref{Section6.1}, we use the S-NER model selection method to find the most parsimonious set of features for predicting the response based on the linear regression model.  Using synthetic data, we compare our method with the state-of-the-art methods in this problem. Also, in Subsection \ref{Section6.2}, we use the NER model order selection method for the feature selection in the classification task of the UCR  dataset.
\subsection{Linear Regression Model Selection Using Synthetic Data}\label{Section6.1}
In this subsection, first, we generate synthetic data $S_n$,  and then we estimate the best regression model using the S-NER model selection method.  Consider $S_n=\{(\textbf{x}_1,y_1),(\textbf{x}_2,y_2),..,(\textbf{x}_n,y_n)\}$ generated based on the linear regression model \eqref{datagenEQ} in the high-dimensional scheme where the number of features is greater than the number of observations, i.e., $L\geq n$. In this subsection, for some $d\geq1$, we let $L=\lceil n^d \rceil$. In  \eqref{datagenEQ},  $\pmb{\theta}^*$ is the true model parameter, where it is a sparse vector with the support set  $\mathcal{S}$. In the model selection procedure, we try to estimate the support set $\mathcal{S}$. In this regard, we assume that $K=card(\mathcal{S})$ is not greater than a known number $k_{max}$ where $K\leq k_{max} \ll  L$. 

For every $ j\in\{1,2,..., L\}$ and $\  i\in \{1,2,...,n\}$, we generate $[\textbf{X}]_{ji}$ using a Normal distribution ($\mathcal{N}(0,1)$). Also, for $j\notin \mathcal{S}$, let $[{\pmb{\theta}}^*]_j=0$, and for $j\in \mathcal{S}$, let $[{\pmb{\theta}}^*]_j$ be a random variable in $\{1,-1\}$ where $\mathbb{P}\{[{\pmb{\theta}}^*]_j=1\}=\mathbb{P}\{[{\pmb{\theta}}^*]_j=-1\}=1/2$.  Then, the scalar responses $y_1,y_2,...,y_n$ are obtained using the linear regression model \eqref{datagenEQ}. Now, for every $i$, let $y_i^{\prime}=\frac{y_i}{\hat{\sigma}_{\textbf{y}}}$ be the normalized response, where $\hat{\sigma}_{\textbf{y}}=\sqrt{({\sum_{i=1}^{n} (y_i-\hat{\mu}_{\textbf{y}})^2}) /(n-1) }$ and $\hat{\mu}_{\textbf{y}}=({1}/{n}) \sum_{i=1}^{n} y_i$. Then, the normalized set of observations will be $S_n^{\prime}=\{(\textbf{x}_1,y_1^{\prime}),(\textbf{x}_2,y_2^{\prime}),..,(\textbf{x}_n,y_n^{\prime})\}$. Note that $\hat{\sigma}_{\textbf{y}}$ is a known scalar because of our knowledge regarding $\textbf{y}$. The normalization helps the S-NER method be invariant to data scaling. We use $S^{\prime}_n$ in the S-NER  model selection method for the linear regression problem presented in Subsection \ref{Sec51} to estimate the support set $\Sc$ (Algorithm \ref{NERALG}).  
    
As mentioned in Subsection \ref{Sec51},  in the S-NER algorithm, $\hat{K}$  is the estimated order of the model as the most parsimonious model containing risk minimizer. In the S-NER method, the number of models in the sorted sequentially nested models will be $\min{(\hat{K}+1,L)}$. We let $\bar{L}=\min{(\hat{K}+1,L)}$. Therefore, this algorithm can expand the models $\bar{L}$ times  (equivalently expands the set of active features). For every expansion procedure, this algorithm assigns a model from a set of candidate models and verifies it as an expanded model using the test \eqref{TESTEQ}. Based on Subsection \ref{Sec51}, for every $k\in\{1,2,...,\bar{L}\}$, the test $T_k$ is used to verify the assigned model $\Mc_{\hat{j}}^{\prime}$ as an expansion of  the $(k-1)$-th approved model $\Mc_{k-1}$ of the previous step. The  test $T_k$ in algorithm \ref{NERALG} for the linear regression problem is as follows
\small{
\begin{equation}\label{TESTAPP}
    T_k =\mathbb{I}\Big\{ R_{emp}(S_n,\hat{\Ptheta}_{k-1})-R_{emp}(S_n,\hat{\Ptheta}^{\prime}_{\hat{j}}) \geq c\frac{\sigma^2}{n}F_1^{-1}\Big(1-c_1\Big(\frac{{\sigma}^2}{n}\Big)\Big) \Big\},
\end{equation}
}
\normalsize
where $R_{emp}(S_n,\hat{\Ptheta}_{k-1})$ is the minimum empirical risk of model $\Mc_{k-1}$  and $R_{emp}(S_n,\hat{\Ptheta}^{\prime}_{\hat{j}})$ is the minimum empirical risk of assigned model $\Mc_{\hat{j}}^{\prime}$ with the least minimum empirical risk among candidate models of the $k$-th model in the sorted sequentially nested model family.
However, in this problem, the noise variance $\sigma^2$ is not available. Therefore, for every $k\in\{1,2,...,\bar{L}\}$, we estimate the noise variance using the $k$-th assigned model residual $\textbf{r}_k=\textbf{y}-\hat{\textbf{y}}_k$, 
\begin{equation}\label{estsigma}
	\hat{\sigma}^2_k = \frac{\|\textbf{r}_k-\hat{\mu}_{\textbf{r}_k}\|_2^2}{n-1},
\end{equation} 
where  $\hat{\textbf{y}}_k=({\textbf{X}^{\mathcal{I}_{\hat{j}}^{\prime}}})^T ({\textbf{X}^{\mathcal{I}_{\hat{j}}^{\prime}}})^{\dagger}\textbf{y}$ and ${\mathcal{I}_{\hat{j}}^{\prime}}$ is the set of parameter indices of the assigned model $\Mc_{\hat{j}}^{\prime}$ in the $k$-th iteration of the algorithm. Also,  $\hat{\mu}_{\textbf{r}_k}=\frac{1}{n}\sum_{i=1}^{n} [\textbf{r}_k]_i$ is the sample mean of residuals. 

In Algorithm  \ref{NERALG},  since we use model sorting, for every $k\in\{1,2,...,\bar{L}\}$, we have prior knowledge that the parameter of the $(k-1)$-th model is more valuable than the added parameter into the $k$-th model. Based on this prior information from the sorting procedure, we must tune the constants of the threshold in the test of Algorithm \ref{NERALG} such that its threshold is stricter for the larger number of $k$. Indeed, it must be harder to accept the less valuable parameters since we sort the parameters from most valuable to least valuable through the S-NER model selection. In this regard, we let $c=k$ and $c_1 =c^{\prime} k$ in  \eqref{TESTAPP}, where $c^{\prime}$ is a parameter that will be adjusted using cross-validation. Based on this choice of  $c$ and $c_1$, as the model number $k$ grows, the threshold will be stricter as it becomes larger. Therefore, for every $k\in\{1,2,...,\bar{L}\}$, using $c=k$, $c_1 =c^{\prime} k$, and $\sigma^2={\hat{\sigma}}^2_k$ in the test of Algorithm \ref{NERALG}, \eqref{TESTAPP} will be rewritten as follows
\small{
\begin{flalign}
    {T}_k=\mathbb{I}\Big\{ R_{emp}(S_n,\hat{\Ptheta}_{k-1})&-R_{emp}(S_n,\hat{\Ptheta}^{\prime}_{\hat{j}})\geq c^{\prime} k\frac{\hat{\sigma}^2_k}{n}F_1^{-1}\Big(1-k\Big(\frac{\hat{\sigma}^2_k}{n}\Big)\Big)\Big\}.
\end{flalign}}
\normalsize
As mentioned above, $c^{\prime}$ will be adjusted using cross-validation.  However, we observe that the estimated parameter vector using the cross-validation procedure in the S-NER has an over-fitting problem, i.e., it considers more features in the linear model as redundant features. Therefore, in our algorithm, to overcome this over-fitting problem, we eliminate the features whose absolute weights are less than a threshold.  The S-NER  model selection method considers the set of remaining indices of the parameters as the estimation of the support set, i.e., $\hat{\mathcal{S}}$.

We compare the S-NER model selection method with the state-of-the-art model selection methods. In the following, for the first time in the model selection literature, we will show that the S-NER model selection method, without any prior information, can outperform the accuracy of OMP and  LARS feature sorting algorithms aided by the support set cardinality $K$. The accuracy of OMP and LARS algorithms aided with the prior knowledge of the support set cardinality $K$ are upper bounds for the accuracy of the state-of-the-art model selection methods that use these feature sorting algorithms. 

In most model selection methods, a penalty is considered for the minimum empirical risk of every model to control the model complexity. However, in the high-dimensional linear regression problem, the number of sets of features with the maximum $k_{max}$ member will be $\sum_{k=1}^{k_{max}} {L\choose{k}}$. We can compose a linear model to fit the data corresponding to each set.  By increasing the number of features, $L$, the number of candidate models increases exponentially. Hence, calculating the empirical risk for all these candidates and performing model selection is highly complicated. Therefore, in the model selection methods, feature sorting algorithms are used to sort the most $k_{max}$ correlated features with the responses, $\textbf{y}$ \cite{owrang2018model,gohain2023robust,kallummil2018signal}.  OMP \cite{weisberg2005applied} and LARS \cite{efron2004least} are two important examples of feature sorting algorithms used widely in the literature. In the model selection methods, the linear models are composed based on OMP or LARS sorted features, and then the order of the model is estimated using penalized model selection methods where the penalty functions are calculated based on model complexity and data. Different model selection methods proposed different penalties mostly based on an information criterion.  In this paper, we compare the performance of our proposed method, S-NER, with EFIC \cite{owrang2018model} and EBICR \cite{gohain2023robust} model selection methods. To do this, we compare the probability of true detection of the support set, $\mathbb{P}\{\hat{\mathcal{S}}=\mathcal{S}\}$ (probability of true detection in brief), of the methods. Also, we compute the probability of true detection for feature sorting algorithms (OMP and LARS) with the prior knowledge of support set cardinality $K$. We consider the indices of the first $K$ sorted features based on  OMP or LARS as the estimation of the support set $\Sc$. In other words,  these sorting algorithms are aided with the cardinality $K$  that we call aided OMP and aided LARS, respectively. The true detection probability of aided OMP and aided LARS are upper bounds for the probability of true detection of the model selection methods like EFIC and EBICR, which use OMP and LARS feature sorting algorithms, respectively. 

In addition, to calculate the accuracy of model sorting in the S-NER model selection method, we present the accuracy of S-NER that aided with the cardinality of support set $\Sc$ called aided S-NER. In this regard, we let $\gamma_k^R(S_n,\delta)=0$ and go through Algorithm \ref{NERALG}  for only $K$ iterations. For these conditions in Algorithm \ref{NERALG}, $\hat{\mathcal{S}}=\hat{\mathcal{I}}$ is the estimated support set using the S-NER method aided with the support set cardinality $K$.

\begin{figure}[!t]
    \centering
    \subfloat[]
{\includegraphics[width=8cm, height=6cm]{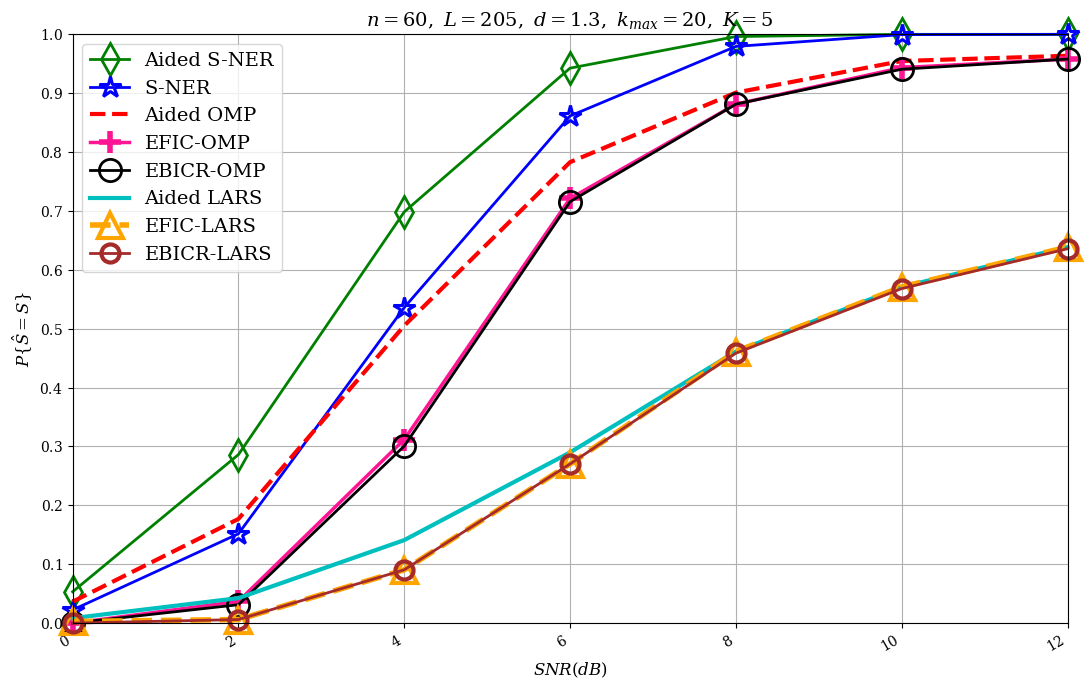}\label{fig3}}
 \subfloat[]   {\includegraphics[width=8cm, height=6cm]{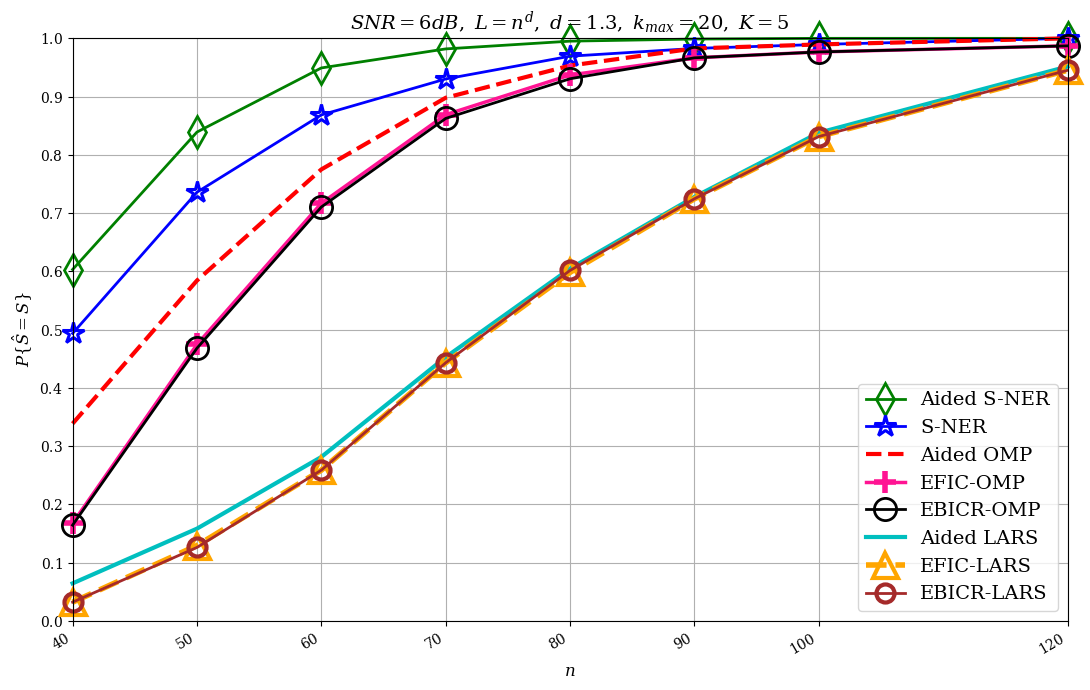}\label{fig4}} 
        \caption{Comparison of the aided S-NER, S-NER  model selection method, aided OMP, aided LARS, and EFIC and EBICR methods using OMP and LARS as the feature sorting algorithm.  (a) Comparison of true detection probability versus different SNR for the number of observations  $n=60$, the number of features  $L=205$, and the order of the true model $K=5$.  (b) Comparison of true detection probability versus the different number of observations $n$ for SNR  $6$ dB, the number of features $L=\lceil n^{1.3}\rceil$, and the order of the true model $K=5$. }
\end{figure}


Fig.\ref{fig3} presents the probability of true detection of the aided S-NER, S-NER model selection methods, aided OMP, aided LARS, EFIC using OMP and LARS feature sorting, and EBICR using OMP and LARS feature sorting methods for different signal-to-noise ratios (SNR). The SNR in the dB scale is calculated by $SNR(dB)=10\log_{10}{\Big( \frac{{\|\textbf{X}\pmb{\theta}^*\|}_2^2}{n\sigma^2} \Big)}$.  In these simulations, the number of observations is $n=60$, the number of features is 205 ($d=1.3$), the cardinality of the true support set is $5$, and $k_{max}=20$. As illustrated in this figure, the S-NER model selection method has better performance than EBICR and EFIC as the state-of-the-art methods. Also, for the first time in the literature, the S-NER model selection method, without any prior information, can outperform the accuracy of aided OMP and aided LARS. For example, in $SNR=6\ dB$, the S-NER  model selection method outperforms the true detection probability of EFIC and aided OMP with at least $14$ and $7$ percent, respectively. 

Note that the probability of true detection of aided OMP and aided LARS are the upper bounds for the performance of the other model order selection methods that use OMP and LARS as feature sorting.  This outperformance of S-NER ensures this method is capable of outperforming all of the model order selection methods that use OMP and LARS as feature sorting.

In Fig.\ref{fig4}, the probability of true detection of different model selection methods is illustrated for different observation numbers $n$. In this simulation, $SNR=6$, $d=1.3$, the cardinality of the true support set is $5$, and $k_{max}=20$.  As illustrated, the S-NER  model selection method has better accuracy than other model selection methods and outperforms the accuracy of aided OMP and aided LARS. In this simulation, for $n=50$, the true detection probability of S-NER is at least $25$ and $13$ percent better than EFIC using OMP-sorted features and aided OMP methods, respectively.   

We investigated the S-NER model selection in the low-dimensional scenario as well.  However, in this section, due to the importance of the high-dimensional scenario, we just report the simulation results in the high-dimensional scenarios.
In the next section, we use the NER order model selection method to select the valuable features of mini-ROCKET features in the UCR dataset classification.
\subsection{Feature Selection in the Classification of UCR Dataset}\label{Section6.2}
In this subsection, we use the NER model order selection as a feature selection method for classifying the UCR datasets using mini-ROCKET features \cite{dempster2021minirocket}. The UCR dataset consists of 109 time-series classification datasets\cite{dau2019ucr}. Each dataset consists of supervised training and test sample data. The mini-ROCKET extracts $9996$ features from each sample (time series). In the mini-ROCKET feature extraction method, 9996 almost deterministic convolutional kernels are applied to each data. Then, the proportion of positive values (PPV) of the output of each convolution operation is extracted as a scalar feature \cite{dempster2021minirocket}.  These generated kernels are the same for the whole data of every dataset. We consider $\textbf{x} \in [0,1]^L$ as an $L$-dimensional mini-ROCKET feature vector where $L=9996$, and scalar $y$ is its class label. We use the features extracted from training data and corresponding labels as the training data, $S_n=\{(\textbf{x}_1,y_1),(\textbf{x}_2,y_2),...,(\textbf{x}_n,y_n)\}$, for the training phase of the classification task. In the training phase, first, we use the NER model order selection method to select the features, and then we train the classifier based on the selected features. In \cite{dempster2021minirocket}, a multi-class ridge classifier is used to classify mini-ROCKET features. Also, we use the same classifier to investigate the effect of the proposed feature selection method on the classification task. 

The classification task based on the NER feature selection consists of three steps. First, we use a feature sorting algorithm to sort the 9996 mini-ROCKET features, then exploit the NER model order selection method to select the most valuable features. Finally, the classifier is trained and evaluated based on the selected features. In the following, we present the feature sorting algorithm.
\subsubsection{Feature Sorting in mini-ROCKET}  We exploit the weights of the multi-class ridge classifier for the feature selection task \cite{omidi2023reducing}. As mentioned, a multi-class ridge classifier will be trained  in the training phase. Using the weights of the multi-class ridge classifier for feature sorting will sort useful and valuable features corresponding to the training phase. Therefore, we train a $J$-class ridge classifier and use its weights for the feature sorting where $J$ is the number of classes in data. A $J$-class ridge classifier consists of $J$ one-versus-all binary classifiers. In this classifier, a ridge regression model will be trained. For this model, the scalar label $y\in \{1,2,...,J\}$ is transformed to a  $J$-dimensional one-hot vector $\textbf{y}^{\prime}$ where the $y$-th element of $\textbf{y}^{\prime}$ is one and the rest is zero. Then, matrix $\textbf{Y}=[\textbf{y}^{\prime}_1,\textbf{y}^{\prime}_2,...,\textbf{y}^{\prime}_n]^T$ is the $(n\times J)$-dimensional output matrix, and $\textbf{X}=[\textbf{x}_1,\textbf{x}_2,...,\textbf{x}_n]$ is the $(L\times n)$-dimensional input features matrix for the ridge regression model. Therefore, the ridge regression weights are calculated as follows
\begin{flalign}
    \hat{\textbf{W}}=(\textbf{X}\textbf{X}^T+\tau \textbf{I}_{L})^{-1}\textbf{X}\textbf{Y},
 \end{flalign}
 where  $\tau $ is the regularization coefficient that will be adjusted using cross-validation. $\hat{\textbf{W}}=[\hat{\textbf{w}}^{\{1\}},\hat{\textbf{w}}^{\{2\}},...,\hat{\textbf{w}}^{\{J\}}]$ is a  $(L\times J)$-dimensional matrix where for every $1\leq j\leq J $, $\hat{\textbf{w}}^{\{j\}}$,  the $j$-th column of $\hat{\textbf{W}}$, is the weight vector of the $j$-th class-versus-all binary ridge classifier. Now, we use the trained classifier's weights, $\hat{\textbf{W}}$,  to sort the mini-ROCKET features \cite{omidi2023reducing}. We sort these features based on the absolute values of the weights such that a feature with a bigger absolute weight will be considered to be more important. For every $j\in\{1,2,...,J\}$, let the indices of the sorted absolute weight corresponding to the $j$-th one-versus-all binary classifier be $\textbf{u}^j=[u_1^j,u_2^j,...,u_{L}^j]$, where $u_1^j,u_2^j,...,u_{L}^j\in \{1,2,...,L\}$, $u_1^j\neq u_2^j\neq...\neq u_{L}^j$, and 
\begin{flalign}
   |[\hat{\textbf{w}}^{\{j\}}]_{u_1^j}|\geq |[\hat{\textbf{w}}^{\{j\}}]_{u_2^j}|\geq...\geq |[\hat{\textbf{w}}^{\{j\}}]_{u_{L}^j}|.
\end{flalign}
 Now, let the sorted indices vector $\textbf{c}_{ridge}=[i_1,i_2,...,i_{L}]$  be  a version of  vector $\bar{\textbf{u}}=[u_1^1,...,u_1^j,u_2^1,...,u_2^j,...,u_{L}^1,...,u_{L}^j]$ without any duplicate member (from every $J$ duplicate member of $\bar{\textbf{u}}$, the first position is preserved in $\textbf{c}_{ridge}$, and the remains are deleted). 
 In this sorted indices vector $\textbf{c}_{ridge}$,  every one-versus-all binary classification of the ridge multi-class classifier has a productive role in sorting. We use the indices of $\textbf{c}_{ridge}$  to arrange a sequentially nested linear regression model class. 
 \subsubsection{The Group NER Model Order Selection}
At this step, the number of features is large ($L=9996$). Therefore, for every model order selection method, we must generate 9996 models, which is a large number of models, and the computations will be exhaustive. To reduce the number of models and consequently reduce the computations, we extend the models by a group of parameters instead of extending the models one parameter by one parameter. Then, we use the NER model order selection method to estimate the best model. We call this method the \textit{group NER model order selection}. Using this method speeds up the model selection process rather than the ordinary NER model order selection method.

Let the set of all parameters be $\Theta=\{\theta_1^{\prime},\theta_2^{\prime},...,\theta_L^{\prime}\}$.  We only consider the most $k_{max}=2500$ valuable features sorted by $\textbf{c}_{ridge}$  in the model selection, and we generate  $ k_0= 250$   sequentially nested models using groups of 10 parameters. In this regard, using the ridge-sorted indices vector  $\textbf{c}_{ridge}=[i_1,i_2,...,i_{L}]$, we organize the sequentially nested model class, $\{\mathcal{M}_k\}_{k=1}^{250}$. In this model class, for every $k\in\{1,2,...,250\}$, $\mathcal{M}_k=\{(\textbf{x}^{\mathcal{I}_k})^T \pmb{\theta}_{k}: \pmb{\theta}_{k}\in \mathbb{R}^{10k} \}$ where $\pmb{\theta}_{k}=\pmb{\theta}^\prime_{\mathcal{I}_k}$ and $\mathcal{I}_k=\{i_{1},i_{2},...,i_{10k}\}$.  We chose $10$ for grouping the parameters since we observed that it reduces the complexity without any significant loss in accuracy. In this problem, we use the ridge loss function presented as follows
 \begin{flalign}
     l(\textbf{x}^{\mathcal{I}_k},y,\pmb{\theta}_{k})=(y-(\textbf{x}^{\mathcal{I}_k})^T\pmb{\theta}_{k})^2+\tau \|\pmb{\theta}_{k}\|_2^2.
 \end{flalign}
  For this loss function, based on \eqref{minempriskparEQ}, the minimum empirical risk in the $k$-th model will be calculated by 
  \begin{flalign}
  R_{emp}(S_n,\hat{\Ptheta}_k)=\frac{\|\textbf{y}-(\textbf{X}^{\mathcal{I}_k})^T \hat{\Ptheta}_k\|_2^2}{n}+\tau \|\hat{\Ptheta}_{k}\|_2^2,
  \end{flalign}
  where $\textbf{y}=[y_1,y_2,...,y_n]^T$, $\textbf{X}^{\Ic_k}$ is a sub-matrix consisting of rows of $\textbf{X}$ indexed by $\Ic_k$, and $\hat{\pmb{\theta}}_{k}=(\textbf{X}^{\mathcal{I}_k}(\textbf{X}^{\mathcal{I}_k})^T+\tau \textbf{I}_{10k})^{-1}\textbf{X}^{\mathcal{I}_k} \textbf{y}$.    
  
  Notice that the bound $c\frac{{\hat{\sigma}}^2_k}{n} F_1^{-1}(1-c_1(\frac{{\hat{\sigma}}^2_k}{n})) $ in Theorem \ref{consistency} is calculated based on the least square loss function and for the SEER of two successive linear regression models that differ only by one parameter.  In this classification task, we use the ridge loss function, and two successive models differ by ten parameters. So, to modify this threshold for these conditions, we let parameters $c_1=1$ and use the cross-validation to adjust the parameter $c$ and consequently tune the threshold for the condition of this problem.  For every $k\in\{1,2,...250\}$, we calculate the minimum empirical risks of all models and obtain the SEERs. Then, for every $c\in  \{10^{-1},10^{-0.75},...,10^{+3}\}$, we let $c\frac{{\hat{\sigma}}^2_k}{n} F_1^{-1}(1-(\frac{{\hat{\sigma}}^2_k}{n})) $ as the threshold in the test \eqref{TESTEQ}, i.e. for every $k\in\{2,3,...250\}$,
  \begin{flalign}
  T_k^c=\mathbb{I}\Big\{\Delta R_{emp}(S_n,k)\geq c\frac{{\hat{\sigma}}^2_k}{n} F_1^{-1}\Big(1-\Big(\frac{{\hat{\sigma}}^2_k}{n}\Big)\Big)\Big\},
  \end{flalign}
  where ${\hat{\sigma}}^2_k$ is the estimation of the noise variance using the $k$-th model residual presented in \eqref{estsigma}.  Then, using  the NER model order selection in \eqref{KHATEQ}, we select a model as follows
  \begin{flalign}
      \hat{K}_c=\max\{k\in\{1,2,...250\}:T_k^c=1\},
  \end{flalign}
  where $\Mc_{\hat{K}_c}$ is the selected model based on threshold $c\frac{{\hat{\sigma}}^2_k}{n} F_1^{-1}(1-(\frac{{\hat{\sigma}}^2_k}{n})) $ and its corresponding features indices is $\Ic_{\hat{K}_c}$. Note that $T_1=1$ since, based on Corollary \ref{KASSUMPTION}, there is always a model index $K\geq 1$ that $\Mc_K$ contains the minimum risk.  Using cross-validation, we choose $\Ic_{{\hat{K}}_{\hat{c}}}$ with the minimum classification error based on validation data, where the training and evaluation of the classifier are presented in the next step. 
  As mentioned in Subsection \ref{Section6.1}, the estimated parameters confront over-fitting, which makes some redundant features. Similar to the previous subsection, to overcome this over-fitting, we eliminate features whose absolute values of their weights are smaller than a threshold. Therefore, we consider the set of remaining indices, $\Ic_{\hat{K}}$, as indices of the most valuable features by NER model order selection for the classification task.  In the following, we use these selected features for the training and evaluation of the classifier.	
 \subsubsection{Training and Test Phases}
 In the training phase, we train a multi-class ridge classifier only with the selected features indexed by $\Ic_{\hat{K}}$. Then, the training dataset will be transformed to $S^{\Ic_{\hat{K}}}_n=\{ (\textbf{x}^{\Ic_{\hat{K}}}_1,y_1),(\textbf{x}^{\Ic_{\hat{K}}}_2,y_2),...,(\textbf{x}^{\Ic_{\hat{K}}}_n,y_n) \}$, and the ridge regression weights  in the training classifier will be calculated as follows
 \begin{flalign}
 \hat{\textbf{W}}_{{\hat{K}}}=({\textbf{X}^{\Ic_{\hat{K}}}}({\textbf{X}^{\Ic_{\hat{K}}}})^T+\tau \textbf{I}_{10\hat{K}})^{-1}{{{\textbf{X}^{\Ic_{\hat{K}}}}}}\textbf{Y}.
 \end{flalign}
Now, we use the trained ridge regression weight  $ \hat{\textbf{W}}_{{\hat{K}}}$ to predict the label of the test samples. For a test sample $\textbf{x}$, we predict its label as follows
\begin{flalign}
    \hat{y}=\argmax_{i\in \{1,2,...,J\}}{{y}^{\prime}_i},
\end{flalign}
where ${y}^{\prime}_i$ is the $i$-th element of the $J$-dimensional vector $\textbf{y}^{\prime}=\big((\textbf{x}^{{\Ic_{\hat{K}}}})^T \hat{\textbf{W}}_{{\hat{K}}}\big)^T$. We evaluate the classifier using the NER-selected features using the test data prepared in the UCR dataset archive \cite{dau2019ucr}. We perform these simulations ten times for every 109 datasets. In addition, we compare the classification accuracy of the NER feature selection method with AIC, EFIC, BIC, and EBICR methods that use the OMP algorithm as the sorting features algorithm. We let $k_{max}=2500$ in all these methods. We report the mean of accuracy and scaled kernel numbers on these ten trials of all 109 datasets. 
    \begin{figure}
	\centering
	\includegraphics[width=12cm, height=6.5cm]{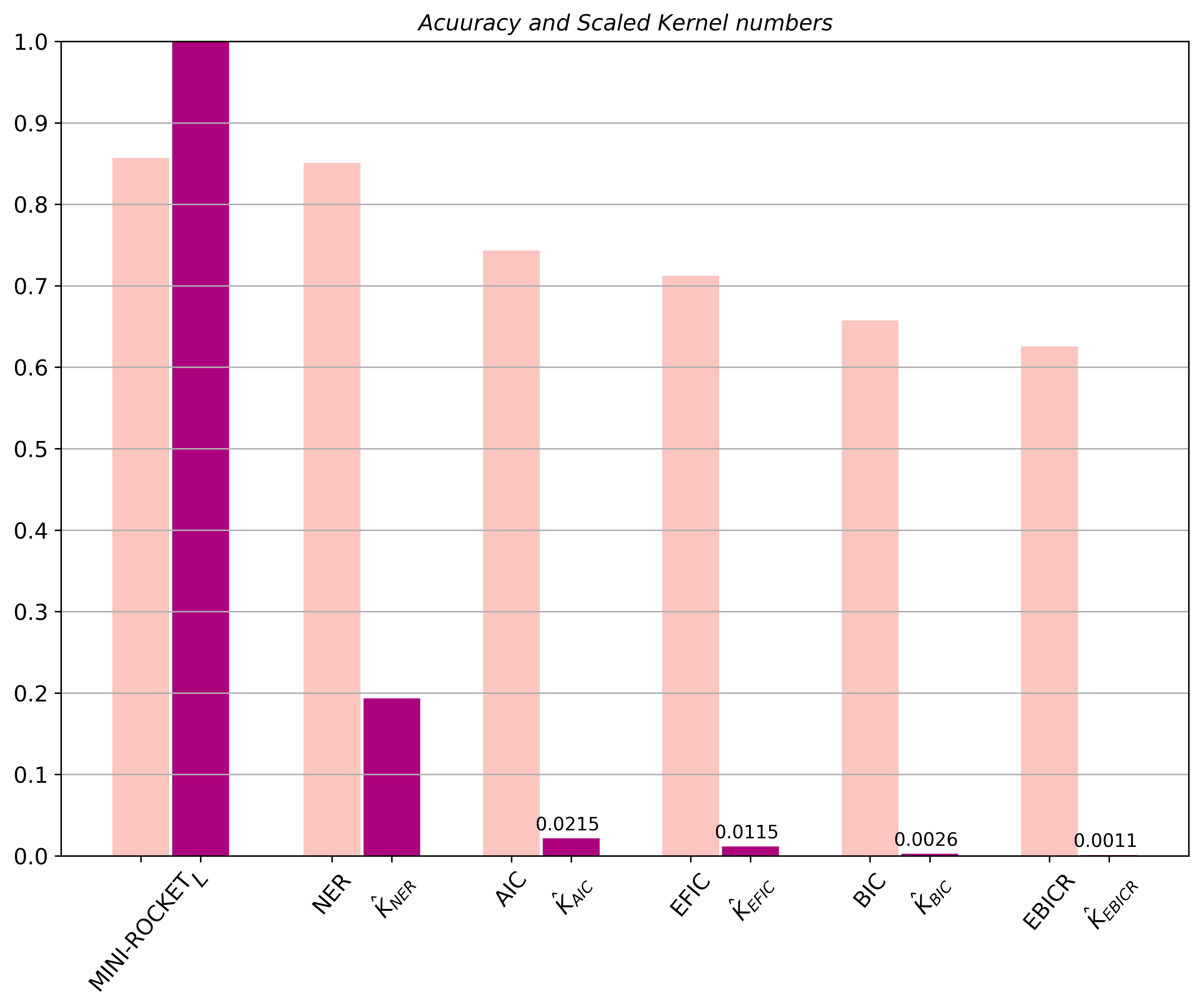}
	\caption{Accuracy and scaled kernels number to 9996 of mini-ROCKET and NER in the UCR dataset versus  EFIC,  AIC, BIC, and EBICR model selection methods feature selection methods. }\label{fig100}
\end{figure}

In the OMP or LARS algorithms, features are sorted based on the correlation between the residual $\textbf{y}-\hat{\textbf{y}}_k$ and features \cite{efron2004least,weisberg2005applied}, where $\hat{\textbf{y}}_k$ is the estimated response based on the first $k$ sorted features. In these algorithms,  $\hat{\textbf{y}}_k$  is calculated using the least square linear regression model without any regularization. Therefore,  for $k\geq n$, the residual in the $k$-th  iteration of the OMP algorithm will be zero. Also, for some $k$ near $n$,  the residual in the $k$-th iteration of the LARS algorithm will be zero. Since these algorithms do not use any regularization, they can only sort about the $n$ most valuable features.  In this regard, consider a high-dimensional scenario where $L$ and $k_{max}$ are greater than the number of training observations $n$. This scenario occurs in the feature selection of most UCR datasets where $n\leq k_{max}\leq L$. For example, the training sample numbers of some UCR datasets are $20$ and $30$ (Beetlefly or Beef datasets), which are much fewer than $k_{max}=2500$ \cite{dau2019ucr}. Therefore, in the AIC, EFIC, BIC, and EBICR methods that use the OMP or LARS algorithms to sort the features, $k_{max}$ will be limited to a maximum number of features that these algorithms can sort and is a number about $n$. However, the group NER model order selection does not have such limitations on $k_{max}$ since we do not use greedy algorithms like OMP or LARS, and use the ridge loss as a regularized loss.
 
In Fig. \ref{fig100}.a, we compare the accuracy of the classification task using the mini-ROCKET feature extraction with the NER, AIC, EFIC, BIC, and EBICR feature selection algorithms averaged out on ten trials of 109 UCR datasets. Also, the kernel number of each method used in the training and test phase is scaled to 9996 and reported in this figure. As illustrated, the NER method reduces the number of kernels by more than 89 percent without a significant loss in accuracy (less than 2 percent). However, in the best case, using the AIC feature selection causes at least a 13 percent accuracy loss. 
\subsubsection{Ablation Study}
In this subsection, to evaluate the performance of NER feature selection, we compare this method with well-known methods such as ANOVA \cite{guyon2003introduction} and the Chi-square test \cite{guyon2003introduction}. Since these methods require the number of selected features as prior information, they cannot independently estimate this quantity. We exploit NER order selection to estimate the number of valuable features of ANOVA and the Chi-square test. Therefore, for a fair comparison, we aid each method with the number of features predicted by NER, denoted as $\hat{K}_{NER}$.

Table \ref{tab:placeholder} presents the classification accuracy and the number of kernels obtained using mini-ROCKET, NER, ANOVA aided with $\hat{K}_{NER}$, and Chi-square aided with $\hat{K}_{NER}$ features, averaged over ten trials on 109 UCR datasets. Based on Table \ref{tab:placeholder}, NER outperforms the accuracy of aided ANOVA and aided Chi-square feature selection methods by more than 1.8 percent.

\begin{table}[!t]
    \centering
    \caption{Accuracy and kernel number of mini-ROCKET, NER, ANOVA, and Chi-test feature selection methods in the UCR dataset. }
    \begin{tabular}{|c|c|c|c|c|}\hline
         Method&  Mini-ROCKET&  NER&  Aided ANOVA&  Aided Chi-Square\\\hline
         Accuracy&  85.52&  85.09&  83.21&  81.07\\\hline
         Kernel Number&  9996&  1933&  1933&  1933\\ \hline
    \end{tabular}
    \label{tab:placeholder}
\end{table}
\subsection{Band Selection in  Hyperspectral Images}\label{Section6.3}

Hyperspectral Images (HSI) are captured over hundreds of narrow spectral bands with different wavelengths. The reflection of land surfaces across these bands forms the HSI, which contains rich spatial and spectral information. Classification is an important process for HSI images. However,
data from some spectral bands may be highly correlated, and not all spectral bands are equally informative for the classification task. In \cite{hajiani2026random}, to eliminate this redundancy, we employ the S-NER model selection method to find the most parsimonious set of HSI bands for this task. We identify the
minimal subset of bands that will yield the optimal linear model for HSI classification. We exploit the S-NER algorithm for the linear regression problem detailed in Subsection 6.1 and the S-NER test (51) to find the most parsimonious set of valuable bands.  We call the method combining S-NER band selection and Mini-ROCKET classification methods, intelligent band selection and random sampling (iBEARDS) classification method. The detailed iBEARDS method explanation and performance results are presented in \cite{hajiani2026random}. The comparison of classification accuracy and complexity between iBEARDS and state-of-the-art deep neural networks for HSI classification on three well-known HSI datasets, Pavia University, Salinas, and HongHu, shows that iBEARDS outperforms these methods in overall classification accuracy. Also, the training time of the iBEARDS method that uses S-NER-selected bands is at least 2 times faster than the Mini-ROCKET method that uses all HSI bands, without any significant loss in accuracy \cite{hajiani2026random}. It is almost 3, 4, and 2 times faster than the best comparable method on the Pavia University, Salinas, and HongHu datasets, respectively, while improving its classification accuracy \cite{hajiani2026random}.

\section{Conclusion}\label{Sec7}
In this paper, first, we provide a precise definition of nested models and present the properties of minimum risks and minimum empirical risk in a nested model class. We introduce an important criterion, SEER, and based on the nested models' properties, we prove the existence of $o(1)$ bounds on the SEER.  We use these bounds to predict the order of the model in a sequentially nested model class called the NER model order selection method. Also, we propose the S-NER model selection method to sort the valuable parameter space simultaneously, eliminate useless parameter spaces, and obtain the most parsimonious model containing the global risk minimizer predictor. We show that our proposed methods are consistent. The SEER bound is presented exclusively for the linear regression problem, and we use the s-NER method for model selection in linear regression applications. We show that the S-NER algorithm has better performance than the state-of-the-art algorithms in the literature for regression problems using synthetic data. NER and S-NER methods significantly decrease the classification complexity of UCR and HSI datasets, respectively.



\vskip 0.2in
\begin{spacing}{1.0}
    \bibliography{./REFGBB}

\begin{thebibliography}{10}
\providecommand{\url}[1]{#1}
\csname url@samestyle\endcsname
\providecommand{\newblock}{\relax}
\providecommand{\bibinfo}[2]{#2}
\providecommand{\BIBentrySTDinterwordspacing}{\spaceskip=0pt\relax}
\providecommand{\BIBentryALTinterwordstretchfactor}{4}
\providecommand{\BIBentryALTinterwordspacing}{\spaceskip=\fontdimen2\font plus
\BIBentryALTinterwordstretchfactor\fontdimen3\font minus \fontdimen4\font\relax}
\providecommand{\BIBforeignlanguage}[2]{{%
\expandafter\ifx\csname l@#1\endcsname\relax
\typeout{** WARNING: IEEEtran.bst: No hyphenation pattern has been}%
\typeout{** loaded for the language `#1'. Using the pattern for}%
\typeout{** the default language instead.}%
\else
\language=\csname l@#1\endcsname
\fi
#2}}
\providecommand{\BIBdecl}{\relax}
\BIBdecl

\bibitem{ding2018model}
J.~Ding, V.~Tarokh, and Y.~Yang, ``Model selection techniques: An overview,'' \emph{IEEE Signal Processing Magazine}, vol.~35, no.~6, pp. 16--34, 2018.

\bibitem{falk2026computation}
S.~Falk, N.~K. Corr{\^e}a, S.~Luccioni, L.~Biber-Freudenberger, and A.~van Wynsberghe, ``From computation to environmental cost: The resource burden of artificial intelligence,'' \emph{Communications Earth \& Environment}, vol.~7, no.~1, p. 114091, 2026.

\bibitem{mishra2025sustainability}
V.~Mishra, D.~Saxena, K.~Gupta, S.~Patni, and A.~Kumar, ``Sustainability in large language model supply chains: Insights and recommendations using analysis of utility for affecting factors,'' \emph{Scientific Reports}, vol.~15, no.~1, p. 33524, 2025.

\bibitem{peykani2026green}
P.~Peykani, A.~Emrouznejad, F.~Ramezanlou, C.~Tanasescu, and S.~Ghanidel, ``Green artificial intelligence: A comprehensive review of metrics, tools, challenges, trends, and future prospects,'' \emph{Archives of Computational Methods in Engineering}, vol.~33, pp. 9767--9808, 2026.

\bibitem{hastie2009elements}
T.~Hastie, R.~Tibshirani, J.~H. Friedman, and J.~H. Friedman, \emph{The Elements of Statistical Learning: Data Mining, Inference, and Prediction}.\hskip 1em plus 0.5em minus 0.4em\relax Springer, 2009.

\bibitem{cruciani2025choosing}
E.~Cruciani and R.~Verdecchia, ``Choosing to be green: Advancing green {AI} via dynamic model selection,'' in \emph{Proceedings of the Green-Aware AI Workshop at ECAI}, 2025, pp. 1--12.

\bibitem{brunton2024engineering}
J.~N. Kutz and S.~L. Brunton, ``Model selection and sparse identification in nonlinear mechanical and structural engineering systems,'' \emph{Mechanical Systems and Signal Processing}, vol. 210, p. 111180, 2024.

\bibitem{mohr2023fast}
F.~Mohr and J.~N. van Rijn, ``Fast and informative model selection using learning curve cross-validation,'' \emph{IEEE Transactions on Pattern Analysis and Machine Intelligence}, vol.~45, no.~8, pp. 9669--9680, 2023.

\bibitem{hajiani2026random}
M.~A. Hajiani and B.~Seyfe, ``Random sampling in hyperspectral image classification with supervised band selection,'' \emph{Pattern Recognition}, vol. 180, p. 114091, 2026.

\bibitem{giannone2024model}
D.~Giannone, M.~Lenza, and G.~E. Primiceri, ``Model selection and predictive performance in high-dimensional financial econometrics,'' \emph{Journal of Econometrics}, vol. 244, no.~2, p. 105820, 2024.

\bibitem{kim1999has}
C.~J. Kim and C.~R. Nelson, ``Has the {US} economy become more stable? a {B}ayesian approach based on a {M}arkov-switching model of the business cycle,'' \emph{Review of Economics and Statistics}, vol.~81, no.~4, pp. 608--616, 1999.

\bibitem{greenland1989modeling}
S.~Greenland, ``Modeling and variable selection in epidemiologic analysis.'' \emph{American Journal of Public Health}, vol.~79, no.~3, pp. 340--349, 1989.

\bibitem{cramer2025epidemiology}
E.~Y. Cramer, E.~L. Ray, and N.~G. Reich, ``Model selection strategies and complexity trade-offs in infectious disease transmission forecasting,'' \emph{PLOS Computational Biology}, vol.~21, no.~3, p. 1012845, 2025.

\bibitem{johnson2004model}
J.~B. Johnson and K.~S. Omland, ``Model selection in ecology and evolution,'' \emph{Trends in Ecology \& Evolution}, vol.~19, no.~2, pp. 101--108, 2004.

\bibitem{burnham2025ecological}
C.~F. Dormann, C.~E. Grueber, and S.~Nakagawa, ``Information-theoretic model selection and multimodel inference in modern ecology: Recent advances and best practices,'' \emph{Methods in Ecology and Evolution}, vol.~16, no.~2, pp. 284--298, 2025.

\bibitem{owrang2018model}
A.~Owrang and M.~Jansson, ``A model selection criterion for high-dimensional linear regression,'' \emph{IEEE Transactions on Signal Processing}, vol.~66, no.~13, pp. 3436--3446, 2018.

\bibitem{li2024feature}
J.~Li, K.~Cheng, S.~Wang, and H.~Liu, ``Feature selection and sparse model selection for high-dimensional data: A systematic review and benchmarking,'' \emph{IEEE Transactions on Knowledge and Data Engineering}, vol.~36, no.~8, pp. 3890--3911, 2024.

\bibitem{li2025globally}
X.~Li and H.~Zhao, ``Globally-optimal greedy active sequential estimation,'' \emph{IEEE Transactions on Information Theory}, vol.~71, no.~5, pp. 3871 -- 3924, 2025.

\bibitem{nandy2024bayes}
S.~Nandy and S.~Sen, ``Bayes optimal learning in high-dimensional linear regression with network side information,'' \emph{IEEE Transactions on Information Theory}, vol.~70, no.~1, pp. 565--591, 2024.

\bibitem{chen2023estimation}
X.~Chen, Z.~Feng, and H.~Peng, ``Estimation and order selection for multivariate exponential power mixture models,'' \emph{Journal of Multivariate Analysis}, vol. 195, p. 105140, 2023.

\bibitem{nguyen2026mixture}
H.~D. Nguyen and T.~Nguyen, ``Modifications of the {BIC} for order selection in finite mixture models,'' \emph{Annals of the Institute of Statistical Mathematics}, vol.~78, no.~3, pp. 1--40, 2026.

\bibitem{rivals2000statistical}
I.~Rivals and L.~Personnaz, ``A statistical procedure for determining the optimal number of hidden neurons of a neural model,'' in \emph{Second International Symposium on Neural Computation (NC’2000)}, 2000, pp. 23--26.

\bibitem{ren2024nas}
P.~Ren, Y.~Xiao, X.~Chang, P.-Y. Huang, Z.~Li, X.~Chen, and X.~Wang, ``Resource-constrained neural architecture search: Automated selection of layers and neurons under hardware constraints,'' \emph{IEEE Transactions on Neural Networks and Learning Systems}, vol.~35, no.~7, pp. 8912--8926, 2024.

\bibitem{asadi2018signal}
H.~Asadi and B.~Seyfe, ``Signal enumeration in {G}aussian and non-{G}aussian noise using entropy estimation of eigenvalues,'' \emph{Digital Signal Processing}, vol.~78, pp. 163--174, 2018.

\bibitem{yi2022source}
H.~Yi, ``Source enumeration via {RMT} estimator with adaptive decision criterion based on linear shrinkage estimation of noise eigenvalues using relatively few samples,'' \emph{IET Signal Processing}, vol.~16, no.~1, pp. 26--44, 2022.

\bibitem{ge2025source}
S.~Ge, X.~Fei, P.~Zhang, and W.~Wang, ``A systematic literature review of source number estimation in multi-sensor array signal processing,'' \emph{IEEE Access}, vol.~13, pp. 104\,756--104\,778, 2025.

\bibitem{rissanen1978modeling}
J.~Rissanen, ``Modeling by shortest data description,'' \emph{Automatica}, vol.~14, no.~5, pp. 465--471, 1978.

\bibitem{yamanishi2023learning}
K.~Yamanishi, \emph{Learning with the Minimum Description Length Principle}.\hskip 1em plus 0.5em minus 0.4em\relax Springer, 2023.

\bibitem{vapnik1998statistical}
V.~Vapnik, \emph{Statistical Learning Theory}.\hskip 1em plus 0.5em minus 0.4em\relax Wiley-Interscience, 1998.

\bibitem{qu2025model}
C.~Qu, J.~He, X.~Duan, and J.~Chen, ``Model selection for inverse reinforcement learning via structural risk minimization,'' in \emph{2025 IEEE 64th Conference on Decision and Control (CDC)}, 2025, pp. 6150--6155.

\bibitem{akaike1974new}
H.~Akaike, ``A new look at the statistical model identification,'' \emph{IEEE Transactions on Automatic Control}, vol.~19, no.~6, pp. 716--723, 1974.

\bibitem{schwarz1978estimating}
G.~Schwarz, ``Estimating the dimension of a model,'' \emph{The Annals of Statistics}, vol.~6, no.~2, pp. 461--464, 1978.

\bibitem{gohain2023robust}
P.~B. Gohain and M.~Jansson, ``Robust information criterion for model selection in sparse high-dimensional linear regression models,'' \emph{IEEE Transactions on Signal Processing}, vol.~71, pp. 2251--2266, 2023.

\bibitem{efron2004least}
B.~Efron, T.~Hastie, I.~Johnstone, and R.~Tibshirani, ``Least angle regression,'' \emph{The Annals of Statistics}, vol.~32, no.~2, p. 407–499, 2004.

\bibitem{yin2025lowrank}
M.~Yin and M.~A. M{\"u}ller, ``Low-rank matrix regression via least-angle regression,'' \emph{IEEE Control Systems Letters}, vol.~9, pp. 637--642, 2025.

\bibitem{weisberg2005applied}
S.~Weisberg, \emph{Applied Linear Regression}.\hskip 1em plus 0.5em minus 0.4em\relax John Wiley \& Sons, 2005.

\bibitem{tohidi2025revisiting}
E.~Tohidi, M.~Coutino, S.~P. Chepuri, and G.~Leus, ``Revisiting matching pursuit: Beyond approximate orthogonal projections and stopping criteria,'' \emph{Signal Processing}, vol. 228, p. 109724, 2025.

\bibitem{kallummil2018signal}
S.~Kallummil and S.~Kalyani, ``Signal and noise statistics oblivious orthogonal matching pursuit,'' in \emph{International Conference on Machine Learning}.\hskip 1em plus 0.5em minus 0.4em\relax PMLR, 2018, pp. 2429--2438.

\bibitem{vapnik1999nature}
V.~Vapnik, \emph{The Nature of Statistical Learning Theory}.\hskip 1em plus 0.5em minus 0.4em\relax Springer Science \& Business Media, 1999.

\bibitem{mohri2018foundations}
M.~Mohri, A.~Rostamizadeh, and A.~Talwalkar, \emph{Foundations of Machine Learning}.\hskip 1em plus 0.5em minus 0.4em\relax MIT Press, 2018.

\bibitem{pesaran1978testing}
M.~H. Pesaran and A.~S. Deaton, ``Testing non-nested nonlinear regression models,'' \emph{Econometrica: Journal of the Econometric Society}, vol.~46, no.~3, pp. 677--694, 1978.

\bibitem{marcellino2008model}
M.~Marcellino and B.~Rossi, ``Model selection for nested and overlapping nonlinear, dynamic and possibly mis-specified models,'' \emph{Oxford Bulletin of Economics and Statistics}, vol.~70, pp. 867--893, 2008.

\bibitem{pei2022local}
Z.~Pei, D.~S. Lee, D.~Card, and A.~Weber, ``Local polynomial order in regression discontinuity designs,'' \emph{Journal of Business \& Economic Statistics}, vol.~40, no.~3, pp. 1259--1267, 2022.

\bibitem{moon2021ar}
J.~Moon, M.~B. Hossain, and K.~H. Chon, ``{A}{R} and {A}{R}{M}{A} model order selection for time-series modeling with imagenet classification,'' \emph{Signal Processing}, vol. 183, p. 108026, 2021.

\bibitem{omidi2026random}
M.~A. Omidi, B.~Seyfe, and S.~Valaee, ``Random feature pruning for multi-class classification,'' \emph{arXiv preprint arXiv:2610.15101}, 2026.

\bibitem{li2019matching}
F.~Li, C.~M. Triggs, B.~Dumitrescu, and C.~D. Giurc{\u{a}}neanu, ``The matching pursuit algorithm revisited: A variant for big data and new stopping rules,'' \emph{Signal Processing}, vol. 155, pp. 170--181, 2019.

\bibitem{lu2012generalized}
Z.~Lu and A.~M. Zoubir, ``Generalized {B}ayesian information criterion for source enumeration in array processing,'' \emph{IEEE Transactions on Signal Processing}, vol.~61, no.~6, pp. 1470--1480, 2012.

\bibitem{barthelme2021machine}
A.~Barthelme and W.~Utschick, ``A machine learning approach to {DoA} estimation and model order selection for antenna arrays with subarray sampling,'' \emph{IEEE Transactions on Signal Processing}, vol.~69, pp. 3075--3087, 2021.

\bibitem{wang1998blind}
X.~Wang and H.~V. Poor, ``Blind multiuser detection: A subspace approach,'' \emph{IEEE Transactions on Information Theory}, vol.~44, no.~2, pp. 677--690, 1998.

\bibitem{chapelle1999model}
O.~Chapelle and V.~Vapnik, ``Model selection for support vector machines,'' \emph{Advances in Neural Information Processing Systems}, vol.~12, pp. 230--236, 1999.

\bibitem{ephraim1995signal}
Y.~Ephraim and H.~L. Van~Trees, ``A signal subspace approach for speech enhancement,'' \emph{IEEE Transactions on Speech and Audio Processing}, vol.~3, no.~4, pp. 251--266, 1995.

\bibitem{alquier2021user}
P.~Alquier, ``User-friendly introduction to {P}{A}{C}-{B}ayes bounds,'' \emph{Foundations and Trends in Machine Learning}, vol.~17, no.~2, p. 174–303, 2024.

\bibitem{shawe1998structural}
J.~Shawe-Taylor, P.~L. Bartlett, R.~C. Williamson, and M.~Anthony, ``Structural risk minimization over data-dependent hierarchies,'' \emph{IEEE Transactions on Information Theory}, vol.~44, no.~5, pp. 1926--1940, 1998.

\bibitem{koltchinskii2001rademacher}
V.~Koltchinskii, ``Rademacher penalties and structural risk minimization,'' \emph{IEEE Transactions on Information Theory}, vol.~47, no.~5, pp. 1902--1914, 2001.

\bibitem{hajiani2023oracle}
M.~A. Hajiani and B.~Seyfe, ``From oracle generalization bound toward empirical inequality,'' \emph{Information Sciences}, p. 119131, 2023.

\bibitem{dempster2020rocket}
A.~Dempster, F.~Petitjean, and G.~I. Webb, ``{ROCKET}: {E}xceptionally fast and accurate time series classification using random convolutional kernels,'' \emph{Data Mining and Knowledge Discovery}, vol.~34, no.~5, pp. 1454--1495, 2020.

\bibitem{dau2019ucr}
H.~A. Dau, A.~Bagnall, K.~Kamgar, C.~C.~M. Yeh, Y.~Zhu, S.~Gharghabi, C.~A. Ratanamahatana, and E.~Keogh, ``The {UCR} time series archive,'' \emph{IEEE/CAA Journal of Automatica Sinica}, vol.~6, no.~6, pp. 1293--1305, 2019.

\bibitem{sipser1996introduction}
M.~Sipser, ``Introduction to the theory of computation,'' \emph{ACM Sigact News}, vol.~27, no.~1, pp. 27--29, 1996.

\bibitem{wainwright2019high}
M.~J. Wainwright, \emph{High-Dimensional Statistics: A Non-Asymptotic Viewpoint}.\hskip 1em plus 0.5em minus 0.4em\relax Cambridge University Press, 2019, vol.~48.

\bibitem{tropp2006just}
J.~A. Tropp, ``Just relax: Convex programming methods for identifying sparse signals in noise,'' \emph{IEEE Transactions on Information Theory}, vol.~52, no.~3, pp. 1030--1051, 2006.

\bibitem{tibshirani1996regression}
R.~Tibshirani, ``Regression shrinkage and selection via the {LASSO},'' \emph{Journal of the Royal Statistical Society Series B: Statistical Methodology}, vol.~58, no.~1, pp. 267--288, 1996.

\bibitem{dai2009subspace}
W.~Dai and O.~Milenkovic, ``Subspace pursuit for compressive sensing signal reconstruction,'' \emph{IEEE Transactions on Information Theory}, vol.~55, no.~5, pp. 2230--2249, 2009.

\bibitem{chen2008extended}
J.~Chen and Z.~Chen, ``Extended bayesian information criteria for model selection with large model spaces,'' \emph{Biometrika}, vol.~95, no.~3, pp. 759--771, 2008.

\bibitem{loeve1977elementary}
M.~Lo{\`e}ve, \emph{Elementary Probability Theory}.\hskip 1em plus 0.5em minus 0.4em\relax Springer, 1977.

\bibitem{zhang2008sparsity}
C.-H. Zhang and J.~Huang, ``The sparsity and bias of the {LASSO} selection in high-dimensional linear regression,'' \emph{The Annals of Statistics}, vol.~36, no.~4, p. 1567–1594, 2008.

\bibitem{dempster2021minirocket}
A.~Dempster, D.~F. Schmidt, and G.~I. Webb, ``Minirocket: A very fast (almost) deterministic transform for time series classification,'' in \emph{Proceedings of the 27th ACM SIGKDD Conference on Knowledge Discovery \& Data Mining}, 2021, pp. 248--257.

\bibitem{omidi2023reducing}
M.~A. Omidi, B.~Seyfe, and S.~Valaee, ``Reducing the computational complexity of learning with random convolutional features,'' in \emph{ICASSP 2023-2023 IEEE International Conference on Acoustics, Speech and Signal Processing (ICASSP)}.\hskip 1em plus 0.5em minus 0.4em\relax IEEE, 2023, pp. 1--5.

\bibitem{guyon2003introduction}
I.~Guyon and A.~Elisseeff, ``An introduction to variable and feature selection,'' \emph{Journal of Machine Learning Research}, vol.~3, no. Mar, pp. 1157--1182, 2003.

\bibitem{proakis2008digital}
J.~G. Proakis and M.~Saleh, \emph{Digital Communications}.\hskip 1em plus 0.5em minus 0.4em\relax McGraw-Hill, Higher Education, 2008.

\end{thebibliography}
\end{spacing}


 \appendix
\section{Proof of Theorem \ref{LboundTH}} \label{APPA}
To prove Theorem \ref{LboundTH}, we present two Theorems \ref{Learno1THapp} and \ref{DEERo1THapp}. First, in the following theorem, we investigate the convergence of the minimum empirical risk to the minimum risk in a model.
{
  \renewcommand{\thetheorem}{A.\arabic{theorem}}
  \setcounter{theorem}{0}

\begin{theorem}\label{Learno1THapp}
    For every $k\in\{1,2,...,L\}$, let $\mathcal{L}_k=\{l(.,\pmb{\theta}_k),\pmb{\theta}_k\in\mathcal{H}_k\}$, the set of loss functions corresponding to the $k$-th model in $\{\mathcal{M}_k\}_{k=1}^L$, be a Gilvenko-Cantelli functions class. Then, for every $k$, $|R_{emp}(S_n,\hat{\Ptheta}_k)-R(\Ptheta_k^*)|$ converges to zero in probability as $n\to\infty$, i.e., for every $\varepsilon>0$, we have
    \begin{flalign}
        \lim_{n\to \infty} \Pb\{ |R_{emp}(S_n,\hat{\Ptheta}_k)-R(\Ptheta_k^*)|>\varepsilon \}=0. 
    \end{flalign}
\end{theorem}

\begin{proof}
 Let $\textbf{Z}_1,\textbf{Z}_2,...,\textbf{Z}_n$ be $n$ i.i.d. random variables with the same distribution to $\textbf{Z}$. Since for every $k$, $\mathcal{L}_k$ is a Gilvenko-Cantelli class of functions, based on Definition \ref{GCdef}, for every $\varepsilon>0$, we have
     \begin{equation}
         \lim_{n\to\infty} \mathbb{P}\left\{ \sup_{\pmb{\theta}_k \in \mathcal{H}_k} \left| \frac{1}{n} \sum_{i=1}^{n} l(\textbf{Z}_i,\pmb{\theta}_k) -\mathbb{E}\{ l(\textbf{Z},\pmb{\theta}_k) \} \right|>\varepsilon  \right\}=0 .\label{eq15eq}
     \end{equation}
      Whereas \eqref{eq15eq} holds for supremum, it also holds for every $\pmb{\theta}_k \in \mathcal{H}_k$,  including $\hat{\pmb{\theta}}_k$ and $\pmb{\theta}^*_k$. Therefore, based on \eqref{riskEQ} and \eqref{empEQ}, for every $k$ and every $\varepsilon>0$, we can express \eqref{eq15eq} for any $\hat{\pmb{\theta}}_k$ and any $\pmb{\theta}^*_k$ as follows
    \begin{flalign}
         \lim_{n\to\infty} \mathbb{P}\left\{ \left| R_{emp}(S_n,\hat{\pmb{\theta}}_k) -R(\hat{\pmb{\theta}}_k)  \right|>\varepsilon  \right\}=0,\label{eq16eq}\\
         \lim_{n\to\infty}  \mathbb{P}\left\{ \left| R_{emp}(S_n,{\pmb{\theta}}_k^*) -R({\pmb{\theta}}_k^*)   \right|>\varepsilon  \right\}=0.\label{eq17eq}
      \end{flalign}
     Based on  \eqref{eq16eq} and \eqref{eq17eq}, for every $\varepsilon>0$  and $0 <\delta^{\prime}\leq 1/2$, there are  positive integers $N_1, N_2$, where for every $n \geq \max{(N_1,N_2)}$, we have 
     \begin{flalign}
         \mathbb{P}\left\{ \left| R_{emp}(S_n,\hat{\pmb{\theta}}_k) -R(\hat{\pmb{\theta}}_k)  \right|>\varepsilon \right\}<\delta^{\prime},\label{eq161eq}\\
            \mathbb{P}\left\{ \left| R_{emp}(S_n,{\pmb{\theta}}_k^*) -R({\pmb{\theta}}_k^*)   \right|>\varepsilon  \right\}<\delta^{\prime}.\label{eq171eq}
      \end{flalign}
      We can rewrite \eqref{eq161eq} and \eqref{eq171eq} as follows
       \begin{flalign}
         \mathbb{P}\left\{ \left| R_{emp}(S_n,\hat{\pmb{\theta}}_k) -R(\hat{\pmb{\theta}}_k)  \right|\leq \varepsilon \right\}\geq 1-\delta^{\prime},\label{eq162eq}\\
            \mathbb{P}\left\{ \left| R_{emp}(S_n,{\pmb{\theta}}_k^*) -R({\pmb{\theta}}_k^*)   \right|\leq \varepsilon \right\}\geq 1-\delta^{\prime}.\label{eq172eq}
      \end{flalign}
      Since the inequalities \eqref{eq162eq} and \eqref{eq172eq} hold for the absolute values of the left-hand side of these inequalities, the following also holds  
      \begin{flalign}
         \mathbb{P}\left\{  R(\hat{\pmb{\theta}}_k)-R_{emp}(S_n,\hat{\pmb{\theta}}_k)   \leq \varepsilon \right\}\geq 1-\delta^{\prime},\label{eq163eq}\\
            \mathbb{P}\left\{  R_{emp}(S_n,{\pmb{\theta}}_k^*) -R({\pmb{\theta}}_k^*)   \leq \varepsilon  \right\}\geq1-\delta^{\prime}.\label{eq173eq}
      \end{flalign}
      Based on \eqref{minriskparEQ} and \eqref{minempriskparEQ}, we use $R({\pmb{\theta}}_k^*)\leq R(\hat{\pmb{\theta}}_k) $ in \eqref{eq163eq} and  $R_{emp}(S_n,\hat{\pmb{\theta}}_k)\leq R_{emp}(S_n,{\pmb{\theta}}_k^*) $ in \eqref{eq173eq}. So, the following are obtained
      \begin{flalign}
         \mathbb{P}\left\{  R({\pmb{\theta}}_k^*)-R_{emp}(S_n,\hat{\pmb{\theta}}_k)   \leq \varepsilon \right\}\geq 1-\delta^{\prime},\label{eq164eq}\\
            \mathbb{P}\left\{  R_{emp}(S_n,\hat{\pmb{\theta}}_k)-R({\pmb{\theta}}_k^*)   \leq \varepsilon  \right\}\geq1-\delta^{\prime}.\label{eq174eq}
      \end{flalign}
      We can express \eqref{eq164eq} and \eqref{eq174eq} as follows
\begin{flalign}
         \mathbb{P}\left\{  R({\pmb{\theta}}_k^*)-R_{emp}(S_n,\hat{\pmb{\theta}}_k)  > \varepsilon \right\}<\delta^{\prime},\label{eq165eq}\\
            \mathbb{P}\left\{ R({\pmb{\theta}}_k^*)- R_{emp}(S_n,\hat{\pmb{\theta}}_k)   <- \varepsilon  \right\}<\delta^{\prime}.\label{eq175eq}
      \end{flalign}
      \sloppy Since $\mathbb{P}\{ R({\pmb{\theta}}_k^*)- R_{emp}(S_n,\hat{\pmb{\theta}}_k)   <- \varepsilon  \}+ \mathbb{P}\{ -\varepsilon \leq  R({\pmb{\theta}}_k^*)-R_{emp}(S_n,\hat{\pmb{\theta}}_k)  \leq \varepsilon \} + \mathbb{P}\{  R({\pmb{\theta}}_k^*)-R_{emp}(S_n,\hat{\pmb{\theta}}_k)  > \varepsilon \}   = 1$, we have
      \begin{flalign}
         \mathbb{P}\{& |  R({\pmb{\theta}}_k^*)-R_{emp}(S_n,\hat{\pmb{\theta}}_k) |  \leq \varepsilon \} =\mathbb{P}\left\{ -\varepsilon \leq  R({\pmb{\theta}}_k^*)-R_{emp}(S_n,\hat{\pmb{\theta}}_k)  \leq \varepsilon \right\} \nonumber \\ =&1-\mathbb{P}\left\{  R({\pmb{\theta}}_k^*)-R_{emp}(S_n,\hat{\pmb{\theta}}_k)  > \varepsilon \right\}-\mathbb{P}\left\{ R({\pmb{\theta}}_k^*)- R_{emp}(S_n,\hat{\pmb{\theta}}_k)     <- \varepsilon  \right\} \nonumber \\ \geq & 1-2\delta^{\prime} \label{eq1717eq}
      \end{flalign}
      where \eqref{eq1717eq} holds based on \eqref{eq165eq} and \eqref{eq175eq}. Now, let $\delta=2\delta^{\prime}$. Then, for every $\varepsilon>0$  and $0 <\delta \leq 1$, there are  positive integers $N_1, N_2$ where for every $n \geq \max{(N_1,N_2)}$, based on \eqref{eq1717eq}, the following holds
      \begin{flalign}
    &\mathbb{P}\left\{ | R_{emp}(S_n,\hat{\pmb{\theta}}_k)-R({\pmb{\theta}}_k^*)|  > \varepsilon \right\}< \delta\label{eq1617eq},
      \end{flalign}
      Therefore, based on  \eqref{eq1617eq}, we have
      \begin{flalign}
    \lim_{n\to\infty} \mathbb{P}\left\{ | R_{emp}(S_n,\hat{\pmb{\theta}}_k)-R({\pmb{\theta}}_k^*)|  > \varepsilon \right\} =0.
      \end{flalign}
\end{proof}

Using Theorem \ref{Learno1THapp}, in the following theorem, we show that the difference of the minimum empirical risks of the two models converges to the difference of minimum risks of these models in probability.

\begin{theorem}\label{DEERo1THapp}
    Consider two models, $\Mc_i$ and $\Mc_j$, where $\mathcal{L}_i=\{l(.,\Ptheta_i):\Ptheta_i \in \Hc_i\}$ and $\mathcal{L}_j=\{l(.,\Ptheta_j):\Ptheta_j \in \Hc_j\}$ are the set of loss functions corresponding to $\Mc_i$ and $\Mc_j$, respectively. Let $\mathcal{L}_i$ and $\mathcal{L}_j$ be the Glivenko-Cantelli classes of functions. Then, $|R_{emp}(S_n,\hat{\Ptheta}_i)-R_{emp}(S_n,\hat{\Ptheta}_j)-R(\Ptheta_i^*)+R(\Ptheta_j^*)|$ converges to zero in probability as $n\to\infty$, i.e., for every $\varepsilon>0$, we have
    \begin{flalign}
        \lim_{n\to \infty} \Pb\{ |R_{emp}(S_n,\hat{\Ptheta}_i)-R_{emp}(S_n,\hat{\Ptheta}_j)-R(\Ptheta_i^*)+R(\Ptheta_j^*)|>\varepsilon \}=0. 
    \end{flalign}
\end{theorem}
 \begin{proof}
     Using Theorem  \ref{Learno1THapp} for the models  $\Mc_i$ and  $\Mc_j$, for  every $\varepsilon^{\prime}>0$, the following holds
    \begin{flalign}
      \lim_{n\to\infty} \mathbb{P}\left\{ | R_{emp}(S_n,\hat{\pmb{\theta}}_i)-R({\pmb{\theta}}_i^*)|  > \varepsilon^{\prime} \right\} =0,\label{eq18eq}\\
      \lim_{n\to\infty} \mathbb{P}\left\{ | R_{emp}(S_n,\hat{\pmb{\theta}}_j)-R({\pmb{\theta}}_j^*)|  > \varepsilon^{\prime} \right\} =0.\label{eq19eq}
      \end{flalign}
      Based on \eqref{eq18eq} and \eqref{eq19eq}, for every $\varepsilon^{\prime}>0$ and $0<\delta\leq 1$, there are $N_1$ and $N_2$, where for every $n\geq \max{(N_1,N_2)}$, we have
      \begin{flalign}
       \mathbb{P}\left\{ | R_{emp}(S_n,\hat{\pmb{\theta}}_i)-R({\pmb{\theta}}_i^*)|  > \varepsilon^{\prime} \right\}<\delta,\label{eq181eq}\\
      \mathbb{P}\left\{ | R_{emp}(S_n,\hat{\pmb{\theta}}_j)-R({\pmb{\theta}}_j^*)|  > \varepsilon^{\prime} \right\} <\delta.\label{eq191eq}
      \end{flalign}
       By combining \eqref{eq181eq} and \eqref{eq191eq}, for every $\varepsilon^{\prime}>0$ and $0<\delta\leq 1$, there are $N_1$ and $N_2$, where for every $n\geq \max{(N_1,N_2)}$, we have
      \begin{flalign}
       \mathbb{P}\{ | R_{emp}(S_n,\hat{\pmb{\theta}}_i)-R({\pmb{\theta}}_i^*)|  + | R_{emp}(S_n,\hat{\pmb{\theta}}_j)-R({\pmb{\theta}}_j^*)|\ > 2\varepsilon^{\prime} \}< \delta.\label{eq1819eq}
      \end{flalign}
      We can rewrite \eqref{eq1819eq} as follows
       \begin{flalign}
       \mathbb{P}\{ | R_{emp}(S_n,\hat{\pmb{\theta}}_i)-R({\pmb{\theta}}_i^*)|  + | R_{emp}(S_n,\hat{\pmb{\theta}}_j)-R({\pmb{\theta}}_j^*)|  \leq  2\varepsilon^{\prime} \} \geq 1-\delta.\label{eq18191eq}
      \end{flalign}
      Using $| R_{emp}(S_n,\hat{\pmb{\theta}}_i)- R_{emp}(S_n,\hat{\pmb{\theta}}_j)-R({\pmb{\theta}}_i^*) +R({\pmb{\theta}}_j^*)| \leq | R_{emp}(S_n,\hat{\pmb{\theta}}_i)-R({\pmb{\theta}}_i^*)|  + | R_{emp}(S_n,\hat{\pmb{\theta}}_j)-R({\pmb{\theta}}_j^*)|$ in \eqref{eq18191eq}, we have
       \begin{flalign}
       \mathbb{P}\{ | R_{emp}(S_n,\hat{\pmb{\theta}}_i)- R_{emp}(S_n,\hat{\pmb{\theta}}_j)-R({\pmb{\theta}}_i^*)+R({\pmb{\theta}}_j^*)| \leq  2\varepsilon^{\prime} \} \geq 1-\delta.\label{eq18192eq}
      \end{flalign}
      Let $\varepsilon=2\varepsilon^{\prime}$, and based on \eqref{eq18192eq}, for every $\varepsilon>0$ and $0<\delta\leq 1$, there are $N_1$ and $N_2$, where for every $n\geq \max{(N_1,N_2)}$, the following hold
      \begin{flalign}
       \mathbb{P}\{ | R_{emp}(S_n,\hat{\pmb{\theta}}_i)- R_{emp}(S_n,\hat{\pmb{\theta}}_j)-R({\pmb{\theta}}_i^*)+R({\pmb{\theta}}_j^*)|  >  \varepsilon \} < \delta.\label{eq18193eq}
      \end{flalign}
      Therefore, based on \eqref{eq18193eq}, for every $\varepsilon>0$, we have
      \begin{flalign}
       \lim_{n\to \infty} \mathbb{P}\{ | R_{emp}(S_n,\hat{\pmb{\theta}}_i)- R_{emp}(S_n,\hat{\pmb{\theta}}_j)-R({\pmb{\theta}}_i^*)+ R({\pmb{\theta}}_j^*)|  >\varepsilon \} =0.\label{eq18194eq}
      \end{flalign}
 \end{proof}

Now, in the following theorem, we will show that a bound with $o(1)$ rate exists for $|R_{emp}(S_n,\hat{\Ptheta}_i)-R_{emp}(S_n,\hat{\Ptheta}_j)-R(\Ptheta_i^*)+R(\Ptheta_j^*)|$ in the Glivenko-Cantelli class of functions.

\begin{theorem}\label{TH21}
      Let $\mathcal{L}_i=\{l(.,\Ptheta_i):\Ptheta_i \in \Hc_i\}$ and $\mathcal{L}_j=\{l(.,\Ptheta_j):\Ptheta_j \in \Hc_j\}$ be the Glivenco-Cantelli classes of loss functions for the models $\Mc_i$ and $\Mc_j$, respectively. Then, there are a function $\gamma_{i,j}(S_n,\delta)$ with the order of $o(1)$ and a positive integer $N$, where for every  $n\geq N$ and every $0<\delta\leq 1$, the following holds
    \begin{flalign} 
        \mathbb{P}\{| R_{emp}(S_n,\hat{\pmb{\theta}}_i)- R_{emp}(S_n,\hat{\pmb{\theta}}_j)-R({\pmb{\theta}}_i^*)+R({\pmb{\theta}}_j^*)|  \leq \gamma_{i,j}(S_n,\delta)\}\geq 1-\delta.\label{O1boundeqAPP}
    \end{flalign}
\end{theorem}
\begin{proof}
Based on Theorem \ref{DEERo1THapp}, for every $\varepsilon>0$, we have
\begin{flalign}
\lim_{n\to \infty} \Pb\{ | R_{emp}(S_n,\hat{\pmb{\theta}}_i)- R_{emp}(S_n,\hat{\pmb{\theta}}_j)-R({\pmb{\theta}}_i^*)+R({\pmb{\theta}}_j^*)| >\varepsilon\}=0.\label{eq162}
\end{flalign}
Therefore, for every $\varepsilon>0$ and $0<\delta \leq  1 $, there is a positive integer $N^{\prime}$, where for every $n\geq N^{\prime}$, the following holds 
\begin{equation}\label{eq1621}
 \Pb\{ | R_{emp}(S_n,\hat{\pmb{\theta}}_i)- R_{emp}(S_n,\hat{\pmb{\theta}}_j)-R({\pmb{\theta}}_i^*)+R({\pmb{\theta}}_j^*)| >\varepsilon\}<\delta.
\end{equation}
We can express \eqref{eq1621} as follows
\begin{flalign}
 \Pb\{  | R_{emp}(S_n,\hat{\pmb{\theta}}_i)- R_{emp}(S_n,\hat{\pmb{\theta}}_j)-R({\pmb{\theta}}_i^*)+R({\pmb{\theta}}_j^*)|  \leq \varepsilon\} \geq 1-\delta.\label{eq1622}
\end{flalign}
 \eqref{eq1622} holds for every $\varepsilon>0$, then, also for every $\varepsilon>0$, $0< \gamma_{i,j}(S_n,\delta)<\varepsilon $, and  $0<\delta \leq 1 $, there is a positive integer $N$ where for every $n\geq N$, the following holds 
\begin{flalign}
 \Pb\{| R_{emp}(S_n,\hat{\pmb{\theta}}_i)- R_{emp}(S_n,\hat{\pmb{\theta}}_j)-R({\pmb{\theta}}_i^*)+R({\pmb{\theta}}_j^*)| \leq  \gamma_{i,j}(S_n,\delta)\} \geq 1-\delta.\label{eq1623}
\end{flalign}
  Since for every $\varepsilon>0$, we can find a positive integer $N$ such that for every $n\geq N$, $0< \gamma_{i,j}(S_n,\delta)<\varepsilon $, then, $\gamma_{i,j}(S_n,\delta)$ is in the order of $o(1)$. Therefore, there are always a function $\gamma_{i,j}(S_n,\delta)$  with the order of $o(1)$ and a positive integer $N$ such that for every $n\geq N$ and $0<\delta \leq 1$, \eqref{O1boundeqAPP} holds. 
      \end{proof}
       Finally, using the PAC bound in Theorem \ref{TH21}, we present a PAC bound for the difference of the minimum empirical risks of two nested models in the following theorem.
       }
        \addtocounter{theorem}{7}
       \begin{theorem}
	Consider two models,  $\mathcal{M}_i $ and $\mathcal{M}_j$,  where $\mathcal{L}_i=\{l(.,\Ptheta_i):\Ptheta_i\in \Hc_i\}$ and $\mathcal{L}_j=\{l(.,\Ptheta_j):\Ptheta_j\in \Hc_j\}$ are Glivenko-Cantelli classes of loss functions. Let $\mathcal{M}_i $ be nested in $\mathcal{M}_j$, so $R(\Ptheta^*_i)\geq R(\Ptheta^*_j)$. Also, assume  $\gamma_{i,j}\left(S_n,\delta\right)$ in Theorem \ref{TH21} is an $o(1)$ upper bound for $| R_{emp}(S_n,\hat{\pmb{\theta}}_i)- R_{emp}(S_n,\hat{\pmb{\theta}}_j)-R({\pmb{\theta}}_i^*)+R({\pmb{\theta}}_j^*)|$.
	\begin{itemize}
		\item{ \textbf{State 1}: For $R(\pmb{\theta}^*_i)>R(\pmb{\theta}^*_j)$ and every $0<\delta\leq 1$, there is a positive integer $N_1$, such that for every $n\geq N_1$, with the probability of at least $1-\delta$, the following holds\begin{equation}\label{minusKboundEQ}
				R_{emp}(S_n,\hat{\pmb{\theta}}_{i})-R_{emp}(S_n,\hat{\pmb{\theta}}_{j})  \geq \gamma_{i,j}\left(S_n,\delta\right)  ,
		\end{equation}}
		\item{\textbf{State 2}: For $R(\pmb{\theta}^*_i)=R(\pmb{\theta}^*_j)$, there is a positive integer $N_2$, such that for every $n\geq N_2$,  $0<\delta\leq 1$,  with the probability of at least $1-\delta$, the following holds
			\begin{equation}\label{empexcessERMINEQ}
				R_{emp}(S_n,\hat{\pmb{\theta}}_{i})-R_{emp}(S_n,\hat{\pmb{\theta}}_{j})  \leq \gamma_{i,j}\left(S_n,\delta\right).
		\end{equation}}
	\end{itemize}
\end{theorem}
\begin{proof}
    In the first step, we prove the state 1 where $R(\Ptheta^*_i)>R(\Ptheta^*_j)$. Based on Theorem \ref{TH21},  there is a positive integer $N$ where for every $n\geq N$ and $0<\delta\leq 1$, we have
      \begin{flalign}
        \mathbb{P}\Big\{| R_{emp}(S_n,\hat{\pmb{\theta}}_j)- R_{emp}(S_n,\hat{\pmb{\theta}}_i)+R({\pmb{\theta}}_i^*)-R({\pmb{\theta}}_j^*)|  \leq \gamma_{i,j}(S_n,\delta)\Big\}\geq 1-\delta.\label{eq33eq1}
    \end{flalign}
    Since the above inequality holds for the absolute value of the left-hand side of the inequality, the following also holds
    \begin{flalign}
        \mathbb{P}\Big\{ R_{emp}(S_n,\hat{\pmb{\theta}}_i)- R_{emp}(S_n,\hat{\pmb{\theta}}_j)\geq R({\pmb{\theta}}_i^*)-R({\pmb{\theta}}_j^*)  -\gamma_{i,j}(S_n,\delta)\Big\}\geq 1-\delta.\label{eq33eq2}
    \end{flalign}
    Note that  
$R({\pmb{\theta}}_i^*),\ R({\pmb{\theta}}_j^*)$, and $\gamma_{i,j}(S_n,\delta) $ are deterministic functions. Also, for state 1, $\frac{1}{2}\big(R({\pmb{\theta}}_i^*)-R({\pmb{\theta}}_j^*)\big)>0$. Moreover, since $\gamma_{i,j}(S_n,\delta)=o(1)$, for every $0<\delta\leq 1$, there is a positive integer $N_1^{\prime}$ where for every $n\geq N_1^{\prime}$, the following holds
\begin{equation}\label{eq33eq4}
    (1/2)\big(R({\pmb{\theta}}_i^*)-R({\pmb{\theta}}_j^*)\big)>\gamma_{i,j}(S_n,\delta)
\end{equation}
We can express \eqref{eq33eq4} as follows
\begin{equation}\label{eq33eq5}
    R({\pmb{\theta}}_i^*)-R({\pmb{\theta}}_j^*)-\gamma_{i,j}(S_n,\delta)>\gamma_{i,j}(S_n,\delta)
\end{equation}
Now, using \eqref{eq33eq5} in \eqref{eq33eq2}, for every $0<\delta\leq 1$, there is a positive integer $N_1=\max{(N,N_1^{\prime})}$ where  for every $n\geq N_1$, we have
\begin{equation}\label{eq33eq6}
        \mathbb{P}\Big\{ R_{emp}(S_n,\hat{\pmb{\theta}}_i)- R_{emp}(S_n,\hat{\pmb{\theta}}_j)  \geq   \gamma_{i,j}(S_n,\delta)\Big\}\geq1-\delta,
    \end{equation}
where it proves the state 1.

    Now, we prove the state 2 where $R(\Ptheta^*_i)=R(\Ptheta^*_j)$. Based on Theorem \ref{TH21}, there is a positive integer $N_2$, where for every $n\geq N_2$ and $0<\delta\leq1$, the following holds
    \begin{flalign}
        \mathbb{P}\Big\{ |R_{emp}(S_n,\hat{\pmb{\theta}}_i)- R_{emp}(S_n,\hat{\pmb{\theta}}_j)-R({\pmb{\theta}}_i^*)+R({\pmb{\theta}}_j^*)|\leq \gamma_{i,j}(S_n,\delta)\Big\}\geq1-\delta.\label{eq330eq7}
        \end{flalign}
        Since \eqref{eq330eq7} holds for the absolute value of the left-hand side of the inequality, the following also holds  
    \begin{flalign}
        \mathbb{P}\Big\{ R_{emp}(S_n,\hat{\pmb{\theta}}_i)- R_{emp}(S_n,\hat{\pmb{\theta}}_j)-R({\pmb{\theta}}_i^*)+R({\pmb{\theta}}_j^*)\leq \gamma_{i,j}(S_n,\delta)\Big\}\geq1-\delta.\label{eq33eq7}
        \end{flalign}
        In state 2, $R({\pmb{\theta}}_j^*)-R({\pmb{\theta}}_i^*)=0$, and we have
         \begin{equation}\label{eq33eq8}
        \mathbb{P}\Big\{ R_{emp}(S_n,\hat{\pmb{\theta}}_i)- R_{emp}(S_n,\hat{\pmb{\theta}}_j)\leq \gamma_{i,j}(S_n,\delta)\Big\}\geq1-\delta,
        \end{equation}
        where it proves the state 2.
\end{proof}
\addtocounter{theorem}{1}

\section{Consistency}\label{APPB} 
\begin{theorem}\label{PCMST} 
Let $\{\mathcal{M}_k\}_{k=1}^L$ be a sequentially nested model family, $\mathcal{L}_k=\{l(.,\Ptheta_k):\Ptheta_k\in \Hc_k\}, k\in\{1,2,...,L\}$ be corresponding Glivenko-Cantelli classes of loss functions, and $\hat{K}$ be calculated based on \eqref{KHATEQ}. Therefore, for every $0<\delta<1/(L-K+1)$, there is a positive integer $N$, where for $n\geq N$, we have 
	\begin{equation}
 \mathbb{P}\{\hat{K}=K\}\geq 1-(L-K+1)\delta.
	\end{equation}
\end{theorem}
\begin{proof}
	 Based on \eqref{KHATEQ}, $\hat{K}$ is an estimate of $K$ when $T_{\hat{K}}=1$ and $T_{\hat{K}+1}=...=T_L=0$. Therefore, we have 
	\begin{flalign}
		\mathbb{P}\{\hat{K}=K\}&=\mathbb{P}\{T_K=1\land T_{K+1}=0 \land ...\land T_{L}=0\}\nonumber \\&= 1- \mathbb{P}\{T_K= 0 \lor T_{K+1}=1 \lor ... \lor T_{L}=1\}\nonumber\\&\geq 1- \mathbb{P}\{T_K= 0 \} -\sum_{k=K+1}^L \mathbb{P}\{T_{k}=1 \}. \label{PKHATKEQ1}
	\end{flalign}
	Based on state 1 of Corollary \ref{NestedboundTH}, for every  $0<\delta<1/(L-K+1)$, there is a positive integer $N\geq N_K$ where, for every $n\geq N$, $ \mathbb{P}\{T_K= 0 \} $ is calculated as follows
	\begin{flalign}
		\mathbb{P}\{T_K= 0 \} &=\mathbb{P}\{{\Delta R}_{emp}(S_n,K)\leq{\gamma_{K}\left(S_n,\delta\right)}\} \leq \delta. \label{TL1NEEQL3}
	\end{flalign}
	Also, based on state 2 of Corollary \ref{NestedboundTH}, for every $k\in\{K+1,K+2,...,L\}$, there is a positive integer $N_k$ for every $n\geq N_k$, and $0<\delta<1/(L-K+1)$, we have the following
	\begin{flalign}
		\mathbb{P}\{T_k= 1 \} &=\mathbb{P}\{{\Delta R}_{emp}(S_n,k)\geq{\gamma_{k}\left(S_n,\delta\right)}\} \leq {\delta}. \label{TLONEEQL2} 
	\end{flalign}
	Using \eqref{TL1NEEQL3} and \eqref{TLONEEQL2} in \eqref{PKHATKEQ1}, for every $0<\delta<1/(L-K+1)$, there is a positive integer $N= \max{(N_{K},N_{K+1},...,N_{L})}$, where for every $n\geq N$, we have
	\begin{flalign}
		\mathbb{P}\{\hat{K}=K\}&\geq1-\delta-\sum_{k=K+1}^L \delta = 1- (L-K+1)\delta. \label{PKHATKEQ3}
	\end{flalign}
\end{proof}
 \addtocounter{theorem}{1}
\begin{theorem}\label{PCMSTH}
Let the set of loss functions of the candidate models in every iteration of  Algorithm \ref{NERALG} be the Glivenko-Cantelli class of functions, and  $\hat{K}$ and $\mathcal{M}_{\hat{K}}$ are calculated based on Algorithm \ref{NERALG}. Therefore, for every  $0<\delta\leq 1/L$, there is a positive integer $N$, where for $n\geq N$, $\mathcal{M}_{\hat{K}}$ is the most parsimonious model containing a risk minimizer predictor with the probability of at least $1-L\delta$, i.e., 
\begin{equation} \label{PCDLBEQ11}
\Pb\mathbb{C}_{S-NER} \geq1-L\delta.
\end{equation}
\end{theorem}
\begin{proof}
	 In the S-NER algorithm, to achieve the most parsimonious model,  the minimum risk must be decreased through every step of model expansion. This algorithm has $\min{(\hat{K}+1, L)}$ iterations where, in every iteration, the test verifies the model expansion. As soon as the test rejects an expansion in the $(\hat{K}+1)$-th iteration, it automatically rejects all the remaining parameters. Indeed,  the S-NER algorithm makes $\hat{K}+1$ decisions $T_1, T_2,..., T_{\hat{K}+1}$ and automatically sets $T_{\hat{K}+1}=T_{\hat{K}+2}=...= T_{L}=0$. In this regard, for every $k\in\{1,2,...,L\}$, let $T_k^*\in \{0,1\}$ be the true value of $k$-th decision. Let $T_k^*= 1$ if and only if  $R(\pmb{\theta}^*_k)<R(\pmb{\theta}^*_{k-1})$, and $T_k^*= 0$ if and only if $R(\pmb{\theta}^*_k)=R(\pmb{\theta}^*_{k-1})$. Therefore, the probability of true model selection, $\mathbb{PC}_{S-NER}$, will be calculated as follows 
\begin{flalign}
    \mathbb{PC}_{S-NER}&=\mathbb{P}\{ T_1=T_1^*\land T_2=T_2^* \land...\land T_{L}=T_{L}^*  \} \nonumber \\ &= 1-\mathbb{P}\{ T_1\neq T_1^*\lor T_2 \neq T_2^* \lor ... \lor T_{L}\neq T_{L}^*  \} \geq  1-\sum_{k=1}^{L} \mathbb{P}\{ T_k\neq T_k^* \}. \label{TEMP52eq11}
\end{flalign}
Now, we calculate the probability of fault decision in the $k$-th iteration, $\mathbb{P}\{ T_k\neq T_k^* \}$. In this regard,  based on Theorem \ref{LboundTH}, for every $k$ and $0<\delta\leq 1/L$, there are positive integers $N_{k}$ and $N_k^{\prime}$ such that for every $n\geq \max{(N_{k},N_{k}^{\prime})}$, the following hold
\begin{flalign}
    &\mathbb{P}\{ T_k\neq T_k^*|T_k^*=1 \}=\mathbb{P}\{ T_k=0|R(\pmb{\theta}^*_k)<R(\pmb{\theta}^*_{k-1}) \} \nonumber \\ 
    &= \mathbb{P}\{ \Delta R_{emp}(S_n,k)\leq \gamma_k(S_n,\delta)|R(\pmb{\theta}^*_k)<R(\pmb{\theta}^*_{k-1}) \} \leq \delta , \label{TEMP53eq}\\&\mathbb{P}\{ T_k\neq T_k^*|T_k^*=0 \}=\mathbb{P}\{ T_k=1|R(\pmb{\theta}^*_k)=R(\pmb{\theta}^*_{k-1}) \} \nonumber \\ 
    &= \mathbb{P}\{ \Delta R_{emp}(S_n,k)\geq \gamma_k(S_n,\delta)|R(\pmb{\theta}^*_k)=R(\pmb{\theta}^*_{k-1}) \} \leq \delta. \label{TEMP54eq}
\end{flalign}
where \eqref{TEMP53eq} and \eqref{TEMP54eq} hold based on state 1 and state 2 of Theorem \ref{LboundTH}, respectively. Using \eqref{TEMP53eq} and \eqref{TEMP54eq}, for every $k$, $0<\delta\leq 1/L$, and $n\geq \max{(N_{k},N_{k}^{\prime})}$, the following hold
\begin{flalign}
    \mathbb{P}\{ T_k\neq T_k^* \}&=\mathbb{P}\{ T_k\neq T_k^*|T_k^*=0 \}p_k +\mathbb{P}\{ T_k\neq T_k^*|T_k^*=1 \}(1-p_k)\nonumber \\ &\leq \delta (1-p_k)+\delta p_k = \delta, \label{ErrEQ}
\end{flalign}
where $p_k\equiv  \mathbb{P}\{T_k^*=0\}= \mathbb{P}\{ R(\pmb{\theta}^*_k)=R(\pmb{\theta}^*_{k-1}) \}$ and  $\mathbb{P}\{T_k^*=1\}=\mathbb{P}\{ R(\pmb{\theta}^*_k)<R(\pmb{\theta}^*_{k-1}) \}=1-p_k$. Note that, since  $\mathcal{M}_{k-1}$ is nested in $\mathcal{M}_{k}$,   based on Lemma \ref{lem1}, $R(\pmb{\theta}^*_k)\leq  R(\pmb{\theta}^*_{k-1}) $.  Using  \eqref{ErrEQ} in \eqref{TEMP52eq11},  for every  $0<\delta\leq 1/L$, a positive integer $N=\max{(N_{1},N_{1}^{\prime},N_{2},N_{2}^{\prime},...,N_{L},N_{L}^{\prime})}$ exists such that for every $n\geq N$,  we have
\begin{flalign}
     \mathbb{PC}_{S-NER} \geq  1- L \delta.\label{PCDLBEQ}
\end{flalign}
\end{proof}
\section{Proof of Theorems \ref{LinearLowerboundTH} and \ref{consistency} } \label{APPC}
In this appendix, we use the least square loss function, $ l({\textbf{x}^{\Ic_k}},y,\pmb{\theta}_k)=(y-({\textbf{x}^{\Ic_k}})^T \pmb{\theta}_k )^2$. Based on \eqref{empEQ}, the empirical risk is expressed as $R_{emp}(S_n,\pmb{\theta}_k)=\frac{\| \textbf{y}-({\textbf{X}^{\Ic_k}})^T {\pmb{\theta}}_k\|_2^2}{n}$, where ${\textbf{X}^{\Ic_k}}$ denotes the sub-matrix of $\textbf{X}$ formed using the rows indexed by $\Ic_k$ and $\|.\|$ is the Euclidean norm. Based on the Moore-Penrose solution, ${\hat{\pmb{\theta}}}_k=({\textbf{X}^{\Ic_k}})^{\dagger} \textbf{y}$ is the $k$-th model's empirical risk minimizer parameter, where $({\textbf{X}^{\Ic_k}})^{\dagger}=\big( {\textbf{X}^{\Ic_k}}({\textbf{X}^{\Ic_k}})^T\big)^{-1}{\textbf{X}^{\Ic_k}}$.  Hence, for every $k\in\{1,2,...,\bar{L}\}$, the minimum empirical risk of the $k$-th model, $R_{emp}(S_n,\hat{\pmb{\theta}}_k)$, will be calculated as follows
\begin{equation}\label{linRempEQ}
	R_{emp}(S_n,\hat{\pmb{\theta}}_k)=\frac{\| \textbf{y}-({\textbf{X}^{\Ic_k}})^T{\hat{\pmb{\theta}}}_k \|_2^2}{n}=\frac{\| (\textbf{I}_n-\textbf{P}_k)\textbf{y} \|_2^2}{n},
\end{equation}
where $\textbf{P}_k=({\textbf{X}^{\Ic_{k}}})^T({\textbf{X}^{\Ic_{k}}})^{\dagger}$. $\textbf{P}_k$ is the orthogonal projection matrix onto $span({\textbf{X}^{\Ic_k}})$, where $span({\textbf{X}^{\Ic_k}})$ is the row space of ${\textbf{X}^{\Ic_k}}$. Indeed, the row space of ${\textbf{X}^{\Ic_k}}$ is the set of all possible linear combinations of rows of ${\textbf{X}^{\Ic_k}}$. Also,  $\textbf{I}_n-\textbf{P}_k$ is the orthogonal projection onto  $span({\textbf{X}^{\Ic_k}})^{\perp}$, where $span({\textbf{X}^{\Ic_k}})^{\perp}$ is the orthogonal complement row space of ${\textbf{X}^{\Ic_k}}$.  Moreover,  using the least-squares function, the empirical risk of $\Mc_0$ will be $R_{emp}(S_n,\hat{\Ptheta}_0)=R_{emp}(S_n,{\Ptheta}_0)={\|\textbf{y}\|_2^2}/{n}$. In this model, the parameter with the minimum empirical risk is $\hat{\Ptheta}_0=\Ptheta_0=\textbf{0}_{L\times 1}$. Therefore, we let $\textbf{P}_0 = \textbf{0}_{n\times n}$, where $\textbf{0}_{n\times n}$ is an all-zero $n\times n$  matrix. Based on Appendix B of \cite{kallummil2018signal}, for every $k$, $(\textbf{I}_n-\textbf{P}_k)(\textbf{P}_k-\textbf{P}_{k-1})=\textbf{0}_{n\times n}$, since $\textbf{I}_n-\textbf{P}_k$ is the orthogonal projection matrix onto $span(\textbf{X}^{\Ic_k})^{\perp}$ and $\textbf{P}_k-\textbf{P}_{k-1}$ is the orthogonal projection matrix onto $span(\textbf{X}^{\Ic_k}) \cap span(\textbf{X}^{\Ic_{k-1}})^{\perp}$. As a consequence, for every $k$, $R_{emp}(S_n,\hat{\pmb{\theta}}_{k-1})$ is calculated as follows
 \begin{flalign} 
      R_{emp}(S_n,\hat{\pmb{\theta}}_{k-1})= \frac{\| (\textbf{I}_n-\textbf{P}_{k-1})\textbf{y} \|_2^2}{n}  &=\frac{\|(\textbf{I}_n - \textbf{P}_{k})\textbf{y} + (\textbf{P}_{k} - \textbf{P}_{k-1})\textbf{y}\|_2^2}{n}   \nonumber \\   &= \frac{\| (\textbf{I}_n-\textbf{P}_{k})\textbf{y} \|_2^2}{n}+\frac{\| (\textbf{P}_{k}-\textbf{P}_{k-1})\textbf{y} \|_2^2}{n}.
 \end{flalign}
Then, the SEER will be expressed as follows
	\begin{flalign}
		R_{emp}(S_n,\hat{\pmb{\theta}}_{k-1})-R_{emp}(S_n,\hat{\pmb{\theta}}_k)&=\frac{\| (\textbf{I}_n-\textbf{P}_{k})\textbf{y} \|_2^2}{n}+\frac{\| (\textbf{P}_{k}-\textbf{P}_{k-1})\textbf{y} \|_2^2}{n} -\frac{\| (\textbf{I}_n-\textbf{P}_{k})\textbf{y} \|_2^2}{n}\nonumber\\&=\frac{\| (\textbf{P}_{k}-\textbf{P}_{k-1})\textbf{y} \|_2^2}{n}. \label{linREMPEQ}
	\end{flalign}
Now, for every $k$, defining $U_k \equiv \frac{\| (\textbf{P}_{k}-\textbf{P}_{k-1})\textbf{y} \|_2^2}{\sigma^2}$, SEER can be rewritten as $R_{emp}(S_n,\hat{\pmb{\theta}}_{k-1})-R_{emp}(S_n,\hat{\pmb{\theta}}_k)=\frac{\sigma^2}{n} U_k$. According to Appendix A in \cite{gohain2023robust},  $U_k$ has non-central chi-2 distribution $\chi^2_1(\zeta_{U_k})$ where $\zeta_{U_k}=\frac{\| (\textbf{P}_{k}-\textbf{P}_{k-1})\mathbb{\textbf{X}^T \pmb{\theta}^*} \|_2^2}{\sigma^2}$. If $\zeta_{U_k}=0$, $U_k$ has a central chi-2 distribution, $\chi^2_1$. Based on these results, the PAC bound of SEER in linear regression is given by the following theorem. 
 \addtocounter{theorem}{1}
\begin{theorem}
	Let $F_{U_k}(.,\zeta_{{U_k}})$ be the cumulative distribution function (CDF) of a non-central $\chi^2_1(\zeta_{{U_k}})$ random variable with non-centrality parameter $ \zeta_{U_k}=\frac{\| (\textbf{P}_{k}-\textbf{P}_{k-1})\mathbb{\textbf{X}^T \pmb{\theta}^*} \|_2^2}{\sigma^2}$, $F_{1}(.)$ be the  CDF function of a central $\chi^2_1$ random variable, and $F_{1}^{-1}(.)$ be the inverse CDF function of a central $\chi^2_1$ random variable. Then,  for every $k\in\{1,2,...,\bar{L}\}$, 
 \begin{itemize}
 \item{if $ \mathcal{S} \subseteq \Ic_{k-1}$, for every $0<\delta\leq 1$,  $n\geq1$, and constant $c>0$, we have
	\begin{flalign}
		\mathbb{P}\{R_{emp}(S_n,\hat{\pmb{\theta}}_{k-1})-R_{emp}(S_n,\hat{\pmb{\theta}}_k)\leq c\frac{\sigma^2}{n} F_1^{-1}(1-\delta) \}=F_1(cF_1^{-1}(1-\delta)),  \label{LinRegUpboundEQAPP}
	\end{flalign}}
 \item {and if $\Ic_{k-1} \subset \mathcal{S} $, for every $0<\delta\leq 1$, $n\geq 1$, and constant $c>0$, the following holds
	\begin{equation}
		\mathbb{P}\{R_{emp}(S_n,\hat{\pmb{\theta}}_{k-1})-R_{emp}(S_n,\hat{\pmb{\theta}}_k)\geq c\frac{\sigma^2}{n} F_1^{-1}(1-\delta) \}=  1-F_{U_k}(cF_1^{-1}(1-\delta),\zeta_{U_k}).\label{LinRegLowboundEQAPP}
	\end{equation}}
 \end{itemize}
\end{theorem}
\begin{proof}
We prove the first statement of the theorem.   $\textbf{X}^T{\pmb{\theta}^*}\in span(\textbf{X}^{\Sc})$ and $\textbf{I}_n-\textbf{P}_{k-1}$ is the orthogonal projection on $span(\textbf{X}^{\Ic_{k-1}})^{\perp}$. Since $  \mathcal{S} \subseteq \Ic_{k-1}$,   $(\textbf{I}_n-\textbf{P}_{k-1})\textbf{X}^T{\pmb{\theta}^*}=\textbf{0}_{n\times 1}$ \cite{kallummil2018signal}. In addition, since  $\Ic_{k-1}\subseteq\Ic_{k}$ and $ \mathcal{S} \subseteq \Ic_{k-1}$, we have $  \mathcal{S} \subseteq \Ic_{k}$ and  $(\textbf{I}_n-\textbf{P}_{k})\textbf{X}^T{\pmb{\theta}^*}=\textbf{0}_{n\times 1}$ \cite{kallummil2018signal}. Based on these facts, for every $k\in\{1,2,...,\bar{L}\}$ that $ \mathcal{S} \subseteq \Ic_{k-1}$,  we have
    \begin{flalign}
        (\textbf{P}_{k}-\textbf{P}_{k-1})\textbf{X}^T{\pmb{\theta}^*}=(\textbf{I}_n-\textbf{P}_{k-1})\textbf{X}^T{\pmb{\theta}^*} -(\textbf{I}_n-\textbf{P}_{k})\textbf{X}^T{\pmb{\theta}^*} =0.\label{tempappb33eq}
    \end{flalign}
Therefore, using $\textbf{y}=\textbf{X}^T\Ptheta^*+\pmb{\varepsilon}$,  the SEER will be obtained as follows
	\begin{flalign}
		R_{emp}(S_n,\hat{\pmb{\theta}}_{k-1})-R_{emp}(S_n,\hat{\pmb{\theta}}_k)=\frac{ \|(\textbf{P}_{k}-\textbf{P}_{k-1}) \textbf{y} \|_2^2 }{n}\nonumber  &=\frac{ \|(\textbf{P}_{k}-\textbf{P}_{k-1}) \textbf{X}^T\Ptheta^* +(\textbf{P}_{k}-\textbf{P}_{k-1}) {\pmb{\varepsilon}} \|_2^2 }{n}\nonumber \\ &=\frac{\|(\textbf{P}_{k}-\textbf{P}_{k-1}) {\pmb{\varepsilon}} \|_2^2 }{n},\label{linempexcessEQ}
	\end{flalign}
where \eqref{linempexcessEQ} holds based on  \eqref{tempappb33eq}. 	In Appendix B of \cite{kallummil2018signal}, it is shown that for zero mean normal random vector $\pmb{\varepsilon}$ with covariance matrix $\textbf{I}_n$, we have $\frac{ \|(\textbf{P}_{k}-\textbf{P}_{k-1}) {\pmb{\varepsilon}} \|_2^2 }{\sigma^2} \sim \chi^2_1$. Based on \eqref{linempexcessEQ}, $R_{emp}(S_n,\hat{\pmb{\theta}}_{k-1})-R_{emp}(S_n,\hat{\pmb{\theta}}_k)=\frac{\sigma^2}{n} \frac{ \|(\textbf{P}_{k}-\textbf{P}_{k-1}) {\pmb{\varepsilon}} \|_2^2 }{\sigma^2}$. Then, using the CDF of the $\chi^2_1$ distribution, for every $n\geq 1$ and $u\geq0$, we have
	\begin{flalign}
		\mathbb{P}\{R_{emp}(S_n,\hat{\pmb{\theta}}_{k-1})-R_{emp}(S_n,\hat{\pmb{\theta}}_k)\leq \frac{\sigma^2}{n} u \} &=\mathbb{P}\Big\{\frac{\sigma^2}{n}\frac{ \|(\textbf{P}_{k}-\textbf{P}_{k-1}) {\pmb{\varepsilon}} \|_2^2 }{\sigma^2}\leq \frac{\sigma^2}{n}u \Big\}\nonumber \\ &=\mathbb{P}\Big\{\frac{ \|(\textbf{P}_{k}-\textbf{P}_{k-1}) {\pmb{\varepsilon}} \|_2^2 }{\sigma^2}\leq u \Big\} =  F_1(u), \label{linempexcessboundEQ}
	\end{flalign}
	where $F_1(.)$ is the CDF of a central Chi-square distribution with one degree of freedom. For every $0<\delta\leq1$, we let  $F_1(u/c)=1-\delta$. Then, using $u=cF_1^{-1}(1-\delta)$, for every $0<\delta\leq1$ and $n\geq 1$, we can rewrite \eqref{linempexcessboundEQ} as follows
	\begin{flalign}
		\mathbb{P}\{R_{emp}(S_n,\hat{\pmb{\theta}}_{k-1})-R_{emp}(S_n,\hat{\pmb{\theta}}_k)\leq c\frac{\sigma^2}{n} F_1^{-1}(1-\delta) \}= F_1(cF_1^{-1}(1-\delta)),\label{linempexcessboundEQ1}
	\end{flalign}
	where $F_1^{-1}$ is the inverse of the cumulative distribution function of a central  $\chi_1^2$ random variable. This proves the first statement.

Now, we prove the second statement of the theorem. Using $U_k=\frac{\|(\textbf{P}_k-\textbf{P}_{k-1})\textbf{y}\|_2^2}{\sigma^2}$ and based on \eqref{linREMPEQ}, we have   $R_{emp}(S_n,\hat{\pmb{\theta}}_{k-1})-R_{emp}(S_n,\hat{\pmb{\theta}}_k)=\frac{\sigma^2}{n}U_k$. Then, for every $u\geq0$ and $n\geq 1$, the following holds
	\begin{flalign}
		\mathbb{P}\{R_{emp}(S_n,\hat{\pmb{\theta}}_{k-1})-R_{emp}(S_n,\hat{\pmb{\theta}}_k)\geq \frac{\sigma^2}{n}u\} =1-\mathbb{P}&\{R_{emp}(S_n,\hat{\pmb{\theta}}_{k-1})-R_{emp}(S_n,\hat{\pmb{\theta}}_k)\leq  \frac{\sigma^2}{n} u\}\nonumber\\=1-\mathbb{P}\{\frac{\sigma^2}{n}U_k\leq  \frac{\sigma^2}{n}u\}=1-\mathbb{P}&\{U_k\leq u\}=1-F_{U_k}(u,\zeta_{U_k}),\label{upperboundulinEQ}
	\end{flalign}
 where \eqref{upperboundulinEQ} holds since $U_k$ has a non-central Chi-square distribution with one degree of freedom and non-centrality parameter $\zeta_{U_k}=\frac{\|(\textbf{P}_k-\textbf{P}_{k-1})\textbf{X}^T\Ptheta^*\|}{\sigma^2}$ and $F_{U_k}(.,\space\zeta_{U_k})$ is its CDF function. Now, for every $0<\delta\leq 1$, let $u=c F_1^{-1}(1-\delta)$. Then, for every $0<\delta\leq 1$ and $n\geq 1$, we can rewrite \eqref{upperboundulinEQ}, as follows
	\begin{flalign}
		\mathbb{P}\{R_{emp}(S_n,\hat{\pmb{\theta}}_{k-1})-R_{emp}(S_n,\hat{\pmb{\theta}}_k)\geq c\frac{\sigma^2}{n}F_1^{-1}(1-\delta))=1-F_{U_k}(cF_1^{-1}(1-\delta),\zeta_{U_k}\}.\label{upperbounddeltalinEQ}
	\end{flalign}
\end{proof}

\begin{theorem} 
Let  $F_{1}^{-1}(.)$ be the inverse CDF function of a central $\chi^2_1$ random variable and assume that sparse Riesz condition \eqref{RieszASS} holds. Also, for a constant  $c_1>0$, let $\delta(n)=c_1\Big(\frac{\sigma^2}{n}\Big)$. Then, for every constants $c,c_1>0$ and every $k\in\{1,2,...,L\}$,  
 \begin{enumerate}
 \item if  $ \mathcal{S} \subseteq \Ic_{k-1}$,  we have 
	\begin{flalign}
 \lim_{n\to \infty} \mathbb{P}\big\{R_{emp}(S_n,\hat{\pmb{\theta}}_{k-1})-R_{emp}(S_n,\hat{\pmb{\theta}}_k)\ \leq c\frac{\sigma^2}{n}F_1^{-1}\big(1-c_1\big(\frac{\sigma^2}{n}\big)\big)\}=1,
 \end{flalign}
 \item and if  $\Ic_{k-1} \subset \mathcal{S} $, we have
 \begin{flalign}
		\lim_{n\to \infty} \mathbb{P}\big\{R_{emp}(S_n,\hat{\pmb{\theta}}_{k-1})-R_{emp}(S_n,\hat{\pmb{\theta}}_k)\geq c\frac{\sigma^2}{n}F_1^{-1}\big(1-c_1\big(\frac{\sigma^2}{n}\big)\big)\big\}=1.
	\end{flalign}
  \end{enumerate}
\end{theorem}
 
 \begin{proof}
 We prove the first statement of the theorem. Based on Theorem \ref{LinearLowerboundTH}, if $\Sc\subseteq \Ic_{k-1} $, for $\delta(n)=c_1(\frac{\sigma^2}{n}))$, we have
 \begin{flalign}
  \mathbb{P}\big\{R_{emp}(S_n,\hat{\pmb{\theta}}_{k-1})-R_{emp}(S_n,\hat{\pmb{\theta}}_k)\leq c \frac{\sigma^2}{n}F_1^{-1}\big(1-c_1\big(\frac{\sigma^2}{n}\big)\big)\big\} = F_1\big(cF_1^{-1}\big(1-c_1\big(\frac{\sigma^2}{n}\big)\big)\big).
 \end{flalign}
 Now, we will show that $lim_{n \to \infty} F_1(cF_1^{-1}(1-c_1(\frac{\sigma^2}{n})))=1$.
\begin{equation}
   \lim\limits_{n \to \infty} F_1\big(cF_1^{-1}\big(1-c_1\big(\frac{\sigma^2}{n}\big)\big)\big)=F_1\big(cF_1^{-1}\big(1-c_1  \lim_{n \to \infty}\big(\frac{\sigma^2}{n}\big)\big)\big)=F_1\big(cF_1^{-1}(1)\big)=1,\label{63EQ1}
\end{equation}
where \eqref{63EQ1} holds since $F_1$ and $F_1^{-1}$ are continuous functions. Then, the first statement of this theorem is proved.

Now, we prove the second statement of the theorem. Based on Theorem \ref{LinearLowerboundTH}, if $\Ic_{k-1}\subset \Sc$, for $\delta(n)=c_1(\frac{\sigma^2}{n})$ we have
 \begin{flalign}
  \mathbb{P}\big\{R_{emp}(S_n,\hat{\pmb{\theta}}_{k-1})-R_{emp}(S_n,\hat{\pmb{\theta}}_k)\geq \frac{\sigma^2}{n}cF_1^{-1}\big(1-c_1\big(\frac{\sigma^2}{n}\big)\big)\big\}\nonumber \\ = 1-F_{U_k}\big(cF_1^{-1}\big(1-c_1\big(\frac{\sigma^2}{n}\big)\big),\zeta_{U_k}\big).
 \end{flalign}

Now, we will show that $lim_{n \to \infty} F_{U_k}(cF_1^{-1}(1-c_1(\frac{\sigma^2}{n})),\zeta_{U_k})=0$. $F_1(.)$ and $F_{U_k}(.,\zeta_{U_k})$ are the CDF of  central $\chi_1^2$  and non-central $\chi_1^2(\zeta_{U_k})$ random variables, respectively. Therefore, for every zero-mean Gaussian random variable $A$ with unit variance and Gaussian random variable $A_k$ with mean $\sqrt{\zeta_{U_k}}$ and unit variance, $A^2$ and  $A_k^2$ have CDF of $F_1(.)$ and $F_{U_k}(.,\zeta_{U_k})$, respectively. Using $A^2$ and  $A_k^2$,  for every real number $u\geq0$, we can rewrite  $F_1(u)$ and $F_{U_k}(u,\zeta_{U_k})$ as follows
	\begin{flalign}
		F_1(u)&=\mathbb{P}\{A^2\leq u\}=\mathbb{P}\{-\sqrt{u} \leq A\leq \sqrt{u}\}=2\Phi(\sqrt{u})-1, \label{PhitoF1EQ}\\
		F_{U_k}(u,\zeta_{U_k})&=\mathbb{P}\{A_k^2\leq u\}=\mathbb{P}\{-\sqrt{u} \leq A_k \leq \sqrt{u}\}=\mathbb{P}\{-\sqrt{u} \leq A+\sqrt{\zeta_{U_k}}\leq \sqrt{u}\}\nonumber\\&=\mathbb{P}\{ -\sqrt{u}-\sqrt{\zeta_{U_k}} \leq A\leq \sqrt{u}-\sqrt{\zeta_{U_k}} \} =\Phi(\sqrt{u}-\sqrt{\zeta_{U_k}})-\Phi(-\sqrt{u}-\sqrt{\zeta_U}), \label{PhitoFUeq}
	\end{flalign}
where $\Phi(.)$ is the CDF of a zero-mean normal random variable with the unit variance. Also, based on \eqref{PhitoF1EQ}, for every $0\leq s\leq 1$, let $F_1(u)=2\Phi(\sqrt{u})-1=s$. Then,  for every $s$, we have 
    $F^{-1}_1(s)=\Phi^{-2}(\frac{s+1}{2})$.
Now, using  \eqref{PhitoFUeq}, we can rewrite $F_{U_k}(cF_1^{-1}(1-c_1(\frac{\sigma^2}{n})),\zeta_{U_k})$ as follows
\begin{flalign}
	&F_{U_k}(cF_1^{-1}(1-c_1(\frac{\sigma^2}{n})),\zeta_{U_k})\nonumber\\&=\Phi(\sqrt{c}F_1^{-1/2}(1-c_1(\frac{\sigma^2}{n}))-\sqrt{\zeta_{U_k}}) -\Phi(-\sqrt{c}F_1^{-1/2}(1-c_1(\frac{\sigma^2}{n}))-\sqrt{\zeta_{U_k}})\nonumber\\&=\Phi(\sqrt{c}\Phi^{-1}(1-(\frac{c_1}{2})(\frac{\sigma^2}{n})) -\sqrt{\zeta_{U_k}})-\Phi(-\sqrt{c}\Phi^{-1}(1-(\frac{c_1}{2})(\frac{\sigma^2}{n})) -\sqrt{\zeta_{U_k}}).\label{probu0EQ}
\end{flalign} 
For $n\to \infty$, we calculate the first and the second terms of \eqref{probu0EQ}, respectively.  First,  for $n\to \infty$, we calculate the first term of \eqref{probu0EQ}. Before that, we obtain an upper bound on  $\sqrt{c}\Phi^{-1}(1-(\frac{c_1}{2})(\frac{\sigma^2}{n})) -\sqrt{\zeta_{U_k}}$, and then using this upper bound, we calculate $lim_{{n}\to \infty} \Phi(\sqrt{c}\Phi^{-1}(1-(\frac{c_1}{2})(\frac{\sigma^2}{n})) -\sqrt{\zeta_{U_k}})$. Based on Chernoff bound, for every $t> 0$, we have $Q(t)\leq exp{(-t^2 /2)}$, where $Q(t)= 1-\Phi(t)$ is called $Q$-function \cite{proakis2008digital}. Then, for every $t>0$, we have  $\Phi(t)\geq 1-exp{(-t^2 /2)}$. Since $\Phi^{-1}(.)$ is a monotone-increasing function, for every $t>0$, we can rewrite this inequality as
\begin{equation} \label{temptemp1eq}
     \Phi^{-1}(1-exp{(-t^2 /2)})\leq t.
\end{equation} 
By changing the variable  $t=\sqrt{2\ln{(\frac{n}{\sigma^2})}+2\ln{(\frac{2}{c_1})}}$, we have $1-\exp{(-t^2 /2)}=1-(\frac{c_1}{2})(\frac{\sigma^2}{n})$. Therefore, for every $n> \frac{c_1}{2}\sigma^2$, we can express \eqref{temptemp1eq} as follows
\begin{flalign}
\Phi^{-1}(1-(\frac{c_1}{2})(\frac{\sigma^2}{n})) \leq \sqrt{2\ln{(\frac{n}{\sigma^2})}+2\ln{(\frac{2}{c_1})}}.
\end{flalign}
Then, by multiplying $\sqrt{c}>0$ and adding $-\sqrt{\zeta_{U_k}}$ to both sides of the latest inequality, we have
\begin{flalign}
 \sqrt{c} \Phi^{-1}(1-(c_1/2)(\sigma^2/n)) -\sqrt{\zeta_{U_k}} \leq \sqrt{2c\ln{(n/\sigma^2)}+2c\ln{(2/c_1)}} -\sqrt{\zeta_{U_k}}.\label{tempBB1}
\end{flalign}
Now, using  inequality \eqref{tempBB1}, $\lim_{{n}\to \infty} \Phi(\sqrt{c}\Phi^{-1}((1-(c_1/2)(\sigma^2/n)) -\sqrt{\zeta_{U_k}})$ is obtained as follows 
\small
\begin{flalign}
	\lim\limits_{{n}\to \infty} \Phi(\sqrt{c}\Phi^{-1}(1-(c_1/2)(\sigma^2/n)) -\sqrt{\zeta_{U_k}})\leq  	\lim\limits_{{n}\to \infty} \Phi(\sqrt{2c\ln{(n/\sigma^2)}+2c\ln{(2/c_1)}} -\sqrt{\zeta_{U_k}})&\label{TempBB3} \\=\Phi(\lim\limits_{{n}\to \infty} [\sqrt{2c\ln{(n/\sigma^2)}+2c\ln{(2/c_1)}} -\sqrt{\zeta_{U_k}}]) =0,& \label{OrderdiffEQ}
\end{flalign}
\normalsize
where \eqref{TempBB3} holds since $\Phi$ is a monotone-increasing function. Note that in lemma 4 of \cite{gohain2023robust}, it is shown that based on spars Riesz condition,  for every $k\in\{1,2,...,\bar{L}\}$ that $\Sc\not\subset \Ic_{k-1}$, $\frac{(\Ptheta^*)^T \textbf{X}\big(\textbf{I}_n-\textbf{P}_{k-1}\big) \textbf{X}^T\Ptheta^*}{n}=\mathcal{O}(1)$. Therefore, the following holds
\small
\begin{flalign}
   \frac{\| (\textbf{P}_{k}-\textbf{P}_{k-1})\textbf{X}^T\Ptheta^* \|_2^2}{n}&=\frac{(\Ptheta^*)^T \textbf{X}(\textbf{P}_k-\textbf{P}_{k-1}) \textbf{X}^T \Ptheta^{*} }{n}\label{projecttemp1eq}  =\frac{(\Ptheta^*)^T \textbf{X}(\textbf{P}_k-\textbf{P}_{k-1}+\textbf{I}_n -\textbf{I}_n) \textbf{X}^T \Ptheta^{*} }{n}\\ & =\frac{(\Ptheta^*)^T \textbf{X}(\textbf{I}_n-\textbf{P}_{k-1}) \textbf{X}^T \Ptheta^{*} }{n}- \frac{(\Ptheta^*)^T \textbf{X}(\textbf{I}_n-\textbf{P}_{k}) \textbf{X}^T \Ptheta^{*} }{n}\nonumber\\ &  =\frac{\|(\textbf{I}_n-\textbf{P}_{k-1}) \textbf{X}^T \Ptheta^{*} \|_2^2}{n}-\frac{\|(\textbf{I}_n-\textbf{P}_{k}) \textbf{X}^T \Ptheta^{*} \|_2^2 }{n} =\mathcal{O}(1)\label{O1Zetakeq},
\end{flalign}
\normalsize
where \eqref{projecttemp1eq} holds since $(\textbf{P}_k-\textbf{P}_{k-1})$ is orthogonal projection matrix \cite{kallummil2018signal}  that is an idempotent matrix. Therefore, based on \eqref{O1Zetakeq}, $\frac{\| (\textbf{P}_{k}-\textbf{P}_{k-1})\mathbb{\textbf{X}^T\pmb{\theta}^*} \|_2^2}{n}=\mathcal{O}(1)$,   $\zeta_{U_k}=\frac{n}{\sigma^2}\frac{\| (\textbf{P}_{k}-\textbf{P}_{k-1})\mathbb{\textbf{X}^T\pmb{\theta}^*} \|_2^2}{n}=\mathcal{O}(\frac{n}{\sigma^2})$, and $\sqrt{\zeta_{U_k}}=\mathcal{O}(\sqrt{n/\sigma^2})$. On the other hand, $\sqrt{2c\ln{(n/\sigma^2)}+2c\ln{(2/c_1)}}=\mathcal{O}(\sqrt{ln(n/\sigma^2)})$.  Therefore,  in \eqref{OrderdiffEQ}, $lim_{{n}\to \infty}  (\sqrt{2c\ln{(n/\sigma^2)}+2c\ln{(2/c_1)}}-\sqrt{\zeta_{U_k}})=-\infty$. Since $\sqrt{\zeta_{U_k}}$ goes to infinity with a faster rate than $(\sqrt{2c\ln{(n/\sigma^2)}+2c\ln{(2/c_1)}})$. 

Now, for $n\to \infty$, we calculate the second term of \eqref{probu0EQ}, $\Phi(-\sqrt{c}\Phi^{-1}(1-(c_1/{2})({\sigma^2}/{n})) -\sqrt{\zeta_{U_k}})$, as follows
\begin{flalign}
	&lim_{{n}\to \infty} \Phi(-\sqrt{c}\Phi^{-1}(1-(c_1/2)(\sigma^2/n)) -\sqrt{\zeta_{U_k}})= \nonumber \\&  \Phi(-\sqrt{c}\Phi^{-1}\big(lim_{{n}\to \infty} (1-(c_1/2)(\sigma^2/n))\big) -lim_{{n}\to \infty} \sqrt{\zeta_{U_k}}) =0,\label{limzero1EQ}
\end{flalign}
 where, \eqref{limzero1EQ} holds since $lim_{{n}\to \infty} (1-(c_1/2)(\sigma^2/n))=1$, $\zeta_{U_k}=\mathcal{O}(n/\sigma^2)$, and for every constant $\sqrt{c}>0$, $\sqrt{c}\Phi^{-1}(1) =lim_{{n}\to \infty} \sqrt{\zeta_{U_k}}=\infty$. 

Therefore, using \eqref{OrderdiffEQ} and \eqref{limzero1EQ} in \eqref{probu0EQ},  
 we have $\lim_{{n}\to \infty} F_{U_k}(cF_1^{-1}(1-c_1(\frac{\sigma^2}{n})),\zeta_{U_k})=0$. Consequently,  $\lim_{{n}\to \infty} \Big[1-F_{U_k}(cF_1^{-1}(1-c_1(\frac{\sigma^2}{n})),\zeta_{U_k})\Big]=1$, and the second statement of the theorem will be proved.
	\end{proof}








\end{document}